\documentclass{article}
\PassOptionsToPackage{numbers, compress}{natbib}
\usepackage{natbib}
\usepackage[preprint]{neurips_2021}

\usepackage{microtype}
\usepackage{graphicx}
\usepackage{subfigure} 
\usepackage{booktabs}
\usepackage[colorlinks,
            linkcolor=blue,
            anchorcolor=black,
            citecolor=blue
            ]{hyperref}

\newcommand*{\dif}{\mathop{}\!\mathrm{d}}

\usepackage{amsmath,amsfonts,bm}

\def\eqref#1{equation~\ref{#1}}

\def\1{\bm{1}}

\DeclareMathAlphabet{\mathsfit}{\encodingdefault}{\sfdefault}{m}{sl}
\SetMathAlphabet{\mathsfit}{bold}{\encodingdefault}{\sfdefault}{bx}{n}

\usepackage{url}
\usepackage{enumitem}
\usepackage{multirow}  
\usepackage{amsthm}
\usepackage{amsmath}
\usepackage{stmaryrd}  
\usepackage{bbm}
\usepackage{amssymb}
\usepackage{wrapfig}                              
\newtheorem{theorem}{Theorem}
\newtheorem*{theorem1}{Theorem 1}
\newtheorem*{theorem2}{Theorem 2}
\newtheorem*{theorem3}{Theorem 3}
\newtheorem*{theorem4}{Theorem 4}
\newtheorem{definition}{Definition}
\newtheorem{lemma}{Lemma}

\newtheorem{assumption}{Assumption}

\usepackage{algorithm}
\usepackage{algorithmic}
\usepackage{nicefrac}
\usepackage{xcolor}
\usepackage{wrapfig}

\newcommand{\Exp}{\mathrm{\mathbb{E}}}
\newcommand{\Real}{\mathbb{R}}

\newcommand{\nt}{{n_t}}

\graphicspath{{figures/}}

\title{Cycle Self-Training for Domain Adaptation}

\begin{document}

\author{
	Hong Liu\\
	Dept of Electronic Engineering\\
  Tsinghua University\\
	{\small\texttt{hongliu9903@gmail.com}}
  \And
  Jianmin Wang\\
	School of Software, BNRist\\
  Tsinghua University\\
	{\small\texttt{jimwang@tsinghua.edu.cn}}
  \And
  Mingsheng Long\thanks{Corresponding author: Mingsheng Long (\texttt{mingsheng@tsinghua.edu.cn})}\\
	School of Software, BNRist\\
  Tsinghua University\\
	{\small\texttt{mingsheng@tsinghua.edu.cn}}
}

\maketitle

\begin{abstract}
	Mainstream approaches for unsupervised domain adaptation (UDA) learn domain-invariant representations to narrow the domain shift, which are empirically effective but theoretically challenged by the hardness or impossibility theorems. Recently, self-training has been gaining momentum in UDA, which exploits unlabeled target data by training with target pseudo-labels. However, as corroborated in this work, under distributional shift, the pseudo-labels can be unreliable in terms of their large discrepancy from target ground truth. In this paper, we propose \textit{Cycle Self-Training} (CST), a principled self-training algorithm that explicitly enforces pseudo-labels to generalize across domains. CST cycles between a forward step and a reverse step until convergence. In the forward step, CST generates target pseudo-labels with a source-trained classifier. In the reverse step, CST trains a target classifier using target pseudo-labels, and then updates the shared representations to make the target classifier perform well on the source data. We introduce the Tsallis entropy as a confidence-friendly regularization to improve the quality of target pseudo-labels. We analyze CST theoretically under realistic assumptions, and provide hard cases where CST recovers target ground truth, while both invariant feature learning and vanilla self-training fail. Empirical results indicate that CST significantly improves over the state-of-the-arts on visual recognition and sentiment analysis benchmarks.
\end{abstract}

\vspace{-10pt}
\section{Introduction}
\vspace{-5pt}

Transferring knowledge from a source domain with rich supervision to an unlabeled target domain is an important yet challenging problem. Since deep neural networks are known to be sensitive to subtle change in underlying distributions~\cite{cite:NIPS14CNN}, models trained on one labeled dataset often fail to generalize to another unlabeled dataset~\citep{szegedy2013intriguing,2018Deep}. Unsupervised domain adaptation (UDA) addresses the challenge of distributional shift by adapting the source model to the unlabeled target data~\citep{cite:MIT2009Dataset,cite:TKDE10TLSurvey}.

The mainstream paradigm for UDA is \emph{feature adaptation}, a.k.a. \emph{domain alignment}. By reducing the distance of the source and target feature distributions, these methods learn \textit{invariant representations} to facilitate knowledge transfer between domains~\citep{cite:ICML15DAN,cite:JMLR17DANN,cite:ICML17JAN,cite:CVPR18MCD,cite:NIPS18CDAN,pmlr-v97-zhang19i}, with successful applications in various areas such as computer vision~\citep{cite:CVPR17ADDA,cite:ICML18CYCADA,10.1007/978-3-030-01219-9_18} and natural language processing~\citep{ziser-reichart-2018-pivot,qu-etal-2019-adversarial}. Despite their popularity, the impossibility theories \cite{10.1007/978-3-642-34106-9_14} uncovered intrinsic limitations of learning invariant representations when it comes to label shift~\citep{pmlr-v97-zhao19a,Li2020RethinkingDM} and shift in the support of domains \citep{pmlr-v89-johansson19a}.

\begin{figure}[htbp]
	\centering
	\includegraphics[width=0.90\columnwidth]{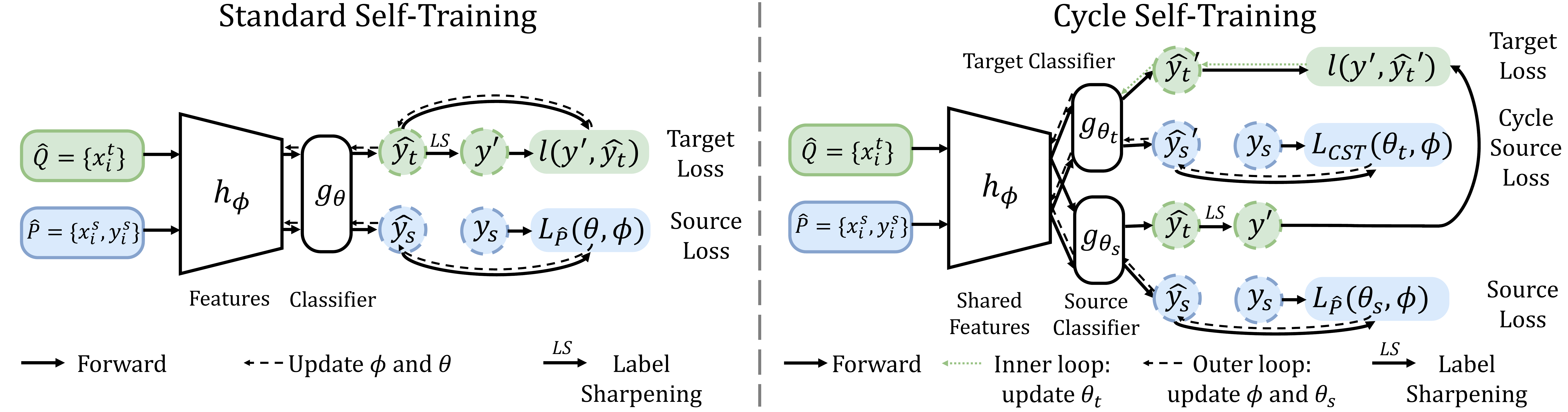}
	\vspace{-5pt}
	\caption{\small\textbf{Standard self-training vs. cycle self-training.} In standard self-training, we generate target pseudo-labels with a source model, and then train the model with both source ground-truths and target pseudo-labels. In cycle self-training, we train a target classifier with target pseudo-labels in the \textbf{inner loop}, and make the target classifier perform well on the source domain by updating the shared representations in the \textbf{outer loop}.}
	\vspace{-10pt}
	\label{fig:rst}
\end{figure}

Recently, \textit{self-training} (a.k.a. \emph{pseudo-labeling})~\citep{french2018selfensembling,Zou_2019_ICCV,pmlr-v119-kumar20c,Li2020RethinkingDM,prabhu2020sentry,xie2020innout} has been gaining momentum as a promising alternative to feature adaptation. Originally tailored to semi-supervised learning, self-training generates pseudo-labels of unlabeled data, and jointly trains the model with source labels and target pseudo-labels~\citep{article-pl,sl2016A,pmlr-v119-kumar20c}. However, the \emph{distributional shift} in UDA makes pseudo-labeling more difficult. Directly using all pseudo-labels is risky due to accumulated error and even trivial solution~\citep{8953748}. Thus previous works tailor self-training to UDA by selecting trustworthy pseudo-labels. Using confidence threshold or reweighting, recent works try to alleviate the negative effect of domain shift in standard self-training~\citep{Zou_2019_ICCV,prabhu2020sentry}, but they can be brittle and require expensive tweaking of the threshold or weight for different tasks, and their performance gain is still inconsistent.

In this work, we first analyze the quality of pseudo-labels with or without domain shift to delve deeper into the difficulty of standard self-training in UDA. On popular benchmark datasets, when the source and target are the same, our analysis indicates that the pseudo-label distribution is almost identical to the ground-truth distribution. However, with distributional shift, their discrepancy can be \emph{very large} with examples of several classes mostly misclassified into other classes. We also study the difficulty of selecting correct pseudo-labels with popular criteria under domain shift. Although entropy and confidence are reasonable selection criteria for correct pseudo-labels without domain shift, the domain shift makes their accuracy decrease sharply.

Our analysis shows that domain shift makes pseudo-labels \emph{unreliable} and that self-training on selected target instances with accurate pseudo-labels is less successful. Thereby, more principled improvement of standard self-training should be tailored to UDA and address the domain shift explicitly.
In this work, we propose \textit{Cycle Self-Training} ({CST}), a principled self-training approach to UDA, which overcomes the limitations of standard self-training (see Figure~\ref{fig:rst}). Different from previous works to select target pseudo-labels with hard-to-tweak protocols, CST learns to generalize the pseudo-labels across domains. Specifically, CST \textit{cycles} between the use of target pseudo-labels to train a target classifier, and the update of shared representations to make the target classifier perform well on the source data. In contrast to the standard Gibbs entropy that makes the target predictions over-confident, we propose a confidence-friendly uncertainty measure based on the \textit{Tsallis entropy} in information theory, which adaptively minimizes the uncertainty without manually tuning or setting thresholds. Our method is simple and generally applicable to vision and language tasks with various backbones.

We empirically evaluate our method on a series of standard UDA benchmarks. Results indicate that CST outperforms previous state-of-the-art methods in 21 out of 25 tasks for object recognition and sentiment classification. Theoretically, we prove that the minimizer of CST objective is endowed with general guarantees of target performance. We also study hard cases on specific distributions, showing that CST recovers target ground-truths while both feature adaptation and standard self-training fail.

\vspace{-5pt}
\section{Preliminaries}\label{sec:pre}
\vspace{-5pt}

We study {unsupervised domain adaptation} (UDA). Consider a source distribution $P$ and a target distribution $Q$ over the input-label space $\mathcal{X} \times \mathcal{Y}$. We have access to $n_s$ labeled \emph{i.i.d.}~samples $\widehat{P}=\{x_i^s,y_i^s\}_{i=1}^{n_s}$ from $P$ and $n_t$ unlabeled \emph{i.i.d.}~samples $\widehat{Q}=\{x_i^t\}_{i=1}^{n_t}$ from $Q$. The model $f$ comprises a feature extractor $h_\phi$ parametrized by $\phi$ and a head (linear classifier) $g_\theta$ parametrized by $\theta$, i.e. $f_{\theta,\phi}(x) = g_\theta(h_\phi(x))$. The loss function is $\ell(\cdot,\cdot)$. Denote by $L_{P}(\theta,\phi):=\Exp_{(x,y)\sim P}\ell(f_{\theta,\phi}(x),y)$ the expected error on $P$. Similarly, we use $L_{\widehat P}(\theta,\phi)$ to denote the empirical error on dataset $\widehat P$.

We discuss two mainstream UDA methods and their formulations: feature adaptation and self-training. 

\textbf{Feature Adaptation}
trains the model $f$ on the source dataset $\widehat P$, and simultaneously matches the source and target distributions in the representation space $\mathcal{Z} = h(\mathcal{X})$:
\begin{align}
\mathop{\min}_{\theta,\phi} L_{\widehat P}(\theta,\phi) + d(h_\sharp{\widehat P},h_\sharp{\widehat Q}).
\end{align}
Here, $h_\sharp{\widehat P}$ denotes the pushforward distribution of $\widehat P$, and $d(\cdot,\cdot)$ is some distribution distance. For instance, \citet{cite:ICML15DAN} used maximum mean discrepancy $d_{\textup{MMD}}$, and \citet{cite:JMLR17DANN} approximated the $\mathcal{H}\Delta\mathcal{H}$-distance $d_{\mathcal{H}\Delta\mathcal{H}}$~\citep{cite:ML10DAT} with adversarial training. Despite its pervasiveness, recent works have shown the intrinsic limitations of feature adaptation under real-world situations~\citep{10.1007/978-3-642-34106-9_14,pmlr-v97-zhao19a,pmlr-v97-liu19b,Li2020RethinkingDM,pmlr-v89-johansson19a}.

\textbf{Self-Training} is considered a promising alternative to feature adaptation. In this work we mainly focus on pseudo-labeling~\citep{article-pl,pmlr-v119-kumar20c}. Stemming from semi-supervised learning, standard self-training trains a source model $f_s$ on the source dataset $\widehat P$: $\mathop{\min}_{\theta_s,\phi_s} L_{\widehat P}(\theta_s,\phi_s)$. The \textit{target pseudo-labels} are then generated by $f_s$ on the target dataset $\widehat Q$.
To leverage unlabeled target data, self-training trains the model on the source and target datasets together with source ground-truths and target pseudo-labels:
\begin{align}\label{eqn:Ljoint}
\mathop{\min}_{\theta,\phi} L_{\widehat P}(\theta,\phi) + \Exp_{x\sim \widehat Q}\ell(f_{\theta,\phi}(x),\mathop{\arg\max}_i\{f_{\theta_s,\phi_s}(x)_{[i]}\}).
\end{align}
Self-training also uses label-sharpening as a standard protocol \citep{article-pl,fixmatch}. Another popular variant of pseudo-labeling is the teacher-student model~\citep{NIPS2014_66be31e4,NIPS2017_68053af2}, which iteratively improves the quality of pseudo-labels via alternatively replacing $\theta_s$ and $\phi_s$ with $\theta$ and $\phi$ of the previous iteration. 

\vspace{-5pt}
\subsection{Limitations of Standard Self-Training}\label{sec:limitations}
\vspace{-5pt}

Standard self-training with pseudo-labels uses unlabeled data efficiently for semi-supervised learning~\citep{article-pl,sl2016A,fixmatch}. Here we carry out exploratory studies on the popular VisDA-2017~\citep{DBLP:journals/corr/abs-1710-06924} dataset using ResNet-50 backbones. We find that domain shift makes the pseudo-labels biased towards several classes and thereby unreliable in UDA. See Appendix~\ref{sec:add11} for details and results on more datasets.

\textbf{Pseudo-label distributions with or without domain shift.} We resample the original VisDA-2017 to simulate different relationship between source and target domains: 1) \emph{i.i.d.}, 2) covariate shift, and 3) label shift. We train the model on the three variants of source dataset and use it to generate target pseudo-labels. We show the distributions of target ground-truths and pseudo-labels in Figure~\ref{fig:pseudo} (Left). When the source and target distributions are identical, the distribution of pseudo-labels is almost the same as ground-truths, indicating the reliability of pseudo-labels. In contrast, when exposed to label shift or covariate shift, the distribution of pseudo-labels is significantly different from target ground-truths. Note that classes 2, 7, 8 and 12 appear rarely in the target pseudo-labels in the covariate shift setting, indicating that the pseudo-labels are biased towards several classes due to domain shift. Self-training with these pseudo-labels is risky since it may lead to misalignment of distributions and misclassify many examples of classes 2, 7, 8 and 12.

\begin{figure*}[t]
  \centering 
   \subfigure{
    \includegraphics[width=0.28\textwidth]{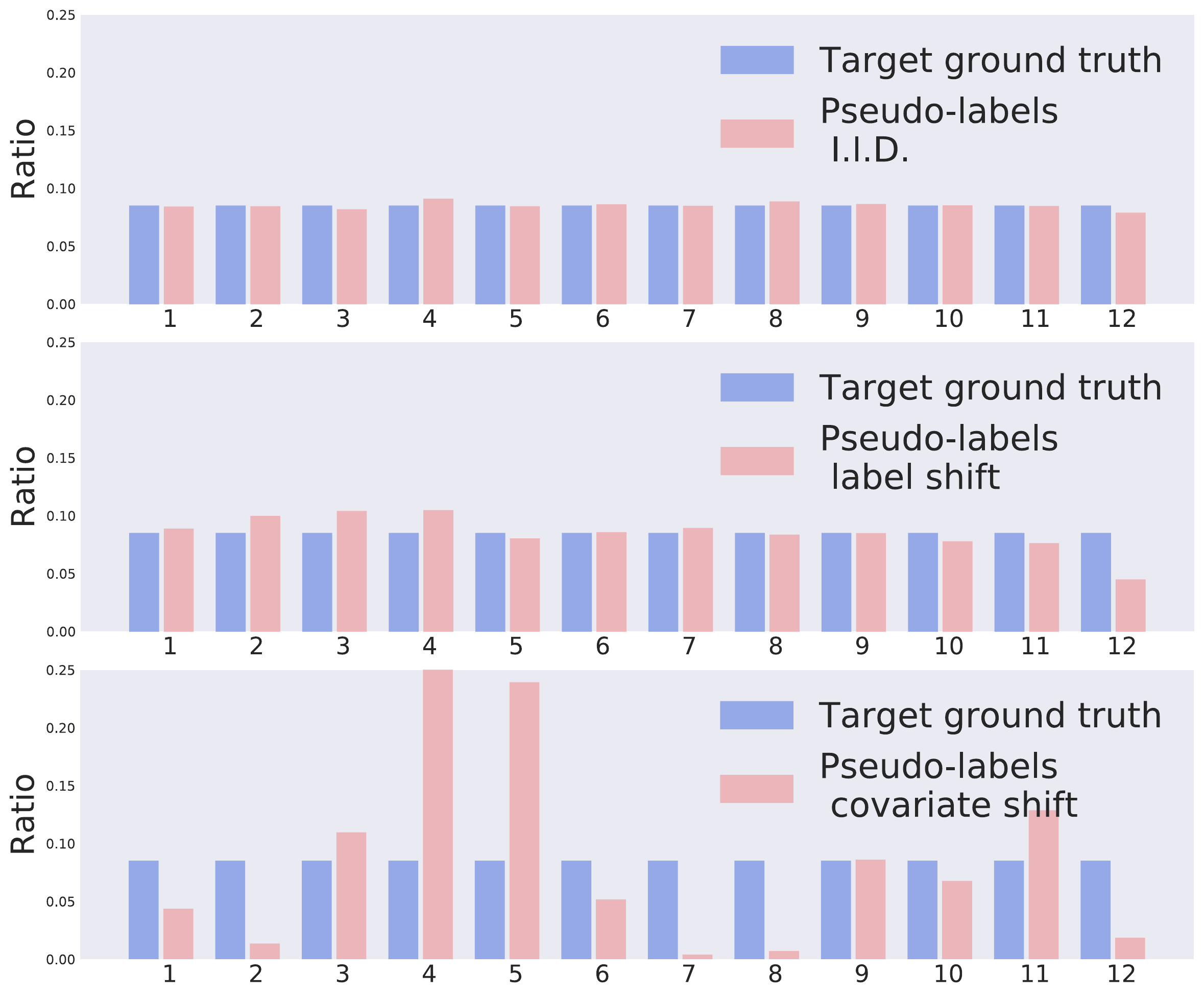} 
    }
    \hfil
   \subfigure{
    \includegraphics[width=0.378\textwidth]{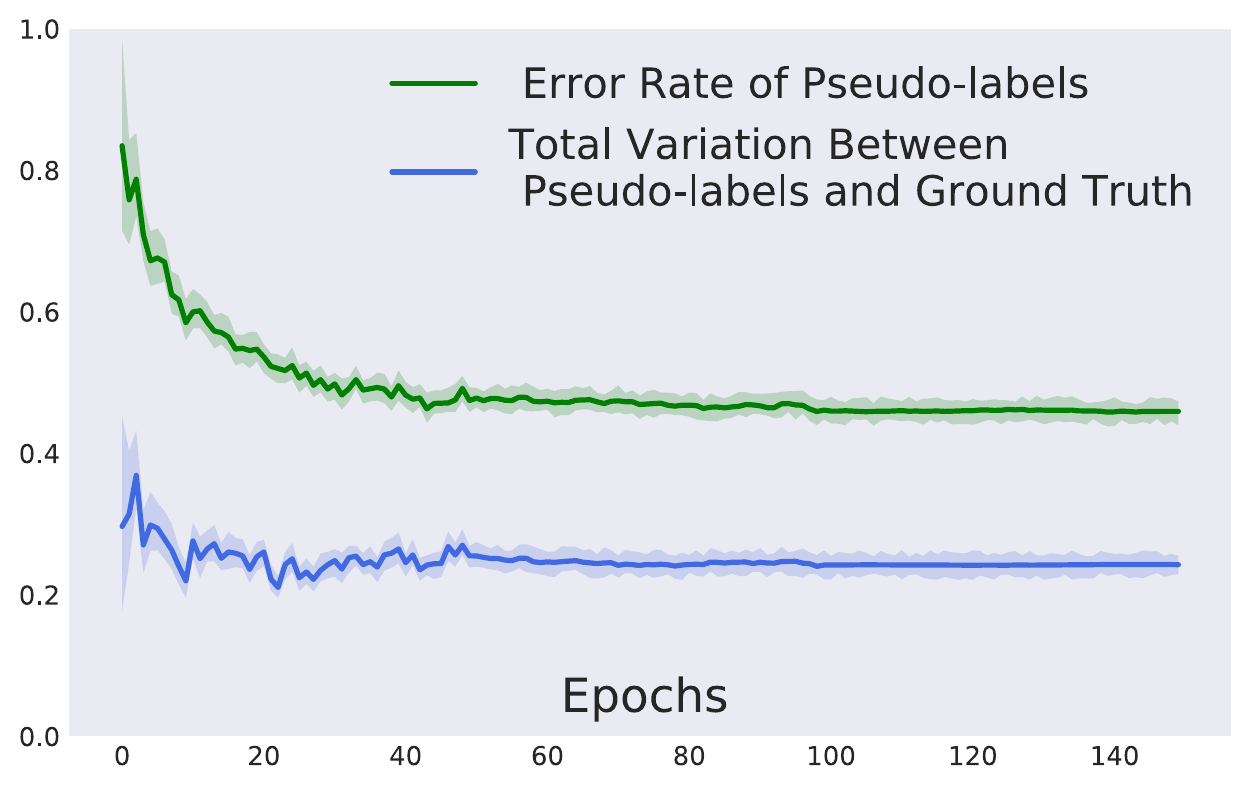}
    }
    \hfil
   \subfigure{
    \includegraphics[width=0.121\textwidth]{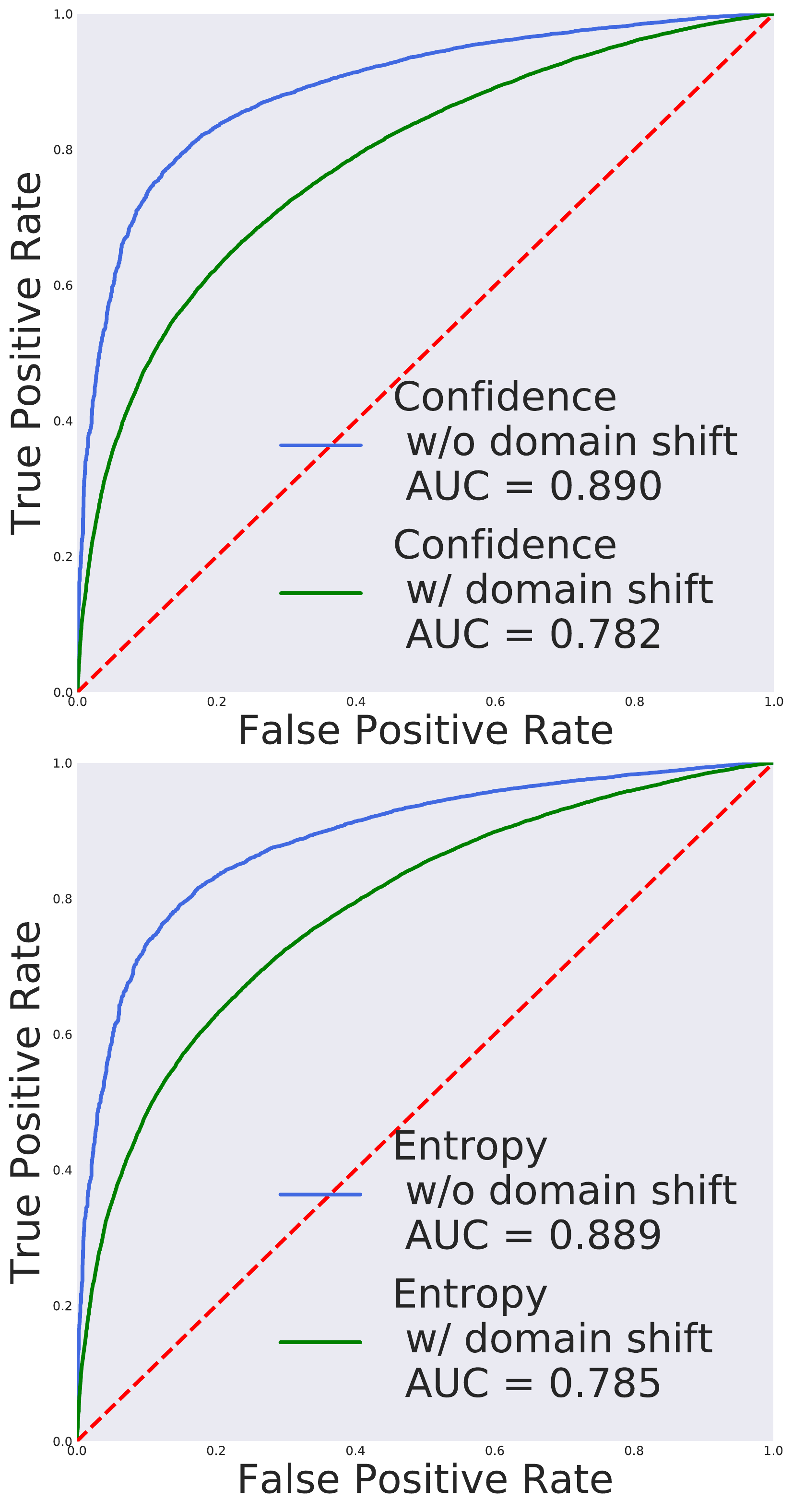}
    }
    \vspace{-5pt}
  \caption{\small\textbf{Analysis of pseudo-labels under domain shift on VisDA-2017. }Left: Pseudo-label distributions with and without domain shift. Middle: Changes of pseudo-label distributions throughout training. Right: Quality of pseudo-labels under different pseudo-label selection criteria. } 
  \label{fig:pseudo}
  \vspace{-10pt}
\end{figure*}

\textbf{Change of pseudo-label distributions throughout training.} To further study the change of pseudo-labels in standard self-training, we compute the total variation (TV) distance between target ground-truths and target pseudo-labels: $d_\textup{TV}(c,c') = \frac{1}{2}\sum_i\|c_i - c_i'\|$, where $c_i$ is the ratio of class $i$. We plot its change during training in Figure~\ref{fig:pseudo} (Middle). Although the error rate of pseudo-labels continues to decrease, $d_\textup{TV}$ remains almost unchanged at $0.26$ throughout training. Note that $d_\textup{TV}$ is the lower bound of the error rate of the pseudo-labels (shown in Appendix~\ref{sec:add11}). If $d_\textup{TV}$ converges to $0.26$, then the accuracy of pseudo-labels is upper-bounded by $0.74$. This indicates that the important denoising ability \cite{wei_2021_ICLR} of pseudo-labels in standard self-training is hindered by domain shift.

\textbf{Difficulty of selecting reliable pseudo-labels under domain shift.} To mitigate the negative effect of false pseudo-labels, recent works proposed to select correct pseudo-labels based on thresholding the entropy or confidence criteria~\citep{cite:NIPS16RTN,french2018selfensembling,cite:NIPS18CDAN,fixmatch}. However, it remains unclear whether these strategies are still effective under domain shift. Here we compare the quality of pseudo-labels selected by different strategies with or without domain shift. For each strategy, we compute False Positive Rate and True Positive Rate for different thresholds and plot its ROC curve in Figure~\ref{fig:pseudo} (Right). When the source and target distributions are identical, both entropy and confidence are reasonable strategies for selecting correct pseudo-labels (AUC=0.89). However, when the target pseudo-labels are generated by the source model, the quality of pseudo-labels decreases sharply under domain shift (AUC=0.78).

\vspace{-5pt}
\section{Approach}\label{sec:method}
\vspace{-5pt}

We present Cycle Self-Training (CST) to improve pseudo-labels under domain shift. An overview of our method is given in Figure~\ref{fig:rst}. Cycle Self-Training iterates between a \textit{forward step} and a \textit{reverse step} to make self-trained classifiers generalize well on both target and source domains.

\vspace{-8pt}
\subsection{Cycle Self-Training}\label{sec:cst}
\vspace{-5pt}

\textbf{Forward Step.} 
Similar to standard self-training, we have a source classifier $\theta_s$ trained on top of the shared representations $\phi$ on the labeled source domain, and use it to generate target pseudo-labels as 
\begin{align}\label{eq:source}
y' = \mathop{\arg\max}_i\{f_{\theta_s,\phi}(x)_{[i]}\},
\end{align}
for each $x$ in the target dataset $\widehat{Q}$. Traditional self-training methods use confidence thresholding or reweighting to select reliable pseudo-labels. For example, \citet{fixmatch} select pseudo-labels with softmax value and \citet{cite:NIPS18CDAN} add entropy reweighting to rely on examples with more confidence prediction. However, the output of deep networks is usually miscalibrated~\citep{pmlr-v70-guo17a}, and is not necessarily related to the ground-truth confidence even on the same distribution. In domain adaptation, as shown in Section~\ref{sec:limitations}, the discrepancy between the source and target domains makes pseudo-labels even more unreliable, and the performance of commonly used selection strategies is also unsatisfactory. Another drawback is the expensive tweaking in order to find the optimal confidence threshold for new tasks. To better apply self-training to domain adaptation, we expect that the model can gradually refine the pseudo-labels by itself without the cumbersome selection or thresholding.

\textbf{Reverse Step.} 
We design a complementary step with the following insights to improve self-training. Intuitively, the labels on the source domain contain both useful information that can transfer to the target domain and harmful information that can make pseudo-labels incorrect. Similarly, \textit{reliable pseudo-labels} on the target domain can transfer to the source domain in turn, while models trained with incorrect pseudo-labels on the target domain cannot transfer to the source domain. In this sense, if we explicitly train the model to make target pseudo-labels informative of the source domain, we can gradually make the pseudo-labels more accurate and learn to generalize to the target domain.

Specifically, with the pseudo-labels $y'$ generated by the source classifier $\theta_s$ at hand as in \eqref{eq:source}, we train a target head $\hat\theta_t(\phi)$ on top of the representation $\phi$ with pseudo-labels on the target domain $\widehat Q$,
\begin{align}\label{eq:target}
\hat\theta_t(\phi) = \mathop{\arg\min}_{\theta}\Exp_{x\sim \widehat Q}\ell(f_{\theta,\phi}(x),y').
\end{align}
We wish to make the target pseudo-labels informative of the source domain and gradually refine them. To this end, we update the shared feature extractor $\phi$ to predict accurately on the source domain and jointly \textit{enforce the target classifier $\hat\theta_t(\phi)$ to perform well on the source domain}. This naturally leads to the objective of \textbf{Cycle Self-Training}:
\begin{align}\label{eq:overall}
\mathop{\textup{minimize}}_{\theta_s,\phi}L_{\textup{Cycle}}({\theta_s,\phi}):= L_{\widehat P}(\theta_s,\phi)+L_{\widehat P}(\hat\theta_t(\phi),\phi).
\end{align}

\vspace{-10pt}
\textbf{Bi-level Optimization.} 
The objective in \eqref{eq:overall} relies on the solution $\hat\theta_t(\phi)$ to the objective in \eqref{eq:target}. Thus, CST formulates a \emph{bi-level} optimization problem. In the \textbf{inner loop} we generate target pseudo-labels with the source classifier (\eqref{eq:source}), and train a target classifier with target pseudo-labels (\eqref{eq:target}). After each inner loop, we update the feature extractor $\phi$ for one step in the \textbf{outer loop} (\eqref{eq:overall}), and start a new inner loop again. However, since the inner loop of the optimization in \eqref{eq:target} only involves the light-weight linear head $\theta_t$, we propose to calculate the analytical form of $\hat\theta_t(\phi)$ and directly back-propagate to the feature extractor $\phi$ instead of calculating the second-order derivatives as in MAML~\citep{pmlr-v70-finn17a}. The resulting framework is as fast as training two heads jointly. Also note that the solution $\hat\theta_t(\phi)$ relies on $\theta_s$ implicitly through $y'$. However, both standard self-training and our implementation use label sharpening, making $y'$ not differentiable. Thus we follow vanilla self-training and \textit{do not} consider the gradient of $\hat\theta_t(\phi)$ w.r.t. $y'$ in the outer loop optimization. We defer the derivation and implementation of bi-level optimization to Appendix~\ref{sec:bilevel}.

\vspace{-8pt}
\subsection{Tsallis Entropy Minimization}\label{sec:tsallis}
\vspace{-5pt}

Gibbs entropy is widely used by existing semi-supervised learning methods to regularize the model output and minimize the uncertainty of predictions on unlabeled data~\citep{cite:NIPS04SSLEM}. In this work, we generalize Gibbs entropy to Tsallis entropy~\cite{1988Possible} in information theory. 
Suppose the softmax output of a model is $y\in\Real^K$, then the $\alpha$-\textit{Tsallis entropy} is defined as
\vspace{-5pt}
\begin{align}
S_\alpha(y) = \frac{1}{\alpha-1} \left(1 - \sum y_{[i]}^\alpha \right),
\end{align} 
where $\alpha>0$ is the \textit{entropic-index}. Note that $\operatorname{lim}_{\alpha\rightarrow 1} S_\alpha(y) = \sum_i -y_{[i]}\textup{log}(y_{[i]})$ which exactly recovers the Gibbs entropy. When $\alpha = 2$, $S_\alpha(y)$ becomes the Gini impurity $1-\sum_{i}y_{[i]}^2$.

\begin{algorithm}[bp]
\caption{Cycle Self-Training (CST)}
\label{alg:CST}
\begin{algorithmic}[1]
\STATE {\bfseries Input:} source dataset $\widehat P$ and target dataset $\widehat Q$.
\FOR{$\textup{epoch}=0$ {\bfseries to} \texttt{MaxEpoch}} 
	\STATE  Select $\hat\alpha$ as~\eqref{eq:alpha} at the start of each epoch.
	\FOR{$t=0$ {\bfseries to} \texttt{MaxIter}} 
	\STATE \textbf{Forward Step}
	\STATE Generate pseudo-labels on the target domain with $\phi$ and $\theta_s$: $y' = \mathop{\arg\max}_i\{f_{\theta_s,\phi}(x)_{[i]}\}$.
	\STATE \textbf{Reverse Step}
	\STATE Train a target head $\hat\theta_t(\phi)$ with target pseudo-labels $y'$ on the feature extractor $\phi$:
	\vspace{-5pt}
	\begin{align}
	\hat\theta_t(\phi) = \mathop{\arg\min}_{\theta}\Exp_{x\sim \widehat Q}\ell(f_{\theta,\phi}(x),y').\nonumber
	\end{align} 
	\vspace{-12pt}
	\STATE Update the feature extractor $\phi$ and the source head $\theta_s$ to make $\hat\theta_t(\phi)$ perform well on the source dataset and minimize the $\hat\alpha$-Tsallis entropy on the target dataset:
	\vspace{-5pt}
	\begin{align}
	\phi \leftarrow \phi &- \eta \nabla_{\phi} [L_{\widehat P}(\theta_s,\phi)
	+L_{\widehat P}(\hat\theta_t(\phi),\phi) + L_{\widehat Q, \textup{Tsallis},\hat\alpha}({\theta_s,\phi})].\label{opt1}\\
	&\theta_s \leftarrow \theta_s  - \eta \nabla_{\theta_s} [L_{\widehat P}(\theta_s,\phi)
	+ L_{\widehat Q, \textup{Tsallis},\hat\alpha}({\theta_s,\phi})].\label{opt2}
	\end{align}
	\vspace{-23.5pt}
\ENDFOR
\ENDFOR
\end{algorithmic}
\end{algorithm}

We propose to control the uncertainty of target pseudo-labels based on \textbf{Tsallis entropy minimization}:
\begin{align}
L_{\widehat Q, \textup{Tsallis},\alpha}({\theta,\phi}):= \Exp_{x\sim \widehat Q}S_\alpha(f_{\theta,\phi}(x)).
\end{align} 

\begin{wrapfigure}{r}{6cm}
\centering
\vspace{-12pt}
\includegraphics[width=0.42\columnwidth]{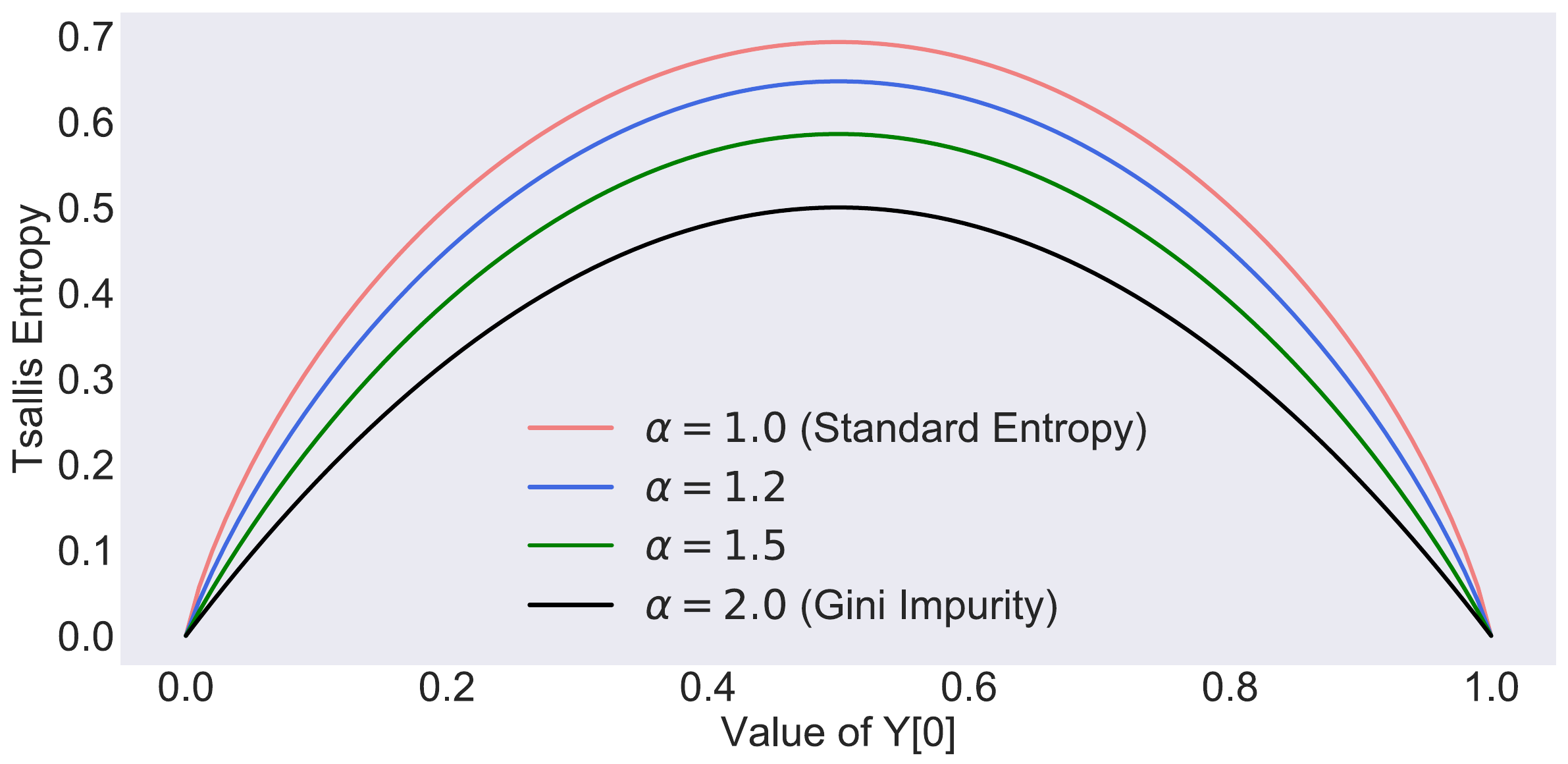}
\vspace{-8pt}
\caption{\small Tsallis entropy vs. entropic-index $\alpha$.}
\vspace{-22pt}
\label{fig:entt}
\end{wrapfigure}

Figure~\ref{fig:entt} shows the change of Tsallis entropy with different entropic-indices $\alpha$ for binary problems. Intuitively, smaller $\alpha$ exerts more penalization on uncertain predictions and larger $\alpha$ allows several scores $y_i$'s to be similar. This is critical in self-training since an overly small $\alpha$ (as in Gibbs entropy) will make the incorrect dimension of pseudo-labels close to $1$ and have no chance to be corrected throughout training. In Section~\ref{analysis}, we further verify this property with experiments.

An important improvement of the Tsallis entropy over Gibbs entropy is that it can choose the suitable measure of uncertainty for different systems to avoid over-confidence caused by overly penalizing the uncertain pseudo-labels. To automatically find the suitable $\alpha$, we adopt a similar strategy as Section~\ref{sec:cst}. The intuition is that if we use the suitable entropic-index $\alpha$ to train the source classifier $\theta_{s,\alpha}$, the target pseudo-labels generated by $\theta_{s,\alpha}$ will contain desirable knowledge of the source dataset, i.e. a target classifier $\theta_{t,\alpha}$ trained with these pseudo-labels will perform well on the source domain. Therefore, we semi-supervisedly train a classifier $\hat{\theta}_{s,\alpha}$ on the source domain with the $\alpha$-Tsallis entropy regularization $L_{\widehat Q, \textup{Tsallis},\alpha}$ on the target domain as: $\hat{\theta}_{s,\alpha} = \mathop{\arg\min}_\theta L_{\widehat P}(\theta,\phi)+L_{\widehat Q, \textup{Tsallis},\alpha}({\theta,\phi})$, from which we obtain the target pseudo-labels. Then we train another head $\hat{\theta}_{t,\alpha}$ with target pseudo-labels. We automatically find $\alpha$ by minimizing the loss of $\hat{\theta}_{t,\alpha}$ on the source data:
\begin{align}\label{eq:alpha}
\hat{\alpha} = \mathop{\arg\min}_{\alpha\in[1,2]}L_{\widehat P}(\hat{\theta}_{t,\alpha},\phi)
\end{align}
To solve \eqref{eq:alpha}, we discretize the feasible region $[1,2]$ of $\alpha$ and use discrete optimization to lower computational cost. We also update $\alpha$ at the start of each epoch, since we found more frequent update leads to no performance gain. Details are deferred to Appendix~\ref{sec:select}.
Finally, with the optimal $\hat\alpha$ found, we add the $\hat\alpha$-Tsallis entropy minimization term $L_{\widehat Q, \textup{Tsallis},\hat\alpha}$ to the overall objective:
\begin{align}\label{eqn:cst_loss}
\mathop{\textup{minimize}}_{\theta_s,\phi}L_{\textup{Cycle}}({\theta_s,\phi})+ L_{\widehat Q, \textup{Tsallis},\hat\alpha}({\theta_s,\phi}).
\end{align}
In summary, Algorithm~\ref{alg:CST} depicts the complete training procedure of Cycle Self-Training (CST).

\vspace{-5pt}
\section{Theoretical Analysis}\label{sec:theory}
\vspace{-5pt}

We analyze the properties of CST theoretically. First, we prove that the minimizer of the CST loss $L_{\textup{CST}}(f_s,f_t)$ will lead to small target loss $\textup{Err}_{Q}(f_s)$ under a simple but realistic expansion assumption. Then, we further demonstrate a concrete instantiation where cycle self-training provably recovers the target ground truth, but both feature adaptation and standard self-training fail. \emph{Due to space limit, we state the main results here and defer all proof details to Appendix~\ref{sec:proof}}.

\vspace{-6pt}
\subsection{CST Provably Works under the Expansion Assumption}\label{sec:expansion}
\vspace{-5pt}

We start from a $K$-way classification model, $f: \mathcal{X} \rightarrow [0,1]^K \in \mathcal{F}$ and $\tilde {f}(x) := \arg\max_{i} f(x)_{[i]}$ denotes the prediction. Denote by $P_i$ the conditional distribution of $P$ given $y = i$. Assume the supports of $P_i$ and $P_j$ are disjoint for $i \neq j$. The definition is similar for $Q_i$. We further Assume $P(y = i) = Q(y = i)$. For any $x \in \mathcal X$, $\mathcal{N}(x)$ is defined as the \emph{neighboring set} of $x$ with a proper metric $d(\cdot,\cdot)$, $\mathcal{N}(x) = \{x':d(x,x') \leq \xi\}$. $\mathcal{N}(A) := \cup_{x \in A}\mathcal{N}(x)$. Denote the expected error on the target domain by $\text{Err}_Q(f) := \Exp_{(x,y)\sim Q}\mathbb{I}(\tilde {f}(x)\neq y)$. 

We study the CST algorithm under the \textit{expansion assumption} of the mixture distribution~\citep{wei_2021_ICLR,cai2021theory}. Intuitively, this assumption indicates that the conditional distributions $P_i$ and $Q_i$ are closely located and regularly shaped, enabling knowledge transfer from the source domain to the target domain.

\begin{definition}[\textbf{$(q,\epsilon)$-constant expansion}~\citep{wei_2021_ICLR}]\label{def:exp} We say $P$ and $Q$ satisfy $(q,\epsilon)$-constant expansion for some constant $q,\epsilon \in (0,1)$, if for any set $A \in \mathcal{X}$ and any $i \in [K]$ with $\frac{1}{2} > P_{\frac{1}{2}(P_i + Q_i)}(A) > q$, we have $P_{\frac{1}{2}(P_i + Q_i)}(\mathcal {N}(A) \backslash A) > \mathop{\min}\{\epsilon, P_{\frac{1}{2}(P_i + Q_i)}(A)\}$.
\end{definition}

Based on this expansion assumption, we consider a \emph{robustness-constrained} version of CST. Later we will show that the robustness is closely related to the uncertainty. Denote by $f_s$ the source model and $f_t$ the model trained on the target with pseudo-labels. Let $R(f_t) := P_{\frac{1}{2}(P + Q)}(\{x:\exists x'\in \mathcal {N}(x), \tilde {f}_t(x) \neq \tilde {f}_t(x')\})$ represent the robustness~\cite{wei_2021_ICLR} of $f_t$ on $P$ and $Q$. Suppose $\Exp_{(x,y)\sim Q}\mathbb{I}(\tilde {f}_s(x)\neq \tilde {f}_t(x)) \le c$ and $R(f_t) \le \rho$. The following theorem states that when $f_s$ and $f_t$ behave similarly on the target domain $Q$ and $f_t$ is robust to local changes in input, the minimizer of the cycle source error $\textup{Err}_P(f_t)$ will guarantee low error of $f_s$ on the target domain $Q$.

\begin{theorem}\label{theorem_gen}
Suppose Definition~\ref{def:exp} holds for $P$ and $Q$. For any $f_s,f_t$ satisfying $\Exp_{(x,y)\sim Q}\mathbb{I}(\tilde {f}_s(x)\neq \tilde {f}_t(x)) \le c$ and $R(f_t) \le \rho$, the expected error of $f_s$ on the target domain $Q$ is bounded,
\begin{align}
\textup{Err}_Q(f_s) \le \textup{Err}_P(f_t) + c + 2q + \frac{\rho}{\min\{\epsilon , q\}}.
\end{align}
\end{theorem}

To further relate the expected error with the CST training objective and obtain finite-sample guarantee, we use the multi-class margin loss: $l_{\gamma}(f(x),y):=\psi_{\gamma}(-\mathcal{M}(f(x),y))$, where $\mathcal{M}(v,y) = v_{[y]} - \max_{y'\neq y}v_{[y']}$ and $\psi_{\gamma}$ is the ramp function. We then extend the margin loss: $\mathcal{M}(v) = \max_{y}(v_{[y]} - \max_{y' \neq y}v_{[y']})$ (The difference between the largest and the second largest scores in $v$), and $l_{\gamma}(f_t(x),f_s(x)):=\psi_{\gamma}(-\mathcal{M}(f_t(x), \tilde{f}_s(x)))$. Further suppose $f_{[i]}$ is $L_{f}$-Lipschitz w.r.t. the metric $d(\cdot,\cdot)$ and $\tau:= 1 - 2L_{f}\xi\min\{\epsilon , q\}> 0$. Consider the following training objective for CST, denoted by $L_{\textup{CST}}(f_s,f_t)$, where $L_{\widehat P, \gamma}(f_t):=\Exp_{(x,y) \sim {\widehat P}} l_{\gamma}(f_t(x),y)$ corresponds to the cycle source loss in~\eqref{eq:overall}, $L_{\widehat Q, \gamma}(f_t, f_s) := \Exp_{(x,y) \sim {\widehat Q}} l_{\gamma}(f_t(x),f_s(x))$ is consistent with the target loss in~\eqref{eq:target}, and $\mathcal{M}(f_t(x))$ is closely related to the uncertainty of predictions in~\eqref{eqn:cst_loss}.
\begin{align}\label{eq:genalg1} 
\mathop{\min}L_{\textup{CST}}(f_s,f_t) := L_{\widehat P, \gamma}(f_t) +L_{\widehat Q, \gamma}(f_t, f_s) + \frac{1 -\Exp_{(x,y)\sim \frac{1}{2}(\widehat P+\widehat Q)}\mathcal{M}(f_t(x))}{\tau}. 
\end{align}
The following theorem shows that the minimizer of the training objective $L_{\textup{CST}}(f_s,f_t)$ guarantees low population error of $f_s$ on the target domain $Q$.
\begin{theorem}\label{theorem_gen1}
$\widehat{\mathcal{R}}(\mathcal{F}|_{\widehat P})$ denotes the empirical Rademacher complexity of function class $\mathcal{F}$ on dataset $\widehat{P}$. For any solution of~\eqref{eq:genalg1} and $\gamma > 0$, with probability larger than $1 - \delta$, 
\begin{align}
\textup{Err}_Q(f_s) &\le L_{\textup{CST}}(f_s,f_t) + 2q + \frac{4K}{\gamma} \left[\widehat{\mathcal{R}}(\mathcal{F}|_{\widehat P}) +\widehat{\mathcal{R}}(\tilde{\mathcal{F}}\times\mathcal{F}|_{\widehat Q})\right] + \frac{2}{\tau} \left[\widehat{\mathcal{R}}(\mathcal{F}|_{\widehat P}) +\widehat{\mathcal{R}}(\mathcal{F}|_{\widehat Q})\right] +\zeta,\nonumber
\end{align}
where $\zeta = O\left(\sqrt{\textup{log}(1 / \delta) / {n_s}}+\sqrt{\textup{log}(1 / \delta) / {n_t}}\right)$ is a low-order term. $\tilde{\mathcal{F}} \times \mathcal{F}$ refers to the function class $\{x \rightarrow f(x)_{[\tilde{f'}(x)]}: f,f' \in \mathcal{F}\}$. 
\end{theorem}

\textbf{Main insights.} Theorem~\ref{theorem_gen1} justifies CST under the expansion assumption. The generalization error of the classifier $f_s$ on the target domain is bounded with the CST loss objective $L_{\textup{CST}}(f_s,f_t)$, the intrinsic property of the data distribution $q$, and the complexity of the function classes. In our algorithm, $L_{\textup{CST}}(f_s,f_t)$ is minimized by the neural networks and $q$ is a constant. The complexity of the function class can be controlled with proper regularization. 

\vspace{-3pt}
\subsection{Hard Case for Feature Adaptation and Standard Self-Training}\label{sec:hard}
\vspace{-3pt}
 
To gain more insight, we study UDA in a quadratic neural network $f_{\theta,\phi}(x) =\theta^\top (\phi^\top x) ^{\odot 2}$, where $\odot$ is element-wise power. In UDA, the source can have \textit{multiple solutions} but we aim to learn the one working on the target~\citep{cite:ICML15DAN}. We design the underlying distributions $p$ and $q$ in Table~\ref{table:pq} to reflect this. Consider the following $P$ and $Q$. $x_{[1]}$ and $x_{[2]}$ are sampled \emph{i.i.d.} from distribution $p$ on $P$, and from

\begin{wraptable}{r}{4.1cm}
\addtolength{\tabcolsep}{-3pt}
\vspace{-20pt}
\label{table:pq}
\centering
\caption{\small The design of $p$ and $q$.}
\label{table:pq}
\begin{small}
\begin{tabular}{l|c|c|c}
\toprule
Distribution & $-1$ & $+1$ & $0$ \\
\midrule
Source $p$ & $0.05$ & $0.05$ & $0.90$\\
Target $q$ & $0.25$ & $0.25$ & $0.50$\\
\bottomrule
\end{tabular}
\end{small}
\vspace{-12pt}
\end{wraptable}

$q$ on $Q$. For $i\in[3,d]$, $x_{[i]}= \sigma_i x_{[2]}$ on $P$ and $x_{[i]}= \sigma_i  x_{[1]}$ on $Q$. $\sigma_i \in \{\pm 1\}$ are \emph{i.i.d.} and uniform. We also assume realizability: $y = x_{[1]}^2 - x_{[2]}^2$ for both source and target. Note that $y = x_{[1]}^2 - x_{[i]}^2$ for all $i\in[2,d]$ are solutions to $P$ but \textit{only} $y = x_{[1]}^2 - x_{[2]}^2$ works on $Q$. We visualize this specialized setting in Figure~\ref{fig:th_setting}.

\begin{figure}[h]
\vspace{-5pt}
  \centering
  \includegraphics[width=1.0\columnwidth]{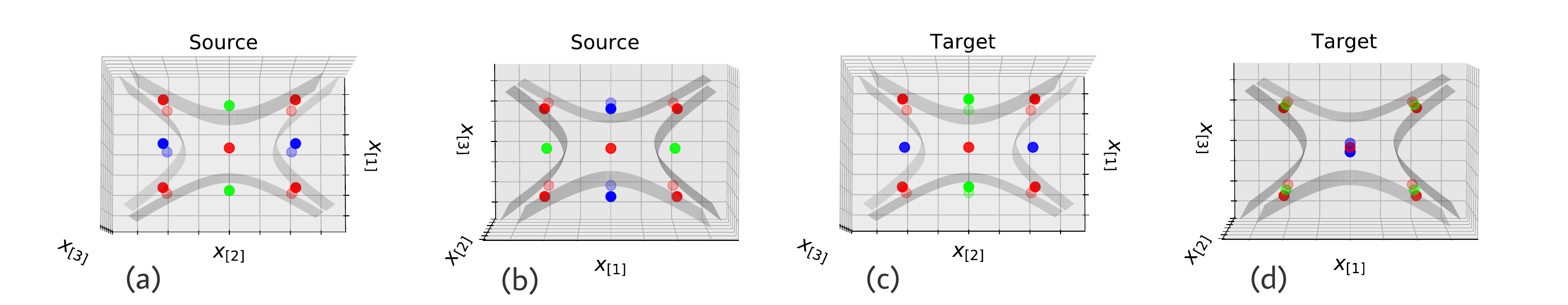}
   \vspace{-15pt}
  \caption{\small\textbf{The hard case where $d = 3$.} Green dots for $y = 1$, red dots for $y = 0$, and blue dots for $y = -1$. The grey curve is the classification boundary of different features. The good feature $x_{[1]}^2 - x_{[2]}^2$ works on the target domain (shown in (a) and (c)), whereas the spurious feature $x_{[1]}^2 - x_{[3]}^2$ only works on the source domain (shown in (b) and (d)). In Section~\ref{sec:hard}, we show that feature adaptation and standard self-training learn $x_{[1]}^2 - x_{[3]}^2$, while CST learns $x_{[1]}^2 - x_{[2]}^2$.} 
  \vspace{-3pt}
  \label{fig:th_setting}
\end{figure}

To make the features more tractable, we study the norm-constrained version of the algorithms (details are deferred to Section~\ref{algorithms}). We compare the features learned by feature adaptation, standard self-training, and CST. Intuitively, feature adaptation fails because the \emph{ideal} target solution $y = x_{[1]}^2 - x_{[2]}^2$ has larger distance in the feature space than other spurious solutions $y = x_{[1]}^2 - x_{[i]}^2$. Standard self-training also fails since it will choose randomly among all solutions. In comparison, CST can recover the ground truth, because it can distinguish the spurious solution resulting in \textit{bad pseudo-labels}. A classifier trained with those pseudo-labels \textit{cannot} work on the source domain in turn. This intuition is rigorously justified in the following two theorems.

\begin{theorem}\label{theorem_fa}
For $\epsilon\in(0,0.5)$, the following statements hold for feature adaptation and self-training:
	\begin{itemize}[leftmargin=*]
		\setlength{\itemsep}{5pt}
		\setlength{\parskip}{-3pt}
		\item{
	For failure rate $\xi>0$, and target dataset size $\nt >\Theta(\log\frac{1}{\xi})$, with probability at least $1-\xi$ over the sampling of target data, the solution $(\hat\theta_\textup{FA},\hat\phi_\textup{FA})$ found by feature adaptation satisfies
         
			\begin{align}
			\textup{Err}_Q(\hat\theta_\textup{FA},\hat\phi_\textup{FA})\ge \epsilon.
			\end{align}	}
			\vspace{-5pt}
		\item{
	With probability at least $1-\frac{1}{d-1}$, the solution $(\hat\theta_\textup{ST},\hat\phi_\textup{ST})$ of standard self-training satisfies
           
			\begin{align}
			\textup{Err}_Q(\hat\theta_\textup{ST},\hat\phi_\textup{ST})\ge \epsilon.
			\end{align}	}
	\end{itemize}
\end{theorem}

\begin{theorem}\label{theorem_transfer}
For failure rate $\xi>0$, and target dataset size $\nt >\Theta(\log\frac{1}{\xi})$, with probability at least $1-\xi$, the solution of CST ($\hat\phi_{\textup{CST}},\hat\theta_\textup{CST})$ recovers the ground truth of the target dataset: 
	
	\begin{align}
	\ \ \ \quad \textup{Err}_Q(\hat\theta_\textup{CST}, \hat\phi_{\textup{CST}})= 0.
	\end{align}
\end{theorem}

\vspace{-10pt}
\section{Experiments}\label{sec:exp}
\vspace{-5pt}

We test the performance of the proposed method on both vision and language datasets. Cycle Self-Training (CST) consistently outperforms state-of-the-art feature adaptation and self-training methods. Code is available at \url{https://github.com/Liuhong99/CST}.

\vspace{-5pt}
\subsection{Setup}
\vspace{-5pt}

\textbf{Datasets.} We experiment on visual object recognition and linguistic sentiment classification tasks:
\textit{Office-Home} \citep{cite:CVPR17OfficeHome} has $65$ classes from four kinds of environment with large domain gap: \textit{Artistic} (\textbf{Ar}), \textit{Clip Art} (\textbf{Cl}), \textit{Product} (\textbf{Pr}), and \textit{Real-World} (\textbf{Rw}); \textit{VisDA-2017}~\citep{DBLP:journals/corr/abs-1710-06924} is a large-scale UDA dataset with two domains named \textbf{Synthetic} and \textbf{Real}. The datasets consist of over 200k images from 12 categories of objects; \textit{Amazon Review}~\citep{blitzer-etal-2007-biographies} is a linguistic sentiment classification dataset of product reviews in four products: \emph{Books} (\textbf{B}), \emph{DVDs} (\textbf{D}), \emph{Electronics} (\textbf{E}), and \emph{Kitchen} (\textbf{K}).

\textbf{Implementation.} We use \textbf{ResNet-50} \citep{cite:CVPR16DRL} (pretrained on ImageNet~\citep{cite:ILSVRC15}) as feature extractors for vision tasks, and \textbf{BERT} \citep{devlin-etal-2019-bert} for linguistic tasks. On VisDA-2017, we also provide results of ResNet-101 to include more baselines.
We use cross-entropy loss for classification on the source domain. When training the target head $\hat\theta_t$ and updating the feature extractor with CST, we use squared loss to get the analytical solution of $\hat\theta_t$ directly and avoid calculating second order derivatives as meta-learning~\citep{pmlr-v70-finn17a}. Details on adapting squared loss to multi-class classification are deferred to Appendix~\ref{sec:implementation}. We adopt SGD with initial learning rate $\eta_0=2e-3$ for image classification and $\eta_0=5e-4$ for sentiment classification. Following standard protocol in \citep{cite:CVPR16DRL}, we decay the learning rate by $0.1$ each $50$ epochs until $150$ epochs. We run all the tasks $3$ times and report mean and deviation in top-1 accuracy. For VisDA-2017, we report the mean class accuracy. Following Theorem~\ref{theorem_gen1}, we also enhance CST with sharpness-aware regularization~\citep{foret2021sharpnessaware} (CST+SAM), which help regularize the Lipschitzness of the function class. Due to space limit, we report mean accuracies in Tables~\ref{Office-Home} and \ref{Sentiment} and defer standard deviation to Appendix~\ref{sec:add_exp}.

\vspace{-5pt}
\subsection{Baselines}
\vspace{-2pt}
We compare with two lines of works in domain adaptation: feature adaptation and self-training. We also compare with more complex state-of-the-arts and create stronger baselines by combining feature adaptation and self-training.

\textbf{Feature Adaptation:} DANN~\citep{cite:JMLR17DANN}, MCD~\citep{cite:CVPR18MCD}, CDAN~\citep{cite:NIPS18CDAN} (which improves DANN with pseudo-label conditioning), MDD~\citep{pmlr-v97-zhang19i} (which improves previous domain adaptation with margin theory), Implicit Alignment (IA)~\citep{pmlr-v119-jiang20d} (which improves MDD to deal with label shift).

\textbf{Self-Training.} 
We include VAT~\citep{cite:TPAMI18VAT}, MixMatch~\citep{berthelot2019mixmatch} and FixMatch~\citep{fixmatch} in the semi-supervised learning literature as self-training methods. We also compare with self-training methods for UDA:
CBST~\citep{10.1007/978-3-030-01219-9_18}, which considers class imbalance in standard self-training, and KLD~\citep{Zou_2019_ICCV}, which improves CBST with label regularization. However, these methods involve tricks specified for convolutional networks. Thus, in sentiment classification tasks where we use BERT backbones, we compare with other consistency regularization baselines: VAT~\citep{cite:TPAMI18VAT}, VAT+Entropy Minimization.

\textbf{Feature Adaptation + Self-Training.} DIRT-T~\citep{shu2018a} combines DANN, VAT, and entropy minimization. We also create more powerful baselines: CDAN+VAT+Entropy and MDD+Fixmatch.

\textbf{Other SOTA.} AFN~\citep{2020Larger} boosts transferability by large norm. STAR~\citep{lu2020stochastic} aligns domains with stochastic classifiers. SENTRY~\citep{Prabhu_2021_ICCV} selects confident examples with a committee of random augmentations.

\begin{table*}[tbp]
\vspace{-5pt}
\addtolength{\tabcolsep}{-4.8pt} 
\centering
\caption{Accuracy (\%) on {Office-Home} for unsupervised domain adaptation (\texttt{ResNet-50}).}
\label{Office-Home}
\begin{small}
\begin{tabular}{l|cccccccccccc|cr}
\toprule
Method & Ar-Cl&Ar-Pr&Ar-Rw&Cl-Ar&Cl-Pr&Cl-Rw&Pr-Ar&Pr-Cl&Pr-Rw&Rw-Ar&Rw-Cl&Rw-Pr&Avg. \\
\midrule
DANN \citep{cite:JMLR17DANN} &45.6&59.3&70.1&47.0&58.5&60.9&46.1&43.7&68.5&63.2&51.8&76.8&57.6\\
CDAN~\citep{cite:NIPS18CDAN} & 50.7 & 70.6 & 76.0 & 57.6 & 70.0 & 70.0 & 57.4 & 50.9 & 77.3 & 70.9 & 56.7 & 81.6 & 65.8\\
CDAN+VAT+Entropy & 52.2 & 71.5 & 76.4 & 61.1 & 70.3 & 67.8 & 59.5 & 54.4 & 78.6 & 73.2 & 59.0 & 82.7 & 67.3\\
FixMatch~\citep{fixmatch} & 51.8 & 74.2 & \underline{80.1} & 63.5 & \underline{73.8} & 61.3 & 64.7 & 51.4 & 80.0 & 73.3 & 56.8 & 81.7 & 67.7 \\
MDD~\citep{pmlr-v97-zhang19i} & 54.9 & 73.7 & 77.8 & 60.0 & 71.4 & 71.8 & 61.2 & 53.6 & 78.1 & 72.5 & 60.2 & 82.3 & 68.1\\
MDD+IA~\citep{pmlr-v119-jiang20d} & 56.2 & 77.9 & 79.2 & 64.4 & 73.1 & 74.4 & 64.2 & 54.2 & 79.9 & 71.2 & 58.1 & 83.1 & 69.5\\
SENTRY~\citep{Prabhu_2021_ICCV} & \textbf{61.8} & \underline{77.4} & \underline{80.1} & \underline{66.3} & 71.6 & \underline{74.7} & \underline{66.8} & \textbf{63.0} & \underline{80.9} & \underline{74.0} & \textbf{66.3} & \underline{84.1} & 72.2\\
\midrule

\textbf{CST} & \underline{59.0} & \textbf{79.6} & \textbf{83.4} & \textbf{68.4} & \textbf{77.1} & \textbf{76.7} & \textbf{68.9} & \underline{56.4} & \textbf{83.0} & \textbf{75.3} & \underline{62.2} & \textbf{85.1} &\textbf{73.0}\\
\bottomrule
\end{tabular}
\end{small}
\vskip -0.1in
\end{table*}

\begin{table*}[tbp]
\vspace{-5pt}
\addtolength{\tabcolsep}{-1.0pt} 
\caption{Accuracy (\%) on {Multi-Domain Sentiment Dataset} for domain adaptation with \texttt{BERT}.}
\label{Sentiment}
\begin{small}
\begin{tabular}{l|cccccccccccc|cr}
\toprule
Method & B-D&B-E&B-K&D-B&D-E&D-K&E-B&E-D&E-K&K-B&K-D&K-E&Avg. \\
\midrule
Source-only & 89.7 & 88.4 & 90.9 & 90.1 & 88.5 & 90.2 & 86.9 & \underline{88.5} & 91.5 & 87.6 & 87.3 & 91.2 & 89.2\\
DANN~\citep{cite:JMLR17DANN} & 90.2 & 89.5 & 90.9 & \underline{91.0} & 90.6 & 90.2 & 87.1 & 87.5 & \underline{92.8} & 87.8 & 87.6 & \underline{93.2} & 89.9\\
VAT~\citep{cite:TPAMI18VAT} & \underline{90.6} & 91.0 & 91.7 & 90.8 & 90.8 & 92.0 & 87.2 & 86.9 & 92.6 & 86.9 & 87.7 & 92.9 & 90.1\\
VAT+Entropy & 90.4 & \underline{91.3} & 91.5 & \underline{91.0} & \underline{91.1} & \underline{92.4} & \underline{87.5} & 86.3 & 92.4 & 86.5 & 87.5 & 93.1 & 90.1\\
MDD~\citep{pmlr-v97-zhang19i} & 90.4 & 90.4 & \underline{91.8} & 90.2 & 90.9 & 91.0 & \underline{87.5} & 86.3 & 92.5 & \textbf{89.0} & \underline{87.9} & 92.1 & 90.0\\
\midrule
\textbf{CST} & \textbf{91.5} & \textbf{92.9} & \textbf{92.6} & \textbf{91.9} & \textbf{92.6} & \textbf{93.5} & \textbf{90.2} & \textbf{89.4} & \textbf{93.8} & \underline{87.9} & \textbf{88.3} & \textbf{93.5} & \textbf{91.5}\\
\bottomrule
\end{tabular}
\end{small}
\end{table*}
\begin{table}[t]
\vskip -0.15in
\addtolength{\tabcolsep}{0.3pt} 
\centering 
\caption{Mean Class Accuracy (\%) for unsupervised domain adaptation on VisDA-2017.}
\vspace{-5pt}
\label{visda}
\begin{small}
\begin{tabular}{l|rr|l|rr}
\toprule
Method & \texttt{ResNet-50} & \texttt{ResNet-101} & Method & \texttt{ResNet-50} & \texttt{ResNet-101}\\
\midrule
DANN~\citep{cite:JMLR17DANN} & 69.3 & 79.5 &
CBST~\citep{10.1007/978-3-030-01219-9_18} & -- & 76.4 $\pm$ 0.9\\
VAT~\citep{cite:TPAMI18VAT} & 68.0 $\pm$ 0.3 & 73.4 $\pm$ 0.5 &
KLD~\citep{Zou_2019_ICCV} & -- & 78.1 $\pm$ 0.2\\
DIRT-T~\citep{shu2018a} & 68.2 $\pm$ 0.3 & 77.2 $\pm$ 0.5 &
MDD~\citep{pmlr-v97-zhang19i} &{74.6} & 81.6 $\pm$ 0.3\\
MCD~\citep{cite:CVPR18MCD} &69.2 & 77.7 &
AFN~\citep{2020Larger} & -- & 76.1\\
CDAN~\citep{cite:NIPS18CDAN} & 70.0 & 80.1 &
MDD+IA~\citep{pmlr-v119-jiang20d} & 75.8 & --\\
CDAN+VAT+Entropy &{76.5} $\pm$ 0.5& 80.4 $\pm$ 0.7 &
MDD+FixMatch &{77.8} $\pm$ 0.3& 82.4 $\pm$ 0.4 \\
MixMatch &{69.3} $\pm$ 0.4 & 77.0 $\pm$ 0.5 &STAR~\citep{lu2020stochastic} & -- & 82.7 \\
FixMatch~\citep{fixmatch} & 74.5 $\pm$ 0.2 & 79.5 $\pm$ 0.3 & SENTRY~\citep{Prabhu_2021_ICCV} & 76.7 & -- \\
\midrule
\textbf{CST} & \underline{79.9} $\pm$ 0.5 & \underline{84.8} $\pm$ 0.6 & \textbf{CST+SAM} & \textbf{80.6} $\pm$ 0.5 & \textbf{86.5} $\pm$ 0.7 \\
\bottomrule
\end{tabular}
\end{small}
\vskip -0.15in
\end{table}

\vspace{-5pt}
\subsection{Results}
\vspace{-5pt}

Results on 12 pairs of \emph{Office-Home} tasks are shown in Table \ref{Office-Home}. When domain shift is large, standard self-training methods such as VAT and FixMatch suffer from the decay in pseudo-label quality. \textbf{CST} outperforms feature adaptation and self-training methods significantly in 9 out of 12 tasks. Note that CST does not involve manually setting confidence threshold or reweighting.

Table~\ref{visda} shows the results on \emph{VisDA-2017}. \textbf{CST} surpasses state-of-the-arts with ResNet-50 and ResNet-101 backbones. We also combine feature adaptation and self-training (DIRT-T, CDAN+VAT+entropy and MDD+FixMatch) to test if feature adaptation alleviates the negative effect of domain shift in standard self-training. Results indicate that CST is a better solution than simple combination.

While most traditional self-training methods include techniques specified for ConvNets such as Mixup~\citep{zhang2018mixup}, \textbf{CST} is a \emph{universal} method and can directly work on sentiment classification by simply replacing the head and training objective of BERT~\citep{devlin-etal-2019-bert}. In Table~\ref{Sentiment}, most feature adaptation baselines improve over source only marginally, but \textbf{CST} outperforms all baselines on most tasks significantly.

\vspace{-5pt}
\subsection{Analysis}\label{analysis}
\vspace{-5pt}

\begin{wraptable}{r}{5.7cm}
\addtolength{\tabcolsep}{0pt} 
\vspace{-22pt}
\centering 
\caption{\small Ablation on VisDA-2017.}
\label{ablation}
\begin{small}
\begin{tabular}{l|r|r}
\toprule
Method & Accuracy $\uparrow$ & $d_{\textup{TV}}$ $\downarrow$ \\
\midrule
FixMatch~\citep{fixmatch} & 74.5 $\pm$ 0.2 & 0.22 \\
Fixmatch+Tsallis & 76.3 $\pm$ 0.8 & 0.15\\
CST w/o Tsallis & 72.0 $\pm$ 0.4 & 0.16 \\
CST+Entropy &76.2 $\pm$ 0.6 & 0.20 \\
\midrule
\textbf{CST} &\textbf{79.9} $\pm$ 0.5 & 0.12 \\  
\bottomrule
\end{tabular}
\end{small}
\vspace{-10pt}
\end{wraptable} 

\textbf{Ablation Study.} We study the role of each part of CST in self-training. CST w/o Tsallis removes the Tsallis entropy $L_{\textup{Tsallis},\alpha}$. CST+Entropy replaces the Tsallis entropy with standard entropy. FixMatch+Tsallis adds $L_{\textup{Tsallis},\alpha}$ to standard self-training. Observations are shown in Table~\ref{ablation}. CST+Entropy performs $3.7\%$ worse than CST, indicating that Tsallis entropy is a better regularization for pseudo-labels than standard entropy. CST performs $5.4\%$ better than FixMatch, indicating that CST is better adapted to domain shift than standard self-training. While FixMatch+Tsallis outperforms FixMatch, it is still $3.6\%$ behind CST, with much larger total variation distance $d_\textup{TV}$ between pseudo-labels and ground-truths, indicating that CST makes pseudo-labels more reliable than standard self-training under domain shift.

\begin{figure*}[t]
  \centering 
   \subfigure{
    \includegraphics[width=0.45\columnwidth]{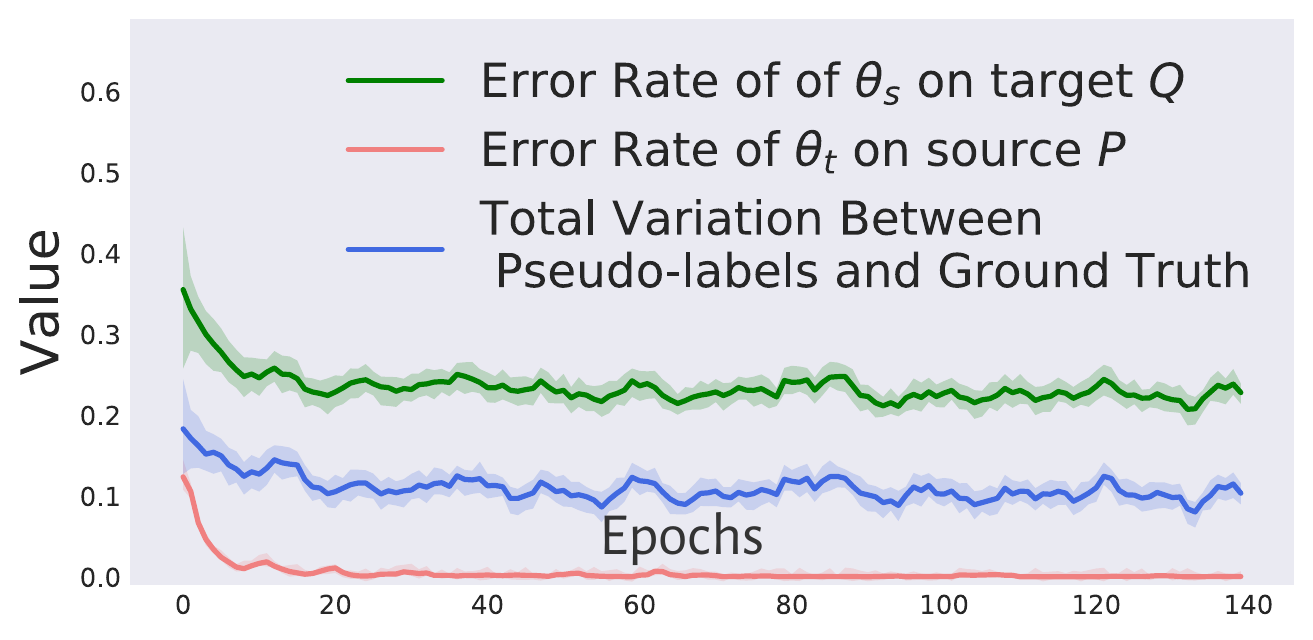}
    }
   \hfil
   \subfigure{
    \includegraphics[width=0.44\columnwidth]{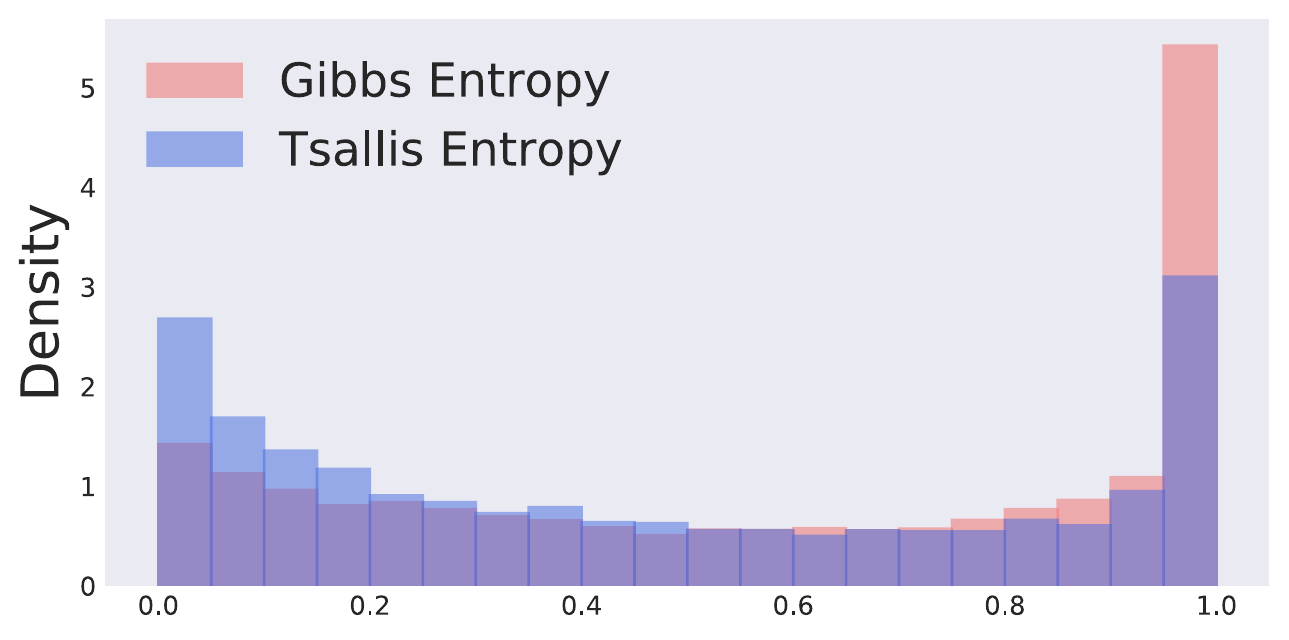}
    }
  \vspace{-10pt}
  \caption{\small\textbf{Analysis.} Left: Error of pseudo-labels and reverse pseudo-labels. The error of target classifier $\theta_t$ on the source domain decreases, indicating the quality of pseudo-labels is refined. Right: Histograms of the difference between the largest and the second largest softmax scores. Tsallis entropy avoids over-confidence.}
  \label{fig:rst_epoch}
  \vspace{-10pt}
\end{figure*}

\textbf{Quality of Pseudo-labels.} We visualize the error of pseudo-labels during training on VisDA-2017 in Figure~\ref{fig:rst_epoch} (Left). The error of target classifier $\theta_t$ on the source domain decreases quickly in training, when both the error of pseudo-labels (error of $\theta_s$ on $Q$) and the total variation (TV) distance between pseudo-labels and ground-truths continue to decay, indicating that CST gradually refines pseudo-labels. This forms a clear contrast to standard self-training as visualized in Figure~\ref{fig:pseudo} (Middle), where the distance $d_{\textup{TV}}$ remains nearly unchanged throughout training.

\textbf{Comparison of Gibbs entropy and Tsallis entropy.} We compare the pseudo-labels learned with standard Gibbs entropy and Tsallis entropy on Ar$\rightarrow$Cl with ResNet-50 at epoch 40. We compute the difference between \textit{the largest and the second largest softmax scores} of each target example and plot the histogram in Figure~\ref{fig:rst_epoch} (Right). Gibbs entropy makes the largest softmax output close to 1, indicating over-confidence. In this case, if the prediction is wrong, it can be hard to correct it using self-training. In contrast, Tsallis entropy allows the largest and the second largest scores to be similar.

\vspace{-5pt}
\section{Related Work}
\vspace{-5pt}

\textbf{Self-Training.} Self-training is a mainstream technique for semi-supervised learning~\citep{cite:Book06SSL}. In this work, we focus on pseudo-labeling~\citep{4129456,article-pl,DBLP:journals/corr/abs-1908-02983}, which uses unlabeled data by training on pseudo-labels generated by a source model. Other lines of work study consistency regularization~\citep{NIPS2014_66be31e4,NIPS2015_378a063b,NIPS2016_30ef30b6,cite:TPAMI18VAT}. Recent works demonstrate the power of such methods~\citep{2020Self,fixmatch,ghiasi2021multi}. Equipped with proper training techniques, these methods can achieve comparable results as standard training that uses much more labeled examples~\citep{du-etal-2021-self}. \citet{NEURIPS2020_27e9661e} compare self-training to pre-training and joint training. \citet{vu2021strata,NEURIPS2020_f23d125d} show that task-level self-training works well in few-shot learning. These methods are tailored to semi-supervised learning or general representation learning and do not take domain shift into consideration explicitly. \citet{wei_2021_ICLR,frei2021self} provide the first nice theoretical analysis of self-training based on the expansion assumption.

\textbf{Domain Adaptation.} Inspired by the generalization error bound of~\citet{cite:ML10DAT}, \citet{cite:ICML15DAN,cite:ICLR17CMD} minimize distance measures between source and target distributions to learn domain-invariant features. \citet{cite:JMLR17DANN} (DANN) proposed to approximate the domain distance by adversarial learning. Follow-up works proposed various improvement upon DANN~\citep{cite:CVPR17ADDA,cite:CVPR18MCD,cite:NIPS18CDAN,pmlr-v97-zhang19i,pmlr-v119-jiang20d}. Popular as they are, failure cases exist in situation like label shift~\citep{pmlr-v97-zhao19a,Li2020RethinkingDM}, shift in support of domains~\citep{pmlr-v89-johansson19a}, and large discrepancy between source and target~\citep{pmlr-v97-liu19b}. Another line of works try to address domain adaptation with self-training. \citet{shu2018a} improves DANN with VAT and entropy minimization. \citet{french2018selfensembling,Zou_2019_ICCV,Li2020RethinkingDM} incorporated various semi-supervised learning techniques to boost domain adaptation performance. \citet{pmlr-v119-kumar20c}, \citet{NEURIPS2020_f1298750} and \citet{cai2021theory} showed self-training provably works in domain adaptation under certain assumptions.

\vspace{-5pt}
\section{Conclusion}
\vspace{-5pt}

We propose cycle self-training in place of standard self-training to explicitly address the distribution shift in domain adaptation. We show that our method provably works under the expansion assumption and demonstrate hard cases for feature adaptation and standard self-training. Self-training (or pseudo-labeling) is only one line of works in the semi-supervised learning literature. Future work can delve into the behaviors of other semi-supervised learning techniques including consistency regularization and data augmentation under distribution shift, and exploit them extensively for domain adaptation.

\vspace{-5pt}
\section*{Acknowledgements}
\vspace{-5pt}

This work was supported by the National Natural Science Foundation of China under Grants 62022050 and 62021002, Beijing Nova Program under Grant Z201100006820041, China's Ministry of Industry and Information Technology, the MOE Innovation Plan and the BNRist Innovation Fund.

\bibliography{example_paper}
\bibliographystyle{icml2021}

\onecolumn
\newpage
\appendix
\section{Details in Section~\ref{sec:theory}}\label{sec:proof}
\subsection{Proof of Theorem~\ref{theorem_gen}}
In Section~\ref{sec:expansion}, we study CST theoretically. In Theorem~\ref{theorem_gen}, we show that when the population error of the target classifier $f_t$ on the source domain $P$ is low and $f_t$ is locally consistent, the source classifier $f_s$ is guaranteed to perform well on the target domain $Q$. We further show that the consistency (robustness) is guaranteed by the confidence of the model (Lemma~\ref{lem:rob_ent}). Finally, we show in Theorem~\ref{theorem_gen1} that the minimizer of an objective function consistent with the CST objective in Section~\ref{sec:cst} leads to small target loss of the source classifier $\textup{Err}_{Q}(f_s)$.

We first review the assumptions made in Section~\ref{sec:expansion} in order to prove Theorem~\ref{theorem_gen}. Consider a $K$-way classification problem. $f: X \rightarrow [0,1]^K \in \mathcal{F}$ and $\tilde {f}(x) := \arg\max_{i} f(x)_{[i]}$. We first state the properties of source and target distributions $P$ and $Q$. Assume the source and target distributions are composed of $K$ sub-populations, each corresponding to one class, and the sub-populations of different classes have disjoint support. This indicates that a ground truth labeling function exists, which is a common assumption as in~\citep{cite:ML10DAT,wei_2021_ICLR,cai2021theory}. We also assume for simplicity of presentation that $P(y = i) = Q(y = i)$. Note that our techniques can be directly applied to the case where $\frac{P(y = i)}{Q(y = i)}$ is bounded as in~\citep{cai2021theory}.

\begin{assumption}\label{ass:disjoint}Denote by $P_i$ and $Q_i$ the conditional distribution of $P$ and $Q$ given $y = i$. We assume that: (1) $P(y = i) = Q(y = i)$, and (2) the supports of $P_i$ and $P_j$ are disjoint for $i \neq j$.
\end{assumption}

Our analysis relies on the \textit{expansion assumption}~\citep{wei_2021_ICLR,cai2021theory}, which intuitively states that the data distribution has good continuity within each class. Therefore, the subset in the support of a class will connect to its neighborhood, enabling knowledge transfer between domains. \citet{wei_2021_ICLR} justifies this assumption on real-world datasets with BigGAN.

\begin{assumption}[\textbf{$(q,\epsilon)$-constant expansion}~\citep{wei_2021_ICLR}]\label{ass:exp} For any $x \in \mathcal X$, $\mathcal{N}(x)$ is defined as the \emph{neighboring set} of $x$, $\mathcal{N}(x) = \{x':d(x,x') \leq \xi\}$, where $d$ is a proper metric. $\mathcal{N}(A) := \cup_{x \in A}\mathcal{N}(x)$. We say $P$ and $Q$ satisfy $(q,\epsilon)$-constant expansion for some constants $q,\epsilon \in (0,1)$, if for any set $A \in \mathcal{X}$ and any $i \in [K]$ with $\frac{1}{2} > P_{\frac{1}{2}(P_i + Q_i)}(A) > q$, we have $P_{\frac{1}{2}(P_i + Q_i)}(\mathcal {N}(A) \backslash A) > \mathop{\min}\{\epsilon, P_{\frac{1}{2}(P_i + Q_i)}(A)\}$.
\end{assumption}

Based on this expansion assumption, we consider a \emph{robustness-constrained} version of CST for now. In Theorem~\ref{theorem_gen1}, we will show that the population robustness loss is closely related to the uncertainty of of CST. Denote by $f_s$ the source model and $f_t$ the model trained on the target with pseudo-labels. Let $R(f_t) := P_{\frac{1}{2}(P + Q)}(\{x:\exists x'\in \mathcal {N}(x), \tilde {f}_t(x) \neq \tilde {f}_t(x')\})$ represent the robustness~\cite{wei_2021_ICLR} of $f_t$ on $P$ and $Q$. Suppose $\Exp_{(x,y)\sim Q}\mathbb{I}(\tilde {f}_s(x)\neq \tilde {f}_t(x)) \le c$ and $R(f_t) \le \rho$. Theorem~\ref{theorem_gen} states that when $f_s$ and $f_t$ behave similarly on $Q$ ($f_t$ fits the pseudo-labels generated by $f_s$ on the target domain) and $f_t$ is robust to local changes in input, the minimizer of the cycle source error $\textup{Err}_P(f_t)$ will guarantee low error on the target domain $Q$.

\begin{theorem1}
Suppose Assumption~\ref{ass:disjoint} and Assumption~\ref{ass:exp} hold for $P$ and $Q$. For any $f_s,f_t$ satisfying $\Exp_{(x,y)\sim Q}\mathbb{I}(\tilde {f}_s(x)\neq \tilde {f}_t(x)) \le c$ and $R(f_t) \le \rho$, the expected error of $f_s$ on the target domain $Q$ is bounded,
\begin{align}
\textup{Err}_Q(f_s) \le \textup{Err}_P(f_t) + c + 2q + \rho \ / \ \min\{\epsilon , q\}.
\end{align}
\end{theorem1}

We now turn to the proof of Theorem~\ref{theorem_gen}. We want to show the error of $f_t$ on the source domain $P$ is close to the error of $f_s$ on the target domain $Q$. We first show that when the robustness error $R(f_t)$ is controlled, the error of $f_t$ on the source and the target will be close. This is done by analyzing the error on each sub-population $P_i$ and $Q_i$ separately. Then we use the fact that the losses of $f_s$ and $f_t$ are also close when their disagreement on the target domain is controlled to obtain the final result.
\begin{lemma}[Robustness on sub-populations]\label{lem:robust_sub}
Divide $[K]$ into $S_1$ and $S_2$, where for every $i \in S_1$, $\Exp_{(x,y) \sim \frac{1}{2}(P_i + Q_i)}\mathbb{I} (\exists x'\in \mathcal {N}(x), \tilde {f}_t(x) \neq \tilde {f}_t(x')\}) < \min\{\epsilon, q\}$, and for every $i \in S_2$, $\Exp_{(x,y) \sim \frac{1}{2}(P_i + Q_i)}\mathbb{I} (\exists x'\in \mathcal {N}(x), \tilde {f}_t(x) \neq \tilde {f}_t(x')\}) \ge \min\{\epsilon, q\}$. Under the condition of Theorem~\ref{theorem_gen}, we have
\begin{align}
\sum_{i \in S_1} P(y = i) \ge 1 - \frac{\rho}{\min\{\epsilon, q\}}.
\end{align}
\end{lemma}
\begin{proof}[Proof of Lemma~\ref{lem:robust_sub}]
Suppose $\sum_{i \in S_1} P(y = i) < 1 - \frac{\rho}{\min\{\epsilon, q\}}$. Then we have $\sum_{i \in S_2} P(y = i) > \frac{\rho}{\min\{\epsilon, q\}}$, which implies
\begin{align*}
&\Exp_{(x,y) \sim \frac{1}{2}(P + Q)}\mathbb{I} (\exists x'\in \mathcal {N}(x), \tilde {f}_t(x) \neq \tilde {f}_t(x')\})\\
 &= \sum_{i \in [K]}\Exp_{(x,y) \sim \frac{1}{2}(P_i + Q_i)}\mathbb{I} (\exists x'\in \mathcal {N}(x), \tilde {f}_t(x) \neq \tilde {f}_t(x')\}) P(y = i)\\
&\ge \sum_{i \in S_2} \Exp_{(x,y) \sim \frac{1}{2}(P_i + Q_i)}\mathbb{I} (\exists x'\in \mathcal {N}(x), \tilde {f}_t(x) \neq \tilde {f}_t(x')\}) P(y = i) \\
&> \min\{\epsilon, q\} \sum_{i \in S_2} P(y = i)\\
&= \rho.
\end{align*}
Since we have $R(f_t) = \Exp_{(x,y) \sim \frac{1}{2}(P + Q)}\mathbb{I} (\exists x'\in \mathcal {N}(x), \tilde {f}_t(x) \neq \tilde {f}_t(x')\}) < \rho$, this forms a contradiction.
\end{proof}
We have established that for a large proportion of the sub-populations, the robustness is guaranteed. The next lemma shows that for each sub-population where the robustness is guaranteed, $\textup{Err}_{P_i}(f_t)$ and $\textup{Err}_{Q_i}(f_t)$ is close to each other by invoking the expansion assumption~\cite{wei_2021_ICLR}.

\begin{lemma}[Accuracy propagates on robust sub-populations]\label{lem:robust_sub1}
Under the condition of Theorem~\ref{theorem_gen}, if the sub-populations $P_i$ and $Q_i$ satisfy $\Exp_{(x,y) \sim \frac{1}{2}(P_i + Q_i)}\mathbb{I} (\exists x'\in \mathcal {N}(x), \tilde {f}_t(x) \neq \tilde {f}_t(x')\}) < \min\{\epsilon, q\}$, we have
\begin{align}
\left|\textup{Err}_{P_i}(f_t) - \textup{Err}_{Q_i}(f_t)\right| \le 2q.
\end{align}
\end{lemma}

\begin{proof}[Proof of Lemma~\ref{lem:robust_sub1}]
We claim that either $\textup{Err}_{\frac{1}{2}(P_i+ Q_i)}(f_t) \le q$ or $\textup{Err}_{\frac{1}{2}(P_i+ Q_i)}(f_t) \ge 1 - q$. On the one hand, if $\frac{1}{2} > \textup{Err}_{\frac{1}{2}(P_i+ Q_i)}(f_t) > q$, by the $(q,\epsilon)$-expansion property (Definition~\ref{def:exp}), $P_{\frac{1}{2}(P_i+ Q_i)}(\mathcal{N}(\{x:\tilde {f}_t(x) \neq i\})\backslash \{x:\tilde {f}_t(x) \neq i\}) > \min\{\epsilon, q\}$. Note that in $\mathcal{N}(\{x:\tilde {f}_t(x) \neq i\})\backslash \{x:\tilde {f}_t(x) \neq i\}$, $\tilde {f}_t(x) = i$. Thus, for $x$ in the set $\mathcal{N}(\{x:\tilde {f}_t(x) \neq i\})\backslash \{x:\tilde {f}_t(x) \neq i\}$, there exists $x' \in \mathcal{N}(x)$, $\tilde {f}_t(x') \neq \tilde {f}_t(x) = i$.
\begin{align*}
R(f_t) &= \Exp_{(x,y) \sim \frac{1}{2}(P_i + Q_i)} \mathbb{I} (\exists x'\in \mathcal {N}(x), \tilde {f}_t(x) \neq \tilde {f}_t(x')\})\\
&\ge \Exp_{(x,y) \sim \frac{1}{2}(P_i + Q_i)} \mathbb{I} (\exists x'\in \mathcal {N}(x), \tilde {f}_t(x) \neq \tilde {f}_t(x')\})\mathbb{I}(x \in \mathcal{N}(\{x:\tilde {f}_t(x) \neq i\})\backslash \{x:\tilde {f}_t(x) \neq i\})\\
&= P_{\frac{1}{2}(P_i+ Q_i)}(\mathcal{N}(\{x:\tilde {f}_t(x) \neq i\})\backslash \{x:\tilde {f}_t(x) \neq i\})\\
&>  \min\{\epsilon, q\}, 
\end{align*}
which contradicts the condition that $R(f_t) < \min\{\epsilon, q\}$. 

On the other hand, if $\frac{1}{2} \le \textup{Err}_{\frac{1}{2}(P_i+ Q_i)}(f_t) < 1 - q$, the argument is similar. By the $(q,\epsilon)$-expansion property (Definition~\ref{def:exp}), $P_{\frac{1}{2}(P_i+ Q_i)}(\mathcal{N}(\{x:\tilde {f}_t(x) = i\})\backslash \{x:\tilde {f}_t(x) = i\}) > \min\{\epsilon, q\}$. Note that in $\mathcal{N}(\{x:\tilde {f}_t(x) = i\})\backslash \{x:\tilde {f}_t(x) = i\}$, $\tilde {f}_t(x) \neq i$. Thus, for $x$ in the set $\mathcal{N}(\{x:\tilde {f}_t(x) = i\})\backslash \{x:\tilde {f}_t(x) = i\}$, there exists $x' \in \mathcal{N}(x)$, $i = \tilde {f}_t(x') \neq \tilde {f}_t(x)$.
\begin{align*}
R(f_t) &= \Exp_{(x,y) \sim \frac{1}{2}(P_i + Q_i)} \mathbb{I} (\exists x'\in \mathcal {N}(x), \tilde {f}_t(x) \neq \tilde {f}_t(x')\})\\
&\ge \Exp_{(x,y) \sim \frac{1}{2}(P_i + Q_i)} \mathbb{I} (\exists x'\in \mathcal {N}(x), \tilde {f}_t(x) \neq \tilde {f}_t(x')\})\mathbb{I}(x \in \mathcal{N}(\{x:\tilde {f}_t(x) = i\})\backslash \{x:\tilde {f}_t(x) = i\})\\
&= P_{\frac{1}{2}(P_i+ Q_i)}(\mathcal{N}(\{x:\tilde {f}_t(x) = i\})\backslash \{x:\tilde {f}_t(x) = i\})\\
&>  \min\{\epsilon, q\}, 
\end{align*}
which also contradicts the condition that $R(f_t) < \min\{\epsilon, q\}$.

Note that $\textup{Err}_{\frac{1}{2}(P_i+ Q_i)}(f_t) = \frac{1}{2}\textup{Err}_{P_i}(f_t) + \frac{1}{2}\textup{Err}_{Q_i}(f_t)$. Also we have $\textup{Err}_{P_i}(f_t) \in [0,1]$. In consequence, we have either $\textup{Err}_{P_i}(f_t),\textup{Err}_{Q_i}(f_t) \in [0,2q]$ or $\textup{Err}_{P_i}(f_t),\textup{Err}_{Q_i}(f_t) \in [1 - 2q, 1]$, which completes the proof.
\end{proof}

With Lemma~\ref{lem:robust_sub} and Lemma~\ref{lem:robust_sub1} at hand, we can prove Theorem~\ref{theorem_gen} by putting the analysis on each sub-population together.
\begin{proof}[Proof of Theorem~\ref{theorem_gen}]
\begin{align*}
\textup{Err}_{Q}(f_t) &= \sum_{i \in [K]}\textup{Err}_{Q_i}(f_t)P(y = i)\\
&\le  \sum_{i \in [S_1]}\textup{Err}_{Q_i}(f_t)P(y = i) + \sum_{i \in [S_2]}P(y = i)\\
&\le  \sum_{i \in [S_1]}(\textup{Err}_{P_i}(f_t) + 2q)P(y = i) + \sum_{i \in [S_2]}P(y = i)\\
&\le \textup{Err}_{P}(f_t) + 2q + \frac{\rho}{\min\{\epsilon, q\}},
\end{align*}
where the second inequality holds due to Lemma~\ref{lem:robust_sub1}, and the last holds due to Lemma~\ref{lem:robust_sub}. Also note that $\textup{Err}_{Q}(f_s) \le \textup{Err}_{Q}(f_t) + \Exp_{(x,y)\sim Q}\mathbb{I}(\arg\max_{[i]} f_s(x)_{[i]}\neq \arg\max_{[i]}f_t(x)_{[i]}) $ by the triangle inequality. Adding these two equations results in Theorem~\ref{theorem_gen}.
\end{proof}

\subsection{Proof of Theorem~\ref{theorem_gen1}}
To obtain finite-sample guarantee, we need additional assumptions on the function class $\mathcal{F}$. 
\begin{assumption}\label{ass:fun} The function class $\mathcal{F}$ satisfies the following properties:
(1) $\mathcal{F}$ is closed to permutations of coordinates, (2) $0 \in \mathcal{F}$, and (3) each coordinate of $f$ is $L_f$-Lipschitz w.r.t. $d(\cdot,\cdot)$.
\end{assumption}
This assumption is also standard since common models for multi-class classification are symmetric for each class. Setting all the weight parameters of neural networks to $0$ will result in $0$ output.

We review the definition of terms in Theorem~\ref{theorem_gen1}. The ramp function $\psi_{\gamma}:\Real \rightarrow [0,1]$ is defined
as:
\begin{align}\label{def:ramp}
\psi_{\gamma}(x) = 
\left\{
             \begin{array}{lr}
             1, & x \le 0   \\
             1 - \frac{x}{\gamma}, & 0 < x \le \gamma\\
             0, & x > \gamma  
             \end{array}
\right.
\end{align}
The margin function is defined as $\mathcal{M}(v,y) = v_{[y]} - \max_{y'\neq y}v_{[y']}$ and $\mathcal{M}(v) = \max_y \{v_{[y]} - \max_{y'\neq y}v_{[y']}\}$. For multi-class classification problems, $\mathcal{M}(v)$ is closely related to the confidence, since it is equal to the difference between the largest and the second largest scores. The multi-class margin loss is composed of $\psi_{\gamma}(x)$ and $\mathcal{M}$: $l_{\gamma}(f(x),y):=\psi_{\gamma}(-\mathcal{M}(f(x),y))$. Denote by $L_{\widehat P, \gamma}(f_t)$ the empirical margin loss of $f_t$ on the source dataset $\widehat P$, $L_{\widehat P, \gamma}(f_t) = \Exp_{(x,y)\sim \widehat P}l_{\gamma}(f_t(x),y)$. To measure the inconsistency of $f_s$ and $f_t$, we extend the multi-class margin loss as $l_{\gamma}(f_s(x),f_t(x)):=\psi_{\gamma}(-\mathcal{M}(f_s(x), \tilde{f_t}(x)))$. Denote by $L_{\widehat P, \gamma}(f_t, f_s)$ the empirical margin inconsistency loss of $f_t$ and $f_s$ on the source dataset $\widehat P$, $L_{\widehat P, \gamma}(f_t, f_s) = \Exp_{(x,y)\sim \widehat P}l_{\gamma}(f_t(x), f_s(x))$.

Consider minimizing the following objective:
\begin{align}\label{eq:genalg3} \mathop{\min}L_{\textup{CST}}(f_s,f_t):=\underbrace{L_{\widehat P, \gamma}(f_t)}_{\textup{Cycle Loss}} + \underbrace{L_{\widehat P, \gamma}(f_t, f_s)}_{\textup{Target Loss}} + \underbrace{\nicefrac{1 -\Exp_{(x,y)\sim \frac{1}{2}(\widehat P+\widehat Q)}\mathcal{M}(f_t(x))}{\tau}}_{\textup{Uncertainty Loss}}.
\end{align}
Note that $L_{\widehat P, \gamma}(f_t)$ is the loss of $f_t$ on the source dataset (the cycle loss), and $L_{\widehat P, \gamma}(f_t, f_s)$ is the training error of $f_t$ on the target dataset. $\mathcal{M}(f_t(x))$ equals the difference between the largest and the second largest scores of $f_t(x)$, indicating the confidence of $f_t$. Thus, $1 -\Exp_{(x,y)\sim \frac{1}{2}(\widehat P+\widehat Q)}\mathcal{M}(f_t(x))$ is the uncertainty of $f_t$ on the source and target datasets.

The following theorem shows that the minimizer of the training objective $L_{\textup{CST}}(f_s,f_t)$ guarantees low population error of $f_s$ on the target domain $Q$.
\begin{theorem2}
Under the condition of Theorem~\ref{theorem_gen} and Assumption~\ref{ass:fun}. For any solution of~\eqref{eq:genalg1} and $\gamma > 0$, with probability larger than $1 - \delta$, 
\begin{align}
\textup{Err}_Q(f_s) &\le L_{\textup{CST}}(f_s,f_t) + 2q + \frac{4K}{\gamma} \left[\widehat{\mathcal{R}}(\mathcal{F}|_{\widehat P}) +\widehat{\mathcal{R}}(\tilde{\mathcal{F}}\times\mathcal{F}|_{\widehat Q})\right] + \frac{2}{\tau} \left[\widehat{\mathcal{R}}(\mathcal{F}|_{\widehat P}) +\widehat{\mathcal{R}}(\mathcal{F}|_{\widehat Q})\right] +\zeta,\nonumber
\end{align}
where $\zeta = O\left(\sqrt{\textup{log}(1 / \delta) / {n_s}}+\sqrt{\textup{log}(1 / \delta) / {n_t}}\right)$ is a low-order term. $\tilde{\mathcal{F}} \times \mathcal{F}$ refers to the function class $\{x \rightarrow f(x)_{[\tilde{f'}(x)]}: f,f' \in \mathcal{F}\}$. $\widehat{\mathcal{R}}(\mathcal{F}|_{\widehat P})$ denotes the empirical Rademacher complexity of function class $\mathcal{F}$ on dataset $\widehat{P}$.
\end{theorem2}

We provide the function classes used in the proof. For a function class $f \in \mathcal{F}:\Real^{d} \rightarrow [0,1]^{K}$, $\mathcal{F}_{[i]}$ denotes each coordinate of $\mathcal{F}$: $\mathcal{F}_{[i]} = \{x \rightarrow f(x)_{[i]}: f \in \mathcal{F}\}$. We also need other
function classes based on $\mathcal{F}_{[i]}$. $\cup \mathcal{F}_{[i]}$ denotes the union of $\mathcal{F}_{[i]}$: $\cup \mathcal{F}_{[i]} = \cup_{i \in [K]}\mathcal{F}_{[i]}$. $\max_i \mathcal{F}_{[i]}$ is composed of the maximum coordinate of $f \in \mathcal{F}$ for all $x$: $\max_i \mathcal{F}_{[i]} = \{x \rightarrow \max_i f_{[i]}(x):f \in \mathcal{F}\}$. $\max_{i' \neq \tilde{\mathcal{F}}}\mathcal{F}_{[i']}$ denotes the function class composed of the second largest coordinate of $f \in \mathcal{F}$ for all $x$: $ \max_{i' \neq \tilde{\mathcal{F}}}\mathcal{F}_{[i']} = \{x \rightarrow \max_{i\neq \tilde{f}(x)} f_{[i]}(x):f \in \mathcal{F}\}$, which we require to study the finite sample properties of the confidence loss $\mathcal{M}(f(x))$. $\tilde{\mathcal{F}}\times\mathcal{F}$ denotes the function class $\{x \rightarrow  f_{\tilde{f}'(x)}(x):f, f'\in \mathcal{F}\}$. The Rademacher complexity of $\mathcal{F}_{[i]}$ on set $S = \{x_j\}_{j = 1}^{n}$ of size $n$ is: $\widehat{\mathcal{R}}(\mathcal{F}_{[i]}|_S)=\frac{1}{n}\Exp_{\sigma_j} \mathop{\sup}_{f \in \mathcal{F}}\sum_{j=1}^{n} \sigma_j f(x_j)_{[i]}$. The Rademacher complexity of $\cup \mathcal{F}_{[i]}$ is $\widehat{\mathcal{R}}(\cup\mathcal{F}_{[i]}|_S)=\frac{1}{n}\Exp_{\sigma_j} \mathop{\sup}_{f \in \mathcal{F}, i \in [K]}\sum_{j=1}^{n} \sigma_j f(x_j)_{[i]}$. The Rademacher complexity of $\max_i \mathcal{F}_{[i]}$ is $\widehat{\mathcal{R}}(\max_i\mathcal{F}_{[i]}|_S)=\frac{1}{n}\Exp_{\sigma_j} \mathop{\sup}_{f \in \mathcal{F}, i \in [K]}\sum_{j=1}^{n} \sigma_j \max_i f(x_j)_{[i]}$. We further denote by $\widehat{\mathcal{R}}(\mathcal{F}|_S)$ the sum of the Rademacher complexity of each $\mathcal{F}_{[i]}$, $\widehat{\mathcal{R}}(\mathcal{F}|_S) = \sum_{i = 1}^{K}\widehat{\mathcal{R}}(\mathcal{F}_{[i]}|_S)$.

To prove Theorem~\ref{theorem_gen1}, we first observe the relationship between the confidence objective $\Exp_{(x,y)\sim \frac{1}{2}(\widehat P+\widehat Q)}\mathcal{M}(f_t(x))$ and the robustness constraint $R(f_t) := P_{\frac{1}{2}(P + Q)}(\{x:\exists x'\in \mathcal {N}(x), \max_{[i]}f_t(x) \neq \max_{[i]}f_t(x')\})$ in Theorem~\ref{theorem_gen}. In fact, as shown in Lemma~\ref{lem:rob_ent}, when the output of the model is confident on the source and target dataset, i.e. $\Exp_{(x,y)\sim \frac{1}{2}(\widehat P+\widehat Q)}\mathcal{M}(f_t(x))$ is large, the model is also robust to the change in input. 

\begin{lemma}[Confidence guarantees robustness]\label{lem:rob_ent} Under the conditions of Theorem~\ref{theorem_gen1}, we have
\begin{align}
R(f_t) \le \frac{1 - \Exp_{(x,y)\sim \frac{1}{2}(P + Q)}\mathcal{M}(f_t(x))}{1 - 2L_f \xi} .
\end{align}
\end{lemma}
\begin{proof}[Proof of Lemma~\ref{lem:rob_ent}]
We first note that when $\max_y \{f_t(x)_{[y]} - \max_{y'\neq y}f_t(x)_{[y']}\} > 2 L_f \xi$, the $\arg\max_i f_t(x)_{[i]}$ will not change in the neighborhood $\mathcal{N}(x)$ since $f_{[i]}$ is $L_f$-Lipschitz for all $i$. Suppose $y^* = \arg\max_y f_t(x)_{[y]}$. For all $y'\neq y^*$ and $x' \in \mathcal{N}(x)$,
\begin{align}
f_t(x')_{[y^*]} - f_t(x')_{[y']} & > f_t(x)_{[y^*]} - L_fd(x,x') - (f_t(x')_{[y']} + L_fd(x,x'))\\
& \ge  \max_y \{f_t(x)_{[y]} - \max_{y'\neq y}f_t(x)_{[y']}\} - 2L_fd(x,x') \\
& \ge 0.
\end{align}
Therefore, we have 
\begin{align}
R(f_t) &\le 1 - P_{\frac{1}{2}(P + Q)}\left(\mathcal{M}(f_t(x)) > 2L_f \xi\right)\\
&\le \frac{1 - \Exp_{(x,y)\sim \frac{1}{2}(P + Q)}\mathcal{M}(f_t(x))}{1 - 2L_f \xi},
\end{align}
where the second inequality holds because $f_{[i]} \in [0,1]$, and $\mathcal{M}(f(x))\in [0,1]$.
\end{proof}

To obtain finite sample guarantee, we aim to show that each term in~\eqref{eq:genalg1} is close to its population version. We first present Lemma~\ref{lem:margin_standard}, the classical result for multi-class classification.

\begin{lemma}[Lemma 3.1 of~\citep{mohri2018foundations}]\label{lem:margin_standard}
Suppose $f \in \mathcal{F}$ and $\gamma > 0$, with probability at least $1 - \delta$ over the sampling of $\widehat P$, the following holds for all $f \in \mathcal{F}$ simultaneously,
\begin{align}
\textup{Err}_P(f) \le L_{\widehat {P}, \gamma}(f) + \frac{4K}{\gamma} \widehat{\mathcal{R}}(\cup \mathcal{F}_{[i]}|_{\widehat {P}}) + O\left(\sqrt{\textup{log}(1 / \delta)\ /\ n_s}\right).
\end{align}
\end{lemma}
We then extend Lemma~\ref{lem:margin_standard} to study the finite sample properties of $L_{\widehat Q, \gamma}(f_t, f_s)$ and $\Exp\mathcal{M}(f_t(x))$.

\begin{lemma}\label{lem:margin_entropy}
Suppose $f \in \mathcal{F}$ and $\gamma > 0$, with probability at least $1 - \delta$ over the sampling of $\widehat P$, the following holds for all $f \in \mathcal{F}$ simultaneously, 
\begin{align}
\Exp_{(x,y)\sim P}\mathcal{M}(f(x)) \le \Exp_{(x,y)\sim \widehat {P}}\mathcal{M}(f(x)) + 4 \widehat{\mathcal{R}}(\mathcal{F}|_{\widehat{P}}) + O\left(\sqrt{\textup{log}(1 / \delta)\ /\ n_s}\right).
\end{align}
\end{lemma}
\begin{proof}[Proof of Lemma~\ref{lem:margin_entropy}]
By standard Rademacher complexity bound (Theorem 7 of \citet{bartlett2002rademacher}), we have 
\begin{align}
\Exp_{(x,y)\sim P}\mathcal{M}(f(x)) \le \Exp_{(x,y)\sim \widehat {P}}\mathcal{M}(f(x)) + 2 \widehat{\mathcal{R}}(\mathcal{M} \circ \mathcal{F}|_{\widehat {P}}) + O\left(\sqrt{\textup{log}(1 / \delta)\ /\ n_s}\right).
\end{align}
Thus it remains to show $\widehat{\mathcal{R}}(\mathcal{M} \circ \mathcal{F}|_{\widehat {P}}) \le 2 \widehat{\mathcal{R}}(\mathcal{F}|_{\widehat{P}})$. In fact,
\begin{align*}
\widehat{\mathcal{R}}(\mathcal{M} \circ \mathcal{F}|_{\widehat {P}}) &= \frac{1}{n_s} \Exp_{\sigma} \sup_{f \in \mathcal{F}} \sum_{i = 1}^{n_s} \sigma_i \max_y \{f(x_i)_{[y]} - \max_{y'\neq y}f(x_i)_{[y']}\}\\
&= \frac{1}{n_s} \Exp_{\sigma} \sup_{f \in \mathcal{F}} \sum_{i = 1}^{n_s} \sigma_i (\max_y f(x_i)_{[y]} - \max_{y'\neq \tilde{f}(x_i)}f(x_i)_{[y']})\\
&\le \frac{1}{n_s} \Exp_{\sigma} \sup_{f \in \mathcal{F}} \sum_{i = 1}^{n_s} \sigma_i \max_y f(x_i)_{[y]} +  \frac{1}{n_s} \Exp_{\sigma} \sup_{f \in \mathcal{F}} \sum_{i = 1}^{n_s} \sigma_i \max_{y'\neq \tilde{f}(x_i)}f(x)_{[y']}\\
&= \widehat{\mathcal{R}}(\max_i \mathcal{F}_{[i]}|_{\widehat {P}}) + \widehat{\mathcal{R}}(\max_{i' \neq \tilde{\mathcal{F}}}\mathcal{F}_{[i']}|_{\widehat {P}}).
\end{align*}
As will be shown in Lemma~\ref{lem:rad_bound}, both $\widehat{\mathcal{R}}(\max_i \mathcal{F}_{[i]}|_{\widehat {P}})$ and $\widehat{\mathcal{R}}(\max_{i' \neq \tilde{\mathcal{F}}}\mathcal{F}_{[i']}|_{\widehat {P}})$ are smaller than $\widehat{\mathcal{R}}(\mathcal{F}|_{\widehat{P}})$, which completes the proof.
\end{proof}

\begin{lemma}\label{lem:margin_consistency}
Suppose $f_s, f_t \in \mathcal{F}$ and $\gamma > 0$, with probability at least $1 - \delta$ over the sampling of $\widehat Q$, the following holds for all $f_s, f_t \in \mathcal{F}$ simultaneously, 
\begin{align}
\Exp_{(x,y)\sim Q}\mathbb{I}(f_t(x)\neq f_s(x)) \le L_{\widehat Q, \gamma}(f_t, f_s) + \frac{2K}{\gamma} \widehat{\mathcal{R}}(\tilde{\mathcal{F}}\times\mathcal{F}|_{\widehat {Q}}) + O\left(\sqrt{\textup{log}(1 / \delta)\ /\  n_t}\right).
\end{align}
\end{lemma}
\begin{proof}[Proof of Lemma~\ref{lem:margin_consistency}]
By the definition of multi-class margin loss, we have $\Exp_{(x,y)\sim Q}\mathbb{I}(f_t(x)\neq f_s(x)) \le L_{Q, \gamma}(f_t, f_s)$. Denote by $\mathcal{G}$ the set of $\{x \rightarrow (-\mathcal{M}(f_t(x),f_s(x))): f_t, f_s \in \mathcal{F}\}$. By standard Rademacher complexity bound, we have,
\begin{align*}
L_{Q, \gamma}(f_t, f_s) \le L_{\widehat Q, \gamma}(f_t, f_s) + 2\widehat{\mathcal{R}}(\psi_{\gamma} \circ \mathcal{G}|_{\widehat {Q}}) + O\left(\sqrt{\textup{log}(1 / \delta)\ /\  n_t}\right).
\end{align*}
By Talagrand contraction Lemma~\cite{talagrand2014upper}, $\widehat{\mathcal{R}}(\psi_{\gamma} \circ \mathcal{G}|_{\widehat {Q}}) \le \frac{1}{\gamma} \widehat{\mathcal{R}}(\mathcal{G}|_{\widehat {Q}})$. Thus, it remains to show $\widehat{\mathcal{R}}(\mathcal{G}|_{\widehat {Q}}) \le K \widehat{\mathcal{R}}(\tilde{\mathcal{F}}\times\mathcal{F}|_{\widehat {Q}})$. We have
\begin{align*}
\widehat{\mathcal{R}}(\mathcal{G}|_{\widehat {Q}}) &= \frac{1}{n_t} \Exp_{\sigma_i} \sup_{f_s,f_t} \sum_{i = 1}^{n_t} \sigma_i \mathcal{M}(f_t(x_i),\tilde{f}_s(x_i))\\
&= \frac{1}{n_t} \Exp_{\sigma_i} \sup_{f_s,f_t} \sum_{i = 1}^{n_t} \sigma_i \left(f_t(x_i)_{[\tilde{f}_s(x_i)]} - \max_{y' \neq \tilde{f}_s(x_i)} f_t(x_i)_{[y']}\right)\\
&\le \frac{1}{n_t} \Exp_{\sigma_i} \sup_{f_s,f_t} \sum_{i = 1}^{n_t} \sigma_i f_t(x_i)_{[\tilde{f}_s(x_i)]} + \frac{1}{n_t} \Exp_{\sigma_i} \sup_{f_s,f_t} \sum_{i = 1}^{n_t} \sigma_i \max_{y' \neq \tilde{f}_s(x_i)} f_t(x_i)_{[y']}\\
& = \widehat{\mathcal{R}}(\tilde{\mathcal{F}}\times\mathcal{F}|_{\widehat {Q}}) + \frac{1}{n_t} \Exp_{\sigma_i} \sup_{f_s,f_t} \sum_{i = 1}^{n_t} \sigma_i \max_{y' \neq \tilde{f}_s(x_i)} f_t(x_i)_{[y']}.
\end{align*}
It remains to show $\frac{1}{n_t} \Exp_{\sigma_i} \sup_{f_s,f_t} \sum_{i = 1}^{n_t} \sigma_i \max_{y' \neq \tilde{f}_s(x_i)} f_t(x_i)_{[y']} \le (K-1) \widehat{\mathcal{R}}(\tilde{\mathcal{F}}\times\mathcal{F}|_{\widehat {Q}})$, which is done by noting the closure of $\mathcal{F}$ under the permutation of coordinates. Consider the permutation $\upsilon : \Real^K \rightarrow \Real^K$: $\upsilon(v)_{[i]} = v_{[i - 1]}$ for $i \in [2,3,\cdots K]$ and $\upsilon(v)_{[1]} = v_{[K]}$.
\begin{align*}
\frac{1}{n_t} \Exp_{\sigma_i} \sup_{f_s,f_t} \sum_{i = 1}^{n_t} \sigma_i \max_{y' \neq \tilde{f}_s(x_i)} f_t(x_i)_{[y']} = \frac{1}{n_t} \Exp_{\sigma_i} \sup_{f_s,f_t} \sum_{i = 1}^{n_t} \sigma_i \max_{k \in [K-1]}\upsilon^k f_t(x_i)_{[\tilde{f}_s(x_i)]}.
\end{align*}
We have $\upsilon\mathcal{F} \subset \mathcal{F}$ by the closure of $\mathcal{F}$. Thus, $\tilde{\mathcal{F}} \times \upsilon\mathcal{F} \subset \tilde{\mathcal{F}} \times \mathcal{F}$. By Lemma~\ref{lem:rad_bound}, the Rademacher complexity of maximum of function classes is bounded with their sum, so we have $\frac{1}{n_t} \Exp_{\sigma_i} \sup_{f_s,f_t} \sum_{i = 1}^{n_t} \sigma_i \max_{k \in [K-1]}\upsilon^k f_t(x_i)_{[\tilde{f}_s(x_i)]} \le (K-1) \widehat{\mathcal{R}}(\tilde{\mathcal{F}}\times\mathcal{F}|_{\widehat {Q}})$.
\end{proof}

The next lemma shows the relationship between function classes. We establish the Rademacher complexity bounds of $\cup \mathcal{F}_{[i]}$, $\max_i \mathcal{F}_{[i]}$, and $\max_{i' \neq \tilde{\mathcal{F}}}\mathcal{F}_{[i']}$. We show that the Rademacher complexity of these function classes can be bounded with $\widehat{\mathcal{R}}(\mathcal{F}|_S) = \sum_{i = 1}^{K}\widehat{\mathcal{R}}(\mathcal{F}_{[i]}|_S)$.

\begin{lemma}\label{lem:rad_bound}
Suppose $\max_i \mathcal{F}_{[i]} = \{\max_i f_{[i]}:f \in \mathcal{F}\}$, $\cup \mathcal{F}_{[i]} = \{f_{[i]}:f \in \mathcal{F}, i \in [K]\}$, and $ \max_{i' \neq \tilde{\mathcal{F}}}\mathcal{F}_{[i']}=\max_{i' \neq \tilde{\mathcal{F}}}\mathcal{F}_{[i']}|_{\widehat {P}} = \{x \rightarrow \max_{i\neq \tilde{f}(x)} f_{[i]}(x):f \in \mathcal{F}\}$.
\begin{align}
\widehat{\mathcal{R}}(\cup\mathcal{F}_{[i]}|_S) \le \widehat{\mathcal{R}}(\mathcal{F}|_S), \ \widehat{\mathcal{R}}(\max_i\mathcal{F}_{[i]}|_S) \le \widehat{\mathcal{R}}(\mathcal{F}|_S) \quad \text{and} \quad \widehat{\mathcal{R}}(\max_{i' \neq \tilde{\mathcal{F}}}\mathcal{F}_{[i']}|_S) \le \widehat{\mathcal{R}}(\mathcal{F}|_S).
\end{align}
\end{lemma}
\begin{proof}[Proof of Lemma~\ref{lem:rad_bound}]
Consider the $K=2$ case. Then we can repeat the arguments for $K-1$ times to get the final results.

For the first inequality, consider $\mathcal{F}_{[i]}': = \mathcal{F}_{[i]} \cup -\mathcal{F}_{[i]}=\{x\rightarrow \pm f(x) :f\in\mathcal{F}_{[i]}\}$. Then we have
\begin{align}
\widehat{\mathcal{R}}(\mathcal{F}_{[1]}\cup\mathcal{F}_{[2]}|_S) &= \frac{1}{n}\Exp_{\sigma_j} \sup_{f_{[i]} \in \mathcal{F}_{[1]} \cup \mathcal{F}_{[2]}} \sum_{j = 1}^{n} \sigma_j f_{[i]}(x_j)\\
&= \frac{1}{n}\Exp_{\sigma_j} \sup_{f_{[i]}' \in \mathcal{F}'_{[1]} \cup \mathcal{F}_{[2]}'} \sum_{j = 1}^{n} \left|\sigma_j f_{[i]}'(x_j)\right|\label{eqn:7-1}\\
&\le \frac{1}{n}\Exp_{\sigma_j} \sup_{f_{[i]}' \in \mathcal{F}'_{[1]}} \sum_{j = 1}^{n} \left|\sigma_j f_{[i]}'(x_j)\right| + \frac{1}{n}\Exp_{\sigma_j} \sup_{f_{[i]}' \in \mathcal{F}'_{[2]}} \sum_{j = 1}^{n} \left|\sigma_j f_{[i]}'(x_j)\right|\label{eqn:7-2}\\
& = \frac{1}{n}\Exp_{\sigma_j} \sup_{f_{[i]} \in \mathcal{F}_{[1]}} \sum_{j = 1}^{n} \sigma_j f_{[i]}(x_j) + \frac{1}{n}\Exp_{\sigma_j} \sup_{f_{[i]} \in \mathcal{F}_{[2]}} \sum_{j = 1}^{n} \sigma_j f_{[i]}(x_j)\label{eqn:7-3}\\
& = \widehat{\mathcal{R}}(\mathcal{F}_{[1]}|_S) + \widehat{\mathcal{R}}(\mathcal{F}_{[2]}|_S),
\end{align}
where equation~\eqref{eqn:7-1} and~\eqref{eqn:7-3} hold by the definition of $\mathcal{F}_{[i]}$. 

For the second inequality, note that $\max\{x,y\} = \frac {x + y} {2} + \frac{|x-y|} {2}$. Then we apply Talagrand contraction lemma for the absolute value ($|\cdot|$ is $1$-Lipschitz),
\begin{align*}
\widehat{\mathcal{R}}(\max_i\mathcal{F}_{[i]}|_S) &= \frac{1}{n}\Exp_{\sigma_j} \sup_{f_{[1]} \in \mathcal{F}_{[1]}, f_{[2]} \in \mathcal{F}_{[2]}} \sum_{j = 1}^{n} \sigma_j \left(\frac {f_{[1]}(x_j) + f_{[2]}(x_j) } {2} + \frac{|f_{[1]}(x_j) - f_{[2]}(x_j) |} {2}\right)\\
&\le \frac{1}{2} \widehat{\mathcal{R}}(\mathcal{F}_{[1]}|_S)  + \frac{1}{2} \widehat{\mathcal{R}}(\mathcal{F}_{[2]}|_S) + \frac{1}{2} \widehat{\mathcal{R}}(|\mathcal{F}_{[1]} - \mathcal{F}_{[2]}||_S)\\
&\le \frac{1}{2} \widehat{\mathcal{R}}(\mathcal{F}_{[1]}|_S)  + \frac{1}{2} \widehat{\mathcal{R}}(\mathcal{F}_{[2]}|_S) + \frac{1}{2} \widehat{\mathcal{R}}(\mathcal{F}_{[1]} - \mathcal{F}_{[2]}|_S)\\
&\le \widehat{\mathcal{R}}(\mathcal{F}_{[1]}|_S)  + \widehat{\mathcal{R}}(\mathcal{F}_{[2]}|_S).
\end{align*}
For the third inequality, observe that the second largest of the set $\{x,y,z\}$ can be expressed as $\max\{\min\{x,y\},\min\{\max\{x,y\},z\}\}$. Following argument similar to the second inequality gives the proof.
\end{proof}
Now equipped with the lemmas above, we are ready to prove Theorem~\ref{theorem_gen1}.
\begin{proof}[Proof of Theorem~\ref{theorem_gen1}]
By Lemmas~\ref{lem:margin_standard},~\ref{lem:margin_entropy}, and~\ref{lem:margin_consistency}, we have the following inequalities hold with probability larger than $1 - \frac{\delta} {3}$,
\begin{align}\label{eqn:ff1}
\textup{Err}_P(f) \le L_{\widehat {P}, \gamma}(f) + \frac{4K}{\gamma} \widehat{\mathcal{R}}(\cup \mathcal{F}_{[i]}|_{\widehat {P}}) + O\left(\sqrt{\textup{log}(1 / \delta)\ /\ n_s}\right).
\end{align}
\begin{align}\label{eqn:ff2}
\Exp_{(x,y)\sim \frac{1}{2}(P+Q)}\mathcal{M}(f(x)) &\le \Exp_{(x,y)\sim \frac{1}{2}(\widehat P+ \widehat Q)}\mathcal{M}(f(x)) + 2 \widehat{\mathcal{R}}(\mathcal{F}|_{\widehat{P}}) + 2 \widehat{\mathcal{R}}(\mathcal{F}|_{\widehat{Q}}) \\
&+ O\left(\sqrt{\textup{log}(1 / \delta)\ /\ n_s} + \sqrt{\textup{log}(1 / \delta)\ /\ n_t}\right).\nonumber
\end{align}
\begin{align}\label{eqn:ff3}
\Exp_{(x,y)\sim Q}\mathbb{I}(f_t(x)\neq f_s(x)) \le L_{\widehat Q, \gamma}(f_t, f_s) + \frac{2K}{\gamma} \widehat{\mathcal{R}}(\tilde{\mathcal{F}}\times\mathcal{F}|_{\widehat {Q}}) + O\left(\sqrt{\textup{log}(1 / \delta)\ /\  n_t}\right).
\end{align}
We also have the following due to Lemma~\ref{lem:rob_ent},
\begin{align}\label{eqn:ff4}
R(f_t) \le \frac{1 - \Exp_{(x,y)\sim \frac{1}{2}(P + Q)}\mathcal{M}(f_t(x))}{1 - 2L_f \xi} .
\end{align}
We use \eqref{eqn:ff2} and \eqref{eqn:ff3} as conditions. Plugging \eqref{eqn:ff1}, \eqref{eqn:ff2}, \eqref{eqn:ff3}, and \eqref{eqn:ff4} into Theorem~\ref{theorem_gen} and applying a union bound complete the proof of Theorem~\ref{theorem_gen1}.
\end{proof}

\subsection{Details in Section~\ref{sec:hard}}\label{sec:alg}

We instantiate the domain adaptation setting in a quadratic neural network that allows us to compare various properties of the related algorithms. For a specific data distribution, we prove that (1) cycle self-training recovers target ground truth, and (2) both feature adaptation and standard self-training fail on the same distribution. 

\subsubsection{Setup}

We study a quadratic neural network composed of a feature extractor $\phi\in\Real^{d\times m}$ and a head $\theta\in\Real^{m}$. $f_{\theta,\phi}(x) = g_\theta(h_\phi(x))$, where $g_\theta(z) = \theta^\top z$ and $h_\phi(x) = (\phi^\top x) \odot (\phi^\top x)$, $\odot$ is element-wise product. In training, we use the squared loss $\ell(f(x), y) = (f(x) - y) ^ 2$. In testing, we map $f(x)$ to the nearest point in the output space: $\tilde f (x) := \mathop{\arg\min}_{y \in \{-1,0,1\}}|y - f(x)|$. Denote the expected error by $\text{Err}_Q(\theta,\phi) := \Exp_{(x,y)\sim Q}\mathbb{I}(\tilde f_{\theta,\phi} (x)\neq y)$.

\textbf{Structural Covariate Shift and Label Shift.} In domain adaptation, the source domain can have \textit{multiple solutions} but we aim to learn the solution which works on the target domain~\citep{cite:ICML15DAN}. Recent works also pointed out the source and the target label distributions are often different in real-world applications~\citep{pmlr-v97-zhao19a}. Following these properties, we design the underlying distributions $p$ and $q$ as shown in Table~\ref{table:pq} to allow both structural covariate shift and label shift.

\begin{table}[h]
\addtolength{\tabcolsep}{2pt}
\label{table:pq}
\centering
\caption{Comparison of the design of the source and target.}
\label{table:pq}
\begin{tabular}{l|c|c|c}
\toprule
Distribution & $-1$ & $+1$ & $0$ \\
\midrule
Source $p$ & $\quad0.05\quad$ & $\quad0.05\quad$ & $\quad0.90\quad$\\
Target $q$ & $\quad0.25\quad$ & $\quad0.25\quad$ & $\quad0.50\quad$\\
\bottomrule
\end{tabular}

\end{table} 

We study the following source distribution $P$. $x_{[1]}$ and $x_{[2]}$ are sampled \emph{i.i.d.} from distribution $p$, and for $i\in[3,d]$, $x_{[i]}= \sigma_i \times x_{[2]}$. $\sigma_i \in \{\pm 1\}$ uniformly. In the target domain, $x_{[1]}$ and $x_{[2]}$ are sampled \emph{i.i.d.} from distribution $q$, and for $i\in[3,d]$, $x_{[i]}= \sigma_i \times x_{[1]}$. $\sigma_i \in \{\pm 1\}$ uniformly. We also assume realizability: $y = x_{[1]}^2 - x_{[2]}^2$ for both source and target. For simplicity, we assume access to infinite \emph{i.i.d.} examples of $P$ ($n_s = \infty$) and $n_t$ \emph{i.i.d.} examples of $Q$. Therefore, the empirical loss and the population loss on the source domain are the same $L_P = L_{\widehat P}$.

Note that since $x_{[i]}^2 = x_{[2]}^2$ for all $i\in[3,d]$ in the source domain, $y = x_{[1]}^2 - x_{[i]}^2$ for all $i\in[2,d]$ are solutions to the source domain but only $y = x_{[1]}^2 - x_{[2]}^2$ works on the target domain. We visualize the setting when $d = 3$ in Figure~\ref{fig:th_setting}. 

\subsubsection{Algorithms}\label{algorithms}

We compare the baseline algorithms (feature adaptation and self-training) in Section~\ref{sec:pre} with the proposed CST. We study the norm-constrained versions of these algorithms.

\textbf{Feature Adaptation} chooses the source solution minimizing the distance between source and target feature distributions. We use total variation (TV) distance \cite{cite:ML10DAT}: $d_\textup{TV}(h_\sharp{\widehat P},h_\sharp{\widehat Q}) = \mathop{\sup}_{E\subset\mathcal{Z}}|h_\sharp{\widehat P}(E) - h_\sharp{\widehat Q}(E)|$.
\begin{align}
&\hat\theta_\textup{FA},\hat\phi_\textup{FA} = \mathop{\arg\min}_{\hat\theta_s,\hat\phi_s} \ d_\textup{TV}(h_\sharp{\widehat P},h_\sharp{\widehat Q}),\\ 
\textup{s.t.}\ \hat\theta_s,\hat\phi_s& = \mathop{\arg\min}_{\theta,\phi} \|\theta\|_2^2 + \|\phi\|_F^2,\ \textup{s.t.}\ L_{P}(\theta,\phi)=0.\nonumber
\end{align}
\textbf{Standard Self-Training} first trains a source model,
\begin{align}
\hat\theta_s,\hat\phi_s = \mathop{\arg\min}_{\theta,\phi} \|\theta\|_2^2 + \|\phi\|_F^2,\ \textup{s.t.}\ L_{P}(\theta,\phi)=0.
\end{align}
Then it trains the model on the source and target datasets jointly with source ground-truths and target pseudo-labels,
\begin{align}
&\hat\theta_{\textup{ST}}, \hat\phi_{\textup{ST}}= \mathop{\arg\min}_{\theta,\phi} \|\theta\|_2^2 + \|\phi\|_F^2,\\
\textup{s.t.}\ L_{P}&(\theta,\phi) + \Exp_{x \sim \widehat Q}\ell(f_{\theta,\phi}(x),f_{\hat\theta_s,\hat\phi_s}(x))=0.\nonumber
\end{align}

\textbf{Cycle Self-Training.} Following Section~\ref{sec:cst}, we train the source head $\theta_s$, and then train another head $\hat\theta_t(\phi)$ on the target dataset $\widehat{Q}$ with pseudo-labels generated by $\theta_s$:
\begin{align}
\hat\theta_t(\phi) = \mathop{\arg\min}_{\theta}\|\theta\|_2^2,\ \text{s.t.}\  \Exp_{x\in \widehat Q}\ell(f_{\theta,\phi}(x),f_{\theta_s,\phi}(x))=0.\nonumber
\end{align}
Finally we update the feature extractor $\phi$ to enforce consistent predictions of $\hat\theta_t(\phi)$ and $\theta_s$ on the source dataset:
\begin{align}
&\hat\theta_\textup{CST},\hat\phi_{\textup{CST}} = \mathop{\arg\min}_{\theta_s,\phi}\|\theta_s\|_2^2 + \|\phi\|_F^2 ,\\ \textup{s.t.}\ L_{P}(\theta_s&,\phi)+\Exp_{x\in P}\ell(g_{\theta_s}(h_\phi(x)),g_{\hat\theta_t(\phi)}(h_\phi(x)))=0.\nonumber
\end{align}

The following theorems show that both feature adaptation and standard self-training fail. The intuition is that the \emph{ideal} solution that works on both source and target $y = x_{[1]}^2 - x_{[2]}^2$ has larger distance $d_\textup{TV}$ in the feature space than other solutions $y = x_{[1]}^2 - x_{[i]}^2$, so feature adaptation will not prefer the ideal solution. Standard self-training also fails because it will choose randomly among $y = x_{[1]}^2 - x_{[i]}^2$.

\begin{theorem3}
For any $\epsilon\in(0,0.5)$, the following statements are true for feature adaptation and standard self-training:
\vspace{-5pt}
	\begin{itemize}[leftmargin=*]
		\setlength{\itemsep}{5pt}%
		\setlength{\parskip}{-3pt}
		\item{
	For any failure rate $\xi>0$, and target dataset of size $\nt >\Theta(\log\frac{1}{\xi})$, with probability at least $1-\xi$ over the sampling of target data, the source solution $\hat\theta_\textup{FA},\hat\phi_\textup{FA}$ found by feature adaptation fails on the target domain:
            \vspace{-3pt}
			\begin{align}
			\textup{Err}_Q(\hat\theta_\textup{FA},\hat\phi_\textup{FA})\ge \epsilon.
			\end{align}	}
			\vspace{-5pt}
		\item{
	With probability at least $1-\frac{1}{d-1}$ over the training the source solution, the solution $(\hat\theta_\textup{ST},\hat\phi_\textup{ST})$ of standard self-training satisfies
            \vspace{-3pt}
			\begin{align}
			\textup{Err}_Q(\hat\theta_\textup{ST},\hat\phi_\textup{ST})\ge \epsilon.
			\end{align}	}
	\end{itemize}
\end{theorem3}

In comparison, we show that CST can recover the ground truth with high probability.

\begin{theorem4}
For any failure rate $\xi>0$, and target dataset of size $\nt >\Theta(\log\frac{1}{\xi})$, with probability at least $1-\xi$ over the sampling of target data, the feature extractor $\hat\phi_{\textup{CST}}$ found by CST and the head $\hat\theta_\textup{CST}$ recovers the ground truth of the target dataset: 
	\vspace{-3pt}
	\begin{align}
	\textup{Err}_Q(\hat\theta_\textup{CST}, \hat\phi_{\textup{CST}})= 0.
	\end{align}
\end{theorem4}

Intuitively, CST successfully learns the \emph{transferable feature} $x_{[1]}^2 - x_{[2]}^2$ because it enforces the generalization of the head $\hat\theta_t(\phi)$ on the source data.

\subsection{Proof of Theorem~\ref{theorem_fa}}
We first describe the insights of the proof. As shown in Lemma~\ref{lem:source_gt}, every source solution can be categorized into $d-1$ classes according to the coordinate $l$ of the learned weight $\phi$. Among those $d-1$ classes, only $l = 2$ works on the target domain and $l \in \{3,\cdots d\}$ do not work on the target. We then show that in feature adaptation, $l \in \{3,\cdots d\}$ results in smaller distance between source and target feature distributions as a result of Lemma~\ref{lem:concentration}, thus feature adaptation will choose $l \in \{3,\cdots d\}$. On the other hand, standard self-training with randomly select $l$ in the possible $d - 1$ choices, but only $l = 2$ works.

\begin{lemma}\label{lem:source_gt}
Under the condition of Section~\ref{sec:theory}, any solution $\theta,\phi$ to the Source Only problem
\begin{align}
\mathop{\min}_{\theta,\phi} \|\theta\|_2^2 + \|\phi\|_F^2,\ \textup{s.t.}\ L_{P}(\theta,\phi)=0.
\end{align}
must have the following form: $\exists i,j \in \{2,3,\cdots m\}, \ l \in \{2,3,\cdots d\}, \ \phi_i = 2^ {\frac{1}{6}} e_1, \ \phi_j = 2^ {\frac{1}{6}} e_l , \ \phi_k = 0$ for $k \neq i, j$, and $\theta_i = \theta_j = 2^ {-\frac{1}{3}}, \ \theta_k = 0$ for $k \neq i, j$.
\end{lemma}

\begin{proof}[Proof of Lemma~\ref{lem:source_gt}]
Define the symmetric matrix $A = \sum_{i = 1} ^ {m} \theta_i \phi_i \phi_i^\top$, then the networks can be represented by $A$: $f_{\theta,\phi}(x) = x^\top A x$. We show that $A_{ij} = 0$ for $i \neq j$ if $f_{\theta,\phi}$ recovers source ground truth.

First, for $i,j > 1$, let $x_1 = \mathbf{1}$, $x_2 = \mathbf{1} - 2e_i - 2e_j$, $x_3 = \mathbf{1} - 2e_i$, and $x_4 = \mathbf{1} - 2e_j$, where $\{e_i\}$ are the standard bases. Since the source ground truth $y = x_{[1]}^2 - x_{[2]}^2$, and $x_{[k]} = \pm 1 x_{[2]}$ for $k\in\{3,4,\cdots d\}$, $y_1 = y_2 = y_3 = y_4$. 
\begin{align}
y_1 + y_2 + y_3 - y_4 =& x_1^\top A x_1 + x_2^\top A x_2 - x_3^\top A x_3 - x_4^\top A x_4 
\\=& 2 \mathbf{1}^\top A \mathbf{1}+ 4 A_{ii} + 4 A_{jj} - 4\mathbf{1}^\top A e_i - 4\mathbf{1}^\top A e_j + 8 A_{ij} \\
&- (\mathbf{1}^\top A \mathbf{1}+ 4 A_{ii} + 4 A_{jj} - 4\mathbf{1}^\top A e_i - 4\mathbf{1}^\top A e_j)\\
= & 8 A_{ij} = 0. \label{eqn:aij}
\end{align}
We then show $A_{1,j} = 0$ for $j \in \{2,3,\cdots d\}$ using the fact that $y_1 = y_4$.
\begin{align}
y_1 - y_4 =& x_1^\top A x_1 - x_4^\top A x_4
\\ =& \mathbf{1}^\top A \mathbf{1} - (\mathbf{1}^\top A \mathbf{1} - 4\mathbf{1}^\top A e_j + 4 A_{jj})
\\ =& 4\mathbf{1}^\top A e_j - 4 A_{jj} 
\\ =& 0.
\end{align}
From \eqref{eqn:aij} we know $A_{ij} = 0$ if $i \neq 1$. Then we also have $A_{1j} = A_{j1} = 0$. Therefore we can write $y$ in the following form: $y = x^\top A x = \sum_{i = 1}^d A_{ii}x_{[i]}^2$
We also have the source ground truth $y = x_{[1]}^2 - x_{[2]}^2$, and $x_{[k]} = \pm 1 x_{[2]}$ for $k\in\{3,4,\cdots d\}$. Then $A$ must satisfy $A_{11} = 1$, $\sum_{i = 2}^d A_{ii} = -1$, and all other entries of $A$ equals to $0$.

We have found the form of source ground truth matrix $A$. It suffices to show that the minimal norm solution of $\theta$ and $\phi$ subject to the form of $A$ must be in the form of Lemma~\ref{lem:source_gt}. 
\begin{align}
\|\theta\|_2^2 + \|\phi\|_F^2 =& \sum_{i}^m \theta_i^2 + \frac{1}{2} \|\phi_i\|_2^2 + \frac{1}{2} \|\phi_i\|_2^2
\\\ge& \sum_{i}^m 3 \cdot 2^\frac{2}{3} \left(|\theta_i| \|\phi_i\|_2^2 \right)^\frac{2}{3}
\\\ge& 3 \cdot 2^\frac{2}{3} \left(\sum_{i: \theta_i>0} \theta_i \|\phi_i\|_2^2 \right)^\frac{2}{3} + 3 \cdot 2^\frac{2}{3}\left(\sum_{i: \theta_i \le 0} -\theta_i \|\phi_i\|_2^2 \right)^\frac{2}{3}
\\=& 3 \cdot 2^\frac{2}{3} \left(\sum_{i: A_{ii} > 0} A_{ii}\right)^\frac{2}{3} + 3 \cdot 2^\frac{2}{3} \left(\sum_{i: A_{ii} \le 0} -A_{ii}\right)^\frac{2}{3} = 3 \cdot 2^\frac{5}{3}.
\end{align}
The first inequality holds due to AM-GM inequality, where it takes equality iff $\theta_i^2 = \frac{1}{2} \|\phi_i\|_2^2$ for all $i$. The second inequality holds due to Jensen inequality. The situation where both inequality take equality is exactly the form of Lemma~\ref{lem:source_gt}.
\end{proof}

\begin{lemma}\label{lem:concentration} Suppose $\widehat{Q}=\{x_i^t\}_{i=1}^{n_t}$ are i.i.d. samples from target distribution $Q$, then with high probability, $\Exp_{(x,y) \sim \widehat Q} \mathbb{I}(x_{[l]} = 0)$ is close to $0.5$:
\begin{align}
P\left( \left| \Exp_{(x,y) \sim \widehat Q} \mathbb{I}(x_{[l]} = 0) - 0.5\right| > t \right) \le e^{-2n_t t^2}
\end{align}
\end{lemma}

\begin{proof}[Proof of Lemma~\ref{lem:concentration}]
Since each coordinate of $x$ follows $q$, $\mathbb{I}(x_{[l]} = 0) - 0.5$ is a sub-Gaussian variable with $\sigma = 0.5$. We then apply standard Hoeffding's inequality to complete the proof.
\end{proof}

\begin{proof}[Proof of Theorem~\ref{theorem_fa}]
We have the conclusion of Lemma~\ref{lem:source_gt}. For simplicity we suppose without loss of generality that the source solution has the following form: $\exists \ l \in \{2,3,\cdots d\}, \ \phi_1 = 2^ {\frac{1}{6}} e_1, \ \phi_2 = 2^ {\frac{1}{6}} e_l , \ \phi_k = 0$ for $k \in \{3,\cdots m\}$, and $\theta_1 = \theta_2 = 2^ {-\frac{1}{3}}, \ \theta_k = 0$ for $k \in \{3,\cdots m\}$. Then these solutions can be categorized into two classes: (1) When $l = 2$, the source solution also works on the target, i.e. $L_Q(\theta, \phi) = 0$. (2) When $l \in \{3,\cdots d\}$, the source solution does not work on the target,
\begin{align}
\textup{Err}_Q(\theta, \phi) &= 1 - \Exp_{(x,y) \sim Q} \mathbb{I}(f_{\theta,\phi}(x) = y) 
\\&= 1 - \Exp_{(x,y) \sim Q} \mathbb{I}(x_{[l]}^2 = x_{[2]}^2)
\\&= 1 - \Exp_{(x,y) \sim Q} \mathbb{I}(x_{[l]} = 0 \ \textup{and}\  x_{[2]} = 0) - \Exp_{(x,y) \sim Q} \mathbb{I}(x_{[l]} \neq 0 \ \textup{and}\  x_{[2]} \neq 0) 
\\&= 0.5
\end{align}
To prove that feature adaptation learns the solution that does not work on the target domain, we show that with high probability, the solution belonging to situation (1) has larger total variation between source and target feature distributions $h_\sharp P$ and $h_\sharp \widehat Q$. In fact, the distributions of $h_\sharp P$ are the same for solutions in situation (1) and situation (2):
\begin{align*}
\Exp_{(x,y) \sim P} \left( h_ \phi (x) = (z_1,z_2)\right) = 
\left\{
             \begin{array}{lr}
             0.81, &(z_1,z_2) = 2^ {\frac{1}{3}}(0,0)  \\
             0.09, &(z_1,z_2) = 2^ {\frac{1}{3}}(0,1)\\
             0.09, &(z_1,z_2) = 2^ {\frac{1}{3}}(1,0)\\  
             0.01, &(z_1,z_2) = 2^ {\frac{1}{3}}(1,1)
             \end{array}
\right.
\end{align*}
For the target dataset, the distribution of features is different for solutions from situation (1) and situation (2). When $l = 2$, denote by $h_{1\sharp} \widehat Q$ the feature distribution.
\begin{align*}
\Exp_{(x,y) \sim \widehat Q} \left( h_ \phi (x) = (z_1,z_2)\right) = 
\left\{
             \begin{array}{lr}
             \Exp_{(x,y) \sim \widehat Q} ^ 2 \mathbb{I}(x_{[l]} = 0), &(z_1,z_2) = 2^ {\frac{1}{3}}(0,0)  \\
             \Exp_{(x,y) \sim \widehat Q} \mathbb{I}(x_{[l]} = 0) \Exp_{(x,y) \sim \widehat Q}  \mathbb{I}(x_{[l]} \neq 0), &(z_1,z_2) = 2^ {\frac{1}{3}}(0,1)\\
             \Exp_{(x,y) \sim \widehat Q} \mathbb{I}(x_{[l]} = 0) \Exp_{(x,y) \sim \widehat Q} \mathbb{I}(x_{[l]} \neq 0), &(z_1,z_2) = 2^ {\frac{1}{3}}(1,0)\\  
             \Exp_{(x,y) \sim \widehat Q} ^ 2 \mathbb{I}(x_{[l]} \neq 0), &(z_1,z_2) = 2^ {\frac{1}{3}}(1,1)
             \end{array}
\right.
\end{align*}
When $l \in \{3,\cdots d\}$, denote by $h_{2\sharp} \widehat Q$ the feature distribution. since $x_{[l]} = x_{[1]}$ in the target domain, $(z_1,z_2)$ can only be $2^ {\frac{1}{3}}(0,0)$ or $2^ {\frac{1}{3}}(1,1)$,
\begin{align*}
\Exp_{(x,y) \sim \widehat Q} \left( h_ \phi (x) = (z_1,z_2)\right) = 
\left\{
             \begin{array}{lr}
             \Exp_{(x,y) \sim \widehat Q} \mathbb{I}(x_{[l]} = 0), &(z_1,z_2) = 2^ {\frac{1}{3}}(0,0)  \\
             \Exp_{(x,y) \sim \widehat Q} \mathbb{I}(x_{[l]} \neq 0), &(z_1,z_2) = 2^ {\frac{1}{3}}(1,1)
             \end{array}
\right.
\end{align*}
We then instantiate Lemma~\ref{lem:concentration} with $t = 0.14$: With probability at least $1 - \delta$, $0.36<\Exp_{(x,y) \sim \widehat Q} \mathbb{I}(x_{[l]} = 0) < 0.64$ for any $n_t \ge C \textup{log}\left(\frac{1}{\delta}\right)$, where $C > 26$ is a constant. Finally, we show that $d_\textup{TV} (h_\sharp P, h_{2\sharp} \widehat Q) < d_\textup{TV} (h_\sharp P, h_{1\sharp} \widehat Q)$ as long as $0.36<\Exp_{(x,y) \sim \widehat Q} \mathbb{I}(x_{[l]} = 0) < 0.64$ to prove that feature adaptation will select solutions in situation (2).
\begin{align} 
d_\textup{TV} (h_\sharp P, h_{1\sharp} \widehat Q) =& \frac{1}{2} \left|\Exp_{(x,y) \sim \widehat Q} ^ 2 \mathbb{I}(x_{[l]} = 0) - 0.81\right| + \frac{1}{2} \left|\Exp_{(x,y) \sim \widehat Q} ^ 2 \mathbb{I}(x_{[l]} \neq 0) - 0.01\right| 
\\ & + \left|\Exp_{(x,y) \sim \widehat Q} \mathbb{I}(x_{[l]} = 0) \Exp_{(x,y) \sim \widehat Q} \mathbb{I}(x_{[l]} \neq 0) - 0.09\right|
\\ =&  0.81 - \Exp^2_{(x,y) \sim \widehat Q} \mathbb{I}(x_{[l]} = 0)
\\ >& 0.9 - \Exp_{(x,y) \sim \widehat Q} \mathbb{I}(x_{[l]} = 0)
\\ = & \frac{1}{2} \left|\Exp_{(x,y) \sim \widehat Q} \mathbb{I}(x_{[l]} = 0) - 0.81\right| + \frac{1}{2} \left|\Exp_{(x,y) \sim \widehat Q} \mathbb{I}(x_{[l]} \neq 0) - 0.01\right| 
\\ =& d_\textup{TV} (h_\sharp P, h_{2\sharp} \widehat Q),
\end{align}
when $0.36<\Exp_{(x,y) \sim \widehat Q} \mathbb{I}(x_{[l]} = 0) < 0.64$, which completes the proof of feature adaptation.

In standard self-training, when training the source solution, the probability of $l$ equalling each value in $\{2,3,\cdots d\}$ is the same, but only $l = 2$ is the solution working on the source domain. Then when training on the source ground truth and target pseudo-labels, the model will make $l$ unchanged. Thus the probability of recovering the target ground truth is only $\frac{1}{d-1}$.  
\end{proof}

\subsection{Proof of Theorem~\ref{theorem_transfer}}
Similar to the proof of Theorem~\ref{theorem_fa}, we use the conclusion of Lemma~\ref{lem:source_gt} to show that $l = 2$ indicates the source solution that works on the target domain, while the solutions corresponding to $l \in \{3,\cdots d\}$ will have large error on the target domain. Then we show that only $l = 2$ makes the training objective of cycle self-training $L_\textup{CST} = 0$. This is due to the fact that $l = 2$ will make the spans of source and target features identical, while $l \in \{3,\cdots d\}$ makes the spans of source and target features different and thus $\hat\theta_t(\phi) \neq \theta_s$.
\begin{proof}[Proof of Theorem~\ref{theorem_transfer}]
We still use Lemma~\ref{lem:source_gt}. To prove that cycle self-training recovers the target ground truth, it suffices to show that $\Exp_{x\in P}\ell(g_{\theta_s}(h_\phi(x)),g_{\hat\theta_t(\phi)}(h_\phi(x)))=0$ when $l = 2$, and $\Exp_{x\in P}\ell(g_{\theta_s}(h_\phi(x)),g_{\hat\theta_t(\phi)}(h_\phi(x))) \neq 0$ when $l \in \{3,\cdots d\}$. 

When $l \in \{3,\cdots d\}$, since $x_{[1]}^2 = x_{[l]}^2$ in the target domain, the target pseudo-labels are all $0$. Then we solve the problem $\hat\theta_t(\phi) = \mathop{\arg\min}_{\theta}\|\theta\|_2^2,\ \text{s.t.}\  \Exp_{x\in \widehat Q}\ell(f_{\theta_s,\phi}(x),f_{\theta,\phi}(x))$ to get the target classifier $\hat\theta_t(\phi)$. Since we want the target solution with minimal norm, $\hat\theta_t(\phi) = 0$, and we can calculate $L_\textup{CST}$
as follows:
\begin{align}
\Exp_{x\in P}\ell(g_{\theta_s}(h_\phi(x)),g_{\hat\theta_t(\phi)}(h_\phi(x))) = \Exp_{(x,y)\sim P} (y - f_{\hat\theta_t(\phi),\phi}(x))^2 = \Exp_{(x,y)\sim P} y^2 = 0.18.
\end{align}
When $l = 2$, we show that $2^ {\frac{1}{3}}e_1 + 2^ {\frac{1}{3}}e_2, 2^ {\frac{1}{3}}e_1, 2^ {\frac{1}{3}}e_2$ and $0$ all appear in the target feature set with high probability. The probability that the target feature set does not contain each one in $2^ {\frac{1}{3}}e_1 + 2^ {\frac{1}{3}}e_2, 2^ {\frac{1}{3}}e_1, 2^ {\frac{1}{3}}e_2$ and $0$ equals to $\left( \frac{3}{4} \right) ^{n_t}$. Therefore with a union bound we can show that $2^ {\frac{1}{3}}e_1 + 2^ {\frac{1}{3}}e_2, 2^ {\frac{1}{3}}e_1, 2^ {\frac{1}{3}}e_2$ and $0$ all appear in the target feature set with probability at least $1 - 4\left( \frac{3}{4} \right) ^{n_t}$. In this case, $\hat\theta_t(\phi) = \mathop{\arg\min}_{\theta}\|\theta\|_2^2,\ \text{s.t.}\  \Exp_{x\in \widehat Q}\ell(f_{\theta_s,\phi}(x),f_{\theta,\phi}(x))$ results in $\hat\theta_t(\phi) = \theta_s$, which means if $\nt >\Theta(\log\frac{1}{\xi})$, with probability at least $1-\xi$ over the sampling of target data, 
\begin{align}
L_\textup{CST} = \Exp_{x\in P}\ell(g_{\theta_s}(h_\phi(x)),g_{\hat\theta_t(\phi)}(h_\phi(x))) = 0.
\end{align}
\end{proof}
\newpage
\section{Implementation Details}\label{sec:implementation}
We use PyTorch~\citep{NEURIPS2019_bdbca288} and run each experiment with 2080Ti GPUs. CBST, KLD, and IA results are from their original papers. We use the highest results in the literature for DANN, MCD, CDAN, and MDD. VAT, FixMatch, MixMatch and DIRT-T are adapted to our datasets from the official code. We adopt the pre-trained ResNet models provided in torchvision. For BERT implementation, we use the official checkpoint and PyTorch code from \url{https://github.com/huggingface/transformers}.
 
\subsection{Dataset Details}
\textbf{OfficeHome}~\citep{cite:CVPR17OfficeHome} \url{https://www.hemanthdv.org/officeHomeDataset.html} is an object recognition dataset which contains images from 4 domains. It has about 15500 images organized into 65 categories. The dataset was collected using a python web-crawler that crawled through several search engines and online image directories. The authors provided a Fair Use Notice on their website.

\textbf{VisDA-2017}~\citep{DBLP:journals/corr/abs-1710-06924} \url{https://github.com/VisionLearningGroup/taskcv-2017-public/tree/master/classification} uses synthetic object images rendered from CAD models as the training domain and real object images cropped from the COCO dataset as the validation domain. The authors provided a Term of Use on the website.

\textbf{DomainNet}~\citep{9010750} \url{http://ai.bu.edu/M3SDA/#dataset} contains images from clipart, infograph, painting, real, and sketch domains collected by searching a category name combined with a domain name from searching engines. The authors provided a Fair Use Notice on their website.

\textbf{Amazon Review}~\citep{blitzer-etal-2007-biographies} \url{https://www.cs.jhu.edu/~mdredze/datasets/sentiment/} contains product reviews taken from Amazon.com from many product types (domains). Some domains (books and dvds) have hundreds of thousands of reviews. Others (musical instruments) have only a few hundred. Reviews contain star ratings (1 to 5 stars) that can be converted into binary labels if needed. 

\subsection{Bi-level Optimization}
\label{sec:bilevel}
In Section~\ref{sec:cst}, we highlight that the optimization of CST involves bi-level optimization. In the inner loop (\eqref{eq:target}), we train the target classifier $\theta_t(\phi)$ on top of the shared representations $\phi$, thus $\theta_t(\phi)$ is a function of $\phi$. Moreover, the target classifier $\theta_t(\phi)$ is trained with target pseudo-labels $y'$, which are the sharpened version of the outputs of the source classifier $\theta_s$ on top of the shared representations $\phi$. In this sense, $\theta_t(\phi)$ relies on $\theta_s$ and $\phi$ through $y'$ implicitly, too. In the outer loop (\eqref{eq:overall}), we update the shared representations $\phi$ and the source classifier $\theta_s$ to make both the source classifier $\theta_s$ and the target classifier $\theta_t(\phi)$ perform well on the source domain. Since $\theta_t(\phi)$ relies on $\phi$ and $\theta_s$, the objective of \eqref{eq:overall} is a bi-level optimization problem. We can derive the gradient of the loss w.r.t. $\phi$ and $\theta_s$ as follows:
\begin{align}\label{dphi}
\nabla_{\phi} [L_{\widehat P}(\theta_s,\phi)
	&+L_{\widehat P}(\hat\theta_t(\phi),\phi)]\\ = &\nabla_{\phi} L_{\widehat P}(\theta_s,\phi) + \frac{ \partial L_{\widehat P}(\hat\theta_t(\phi),\phi)} {\partial \phi} + \frac{ \partial L_{\widehat P}(\hat\theta_t(\phi),\phi)} {\partial \hat\theta_t(\phi)}\frac{\dif \hat\theta_t(\phi) } {\dif \phi}\nonumber\\
	= & \nabla_{\phi} L_{\widehat P}(\theta_s,\phi) + \frac{ \partial L_{\widehat P}(\hat\theta_t(\phi),\phi)} {\partial \phi} + \frac{ \partial L_{\widehat P}(\hat\theta_t(\phi),\phi)} {\partial \hat\theta_t(\phi)}\left[ \frac{ \partial \hat\theta_t(\phi) } {\partial \phi} + \frac{ \partial \hat\theta_t(\phi) } {\partial y'}\frac{ \partial y' } {\partial \phi}\right].\nonumber
\end{align}
\begin{align}\label{dthetas}
\nabla_{\theta_s} [L_{\widehat P}(\theta_s,\phi)
	+L_{\widehat P}(\hat\theta_t(\phi),\phi)] 
	= \nabla_{\theta_s} L_{\widehat P}(\theta_s,\phi) + \frac{ \partial L_{\widehat P}(\hat\theta_t(\phi),\phi)} {\partial \hat\theta_t(\phi)}\frac{ \partial \hat\theta_t(\phi) } {\partial y'}\frac{ \partial y' } {\partial \theta_s}.
\end{align}
However, following the standard practice in self-training, we use label-sharpening to obtain target pseudo-labels $y'$, i.e. $y' = \mathop{\arg\max}_i\{f_{\theta_s,\phi}(x)_{[i]}\}$. Thus, $y'$ is not differentiable w.r.t. $\theta_s$ and $\phi$. We treat the gradient of $y'$ w.r.t. $\theta_s$ and $\phi$ as $0$ in \eqref{dphi} and \eqref{dthetas}, making optimization easier. This modification leads to exactly \eqref{opt1} and \eqref{opt2} in Algorithm~\ref{alg:CST} together with the Tsallis entropy loss.

\textbf{Speeding up bi-level optimization with MSE loss.} Standard methods of bi-level optimization back-propagate through the inner loop algorithm, which requires computing the second-order derivative (Hessian-vector products) and can be unstable. We propose to use MSE loss instead of cross entropy in the inner loop when training the head $\hat \theta(\phi)$ to calculate the analytical solution with least square and directly back-propagate to the outer loop without calculating second-order derivatives. The framework is as fast as training the two heads jointly. To adopt MSE loss in multi-class classification, we use the one-hot embedding as the output and train a multi-variate regressor following the protocol of \citet{NIPS2019_9025}. We calculate the least square solution of $\theta_t(\phi)$ based on one minibatch following the protocol of~\citet{bertinetto2018metalearning}. We also provide results of varying batchsize to verify the performance of this approximation in Table~\ref{batch}. Results indicate that the performance of CST is stable in a wide range of batchsizes.

\begin{table}[htbp]
\vskip -0.1in
\addtolength{\tabcolsep}{1pt} 
\centering 
\caption{Accuracy (\%) on VisDA-2017 with ResNet-50}
\label{batch}
\begin{small}
\begin{tabular}{l|r}
\toprule
Method & Accuracy \\
\midrule
CST (batchzize 32)& 79.9 $\pm$ 0.6 \\  
CST (batchzize 64)& 79.9 $\pm$ 0.5\\
CST (batchzize 128)& 79.6 $\pm$ 0.4\\
CST (batchzize 256)& 79.0 $\pm$ 0.6\\
\bottomrule
\end{tabular}
\end{small}
\vskip -0.1in
\end{table}

\subsection{Selection of $\alpha$}\label{sec:select}
We can also update $\alpha$ with gradient methods auto-differentiation tools as we treat $\phi$. However, since $\alpha$ has only one parameter but many other parameters ($\theta_{s,\alpha}$ and $\theta_{t,\alpha}$) rely on it, using gradient methods is costly. To ease the computational cost, we choose to discretize the feasible region of $\alpha \in [1,2]$ with $\alpha \in \{1.0,1.1,\cdots 1.9,2.0\}$, and train ${\theta}_{s,\alpha}$ with each $\alpha \in \{1.0,1.1,\cdots 1.9,2.0\}$ to generate pseudo-labels and train $\hat{\theta}_{t,\alpha}$ on pseudo-labels corresponding to each value of $\alpha$. Then we select the $\alpha \in \{ 1.0,1.1,\cdots 1.9,2.0\}$ with best performance on the source dataset following \eqref{eq:alpha}. We also update $\alpha$ at the start of each epoch, since we found more frequent update leads to no performance gain. Since we only need to select $\alpha$ once at the start of each epoch, the resulting additional computational cost only relates to training the linear head on the source and target datasets for additional 11 times per epoch, which is negligible compared to training the backbone. 

We plot the change of $\alpha$ throughout training in Figure~\ref{fig:alpha_epoch}. $\alpha$ converges to smaller value at the end of training, indicating that the penalization on uncertainty is increasing. Also note that $\alpha$ tends decrease slower for ``heuristically distant'' source and target domains. This corroborates the intuition that we need to penalize uncertain predictions mildly especially when the domain gap is large.

\begin{figure*}[htbp]
  \centering 
   \subfigure{
    \includegraphics[width=0.48\columnwidth]{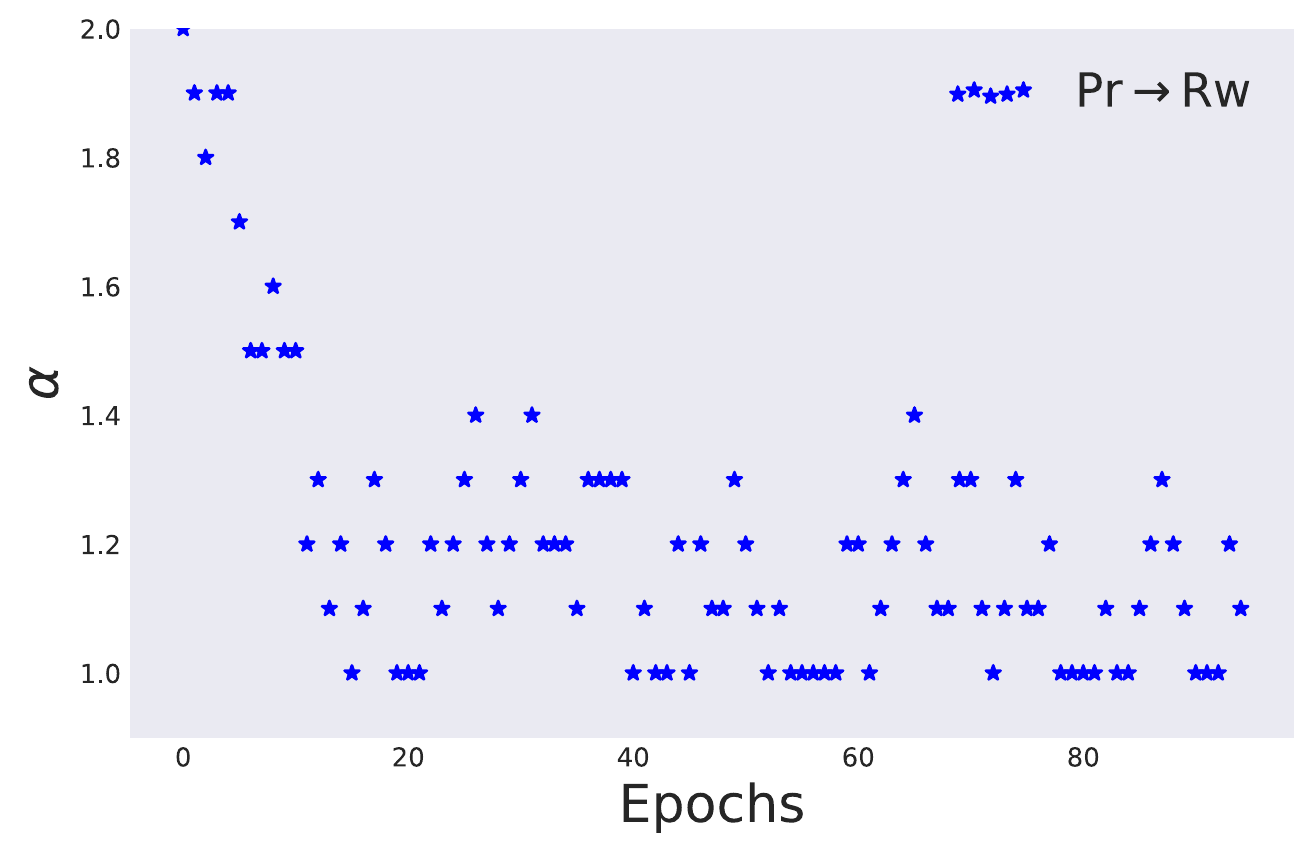}
    }
    \hfill
   \subfigure{
    \includegraphics[width=0.48\columnwidth]{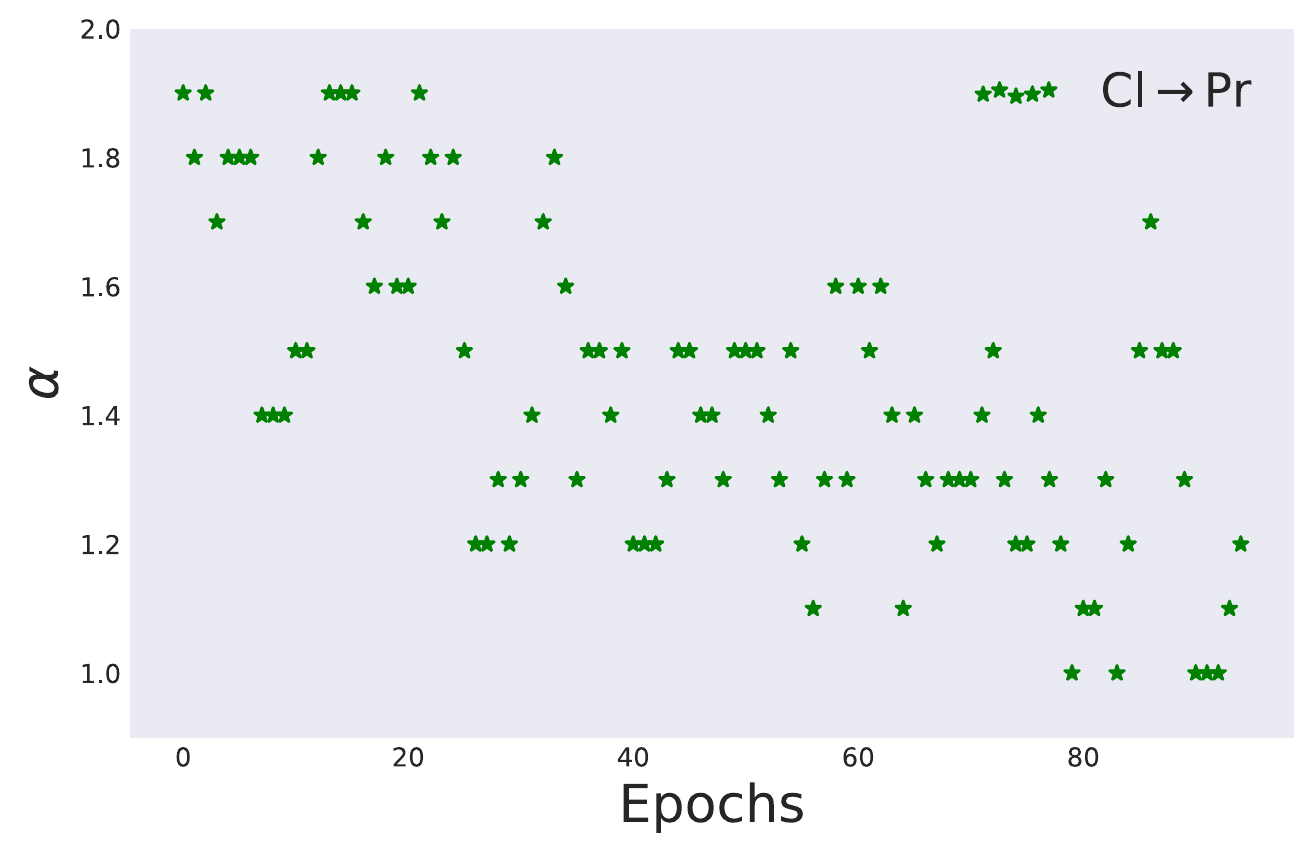}
    }
  \caption{\small \textbf{Change of $\alpha$ during training.}} 
  \label{fig:alpha_epoch}
\end{figure*}

\textbf{Difference between $\theta_s$ and $\theta_{s,\alpha}$.} We use $\theta_s$ to update the feature extractor $\phi$ and test on the target domain after training. $\theta_{s,\alpha}$ is only used to search the optimal $\alpha$, and the gradient does not back-propagate to $\phi$.

\section{Additional Experiment Details}\label{sec:add_exp}

\subsection{Additional Details of Section~\ref{sec:limitations}}\label{sec:add11}
In Figure~\ref{fig:pseudo} (Left), we use VisDA-2017 with 12 classes. To simulate the i.i.d., covariate shift and label shift setup, we resample the original Synthetic (source) and Real (target) datasets. In the i.i.d. setting, the labeled dataset consists of $1000$ random samples per class from Real, and the unlabeled dataset consists of $1000$ random samples (no overlapping with the labeled dataset) per class from Real. In the covariate shift setting, the labeled dataset consists of $1000$ random samples per class from Synthetic, and the unlabeled dataset consists of $1000$ random samples per class from Real. In the label shift setting, the number of examples is $[1800,1440,1152,922,737,590,472,377,302,240,193,154]$ for each class of the labeled dataset and $1000$ for each class of the unlabeled dataset, with both labeled and unlabeled datasets sampled randomly from Real. We train the model on labeled data until convergence to generate pseudo-labels for the unlabeled data. Then we calculate the ratio of classes in pseudo-labels and ground truth.

In Figure~\ref{fig:pseudo} (Middle), we also visualize the change of pseudo-label accuracy and the distance $d_\textup{TV}$ throughout standard self-training on original Synthetic and Real datasets. Here standard self-training refers to \eqref{eqn:Ljoint} with label-sharpening. 

In both Figure~\ref{fig:pseudo} (Right) and this subsection, confidence refers to the maximum soft-max output value, and entropy is defined as $\sum_i -y_i\textup{log}(y_i)$. We change the confidence threshold from $0$ to $1$ and entropy threshold from $0$ to $\textup{log}$\texttt{(numclasses)}. Then we plot the point (False Positive Rate, True Positive Rate) in the plane. 

In Section~\ref{sec:limitations}, we measure the quality of pseudo-labels with the total variation between pseudo-label distribution and ground-truth distribution. We show this quantity is upper-bounded by the accuracy of pseudo-labels. Intuitively, when the pseudo-label distribution and ground-truth distribution are the same, the output can still be incorrect. (e.g., in a binary problem, $P(Y=1) = P(Y=0) = 0.5$, and $\hat Y \sim \textup{uniform}[0,1]$ but is independent of $X$.) 
Recall that $d_{\textup{TV}}(Y,\hat Y) = \textup{sup}_{E\subset[C]}|P(Y\in E) - P(\hat Y \in E)|$. Suppose the supremum is reached by $\hat E$. Without loss of generality, assume $P(Y = c) > P(\hat Y = c)$ for all $c \in \hat E$. Then $P(Y = c) \le P(\hat Y = c)$ for all $c \notin \hat E$.
\begin{align*}
P(Y \neq \hat Y) =& \sum_{c=1}^C P(Y = c,\ \hat Y \neq c)
\\ =& \sum_{c\in\hat E} P(Y = c,\ \hat Y \neq c) + \sum_{c\notin\hat E} P(Y = c,\ \hat Y \neq c)
\\ \ge & \sum_{c\in\hat E} P(Y = c) - P(\hat Y = c)
\\ =& d_\textup{TV}(Y,\hat Y).
\end{align*}
In the equality, we use $P(Y = c) > P(\hat Y = c)$ when $c \in \hat E$, so $P(Y = c,\ \hat Y \neq c) \ge P(Y = c) - P(\hat Y = c)$ if $c \in \hat E$. Also note that $P(Y = c,\ \hat Y \neq c)\ge 0$ if $c \notin \hat E$.

In this subsection, we provide additional results of Section~\ref{sec:limitations}. We visualize the distributions of pseudo-labels and ground-truth with ResNet-50 backbones on Art$\rightarrow$Clipart, Product$\rightarrow$Art, Clipart$\rightarrow$Real World, Art$\rightarrow$Real World and Real World $\rightarrow$Product tasks (without resampling) in Figures~\ref{fig:pseudo1},  \ref{fig:pseudo2}, \ref{fig:pseudo3}, \ref{fig:pseudo4}, and \ref{fig:pseudo5} respectively. We also visualize the ROC curve of pseudo-label selection with confidence threshold. Results on Art$\rightarrow$Clipart, Product$\rightarrow$Art, Clipart$\rightarrow$Real World, Art$\rightarrow$Real World and Real World $\rightarrow$Product are similar to VisDA-2017. When the pseudo-labels are generated from models trained on different distributions, they can become especially unreliable in that examples of several classes are almost misclassified into other classes. Domain shift also makes the selection of correct pseudo-labels more difficult than standard semi-supervised learning.

\newpage
\begin{figure*}[!htbp]
  \centering
   \subfigure{
    \includegraphics[width=0.38\textwidth]{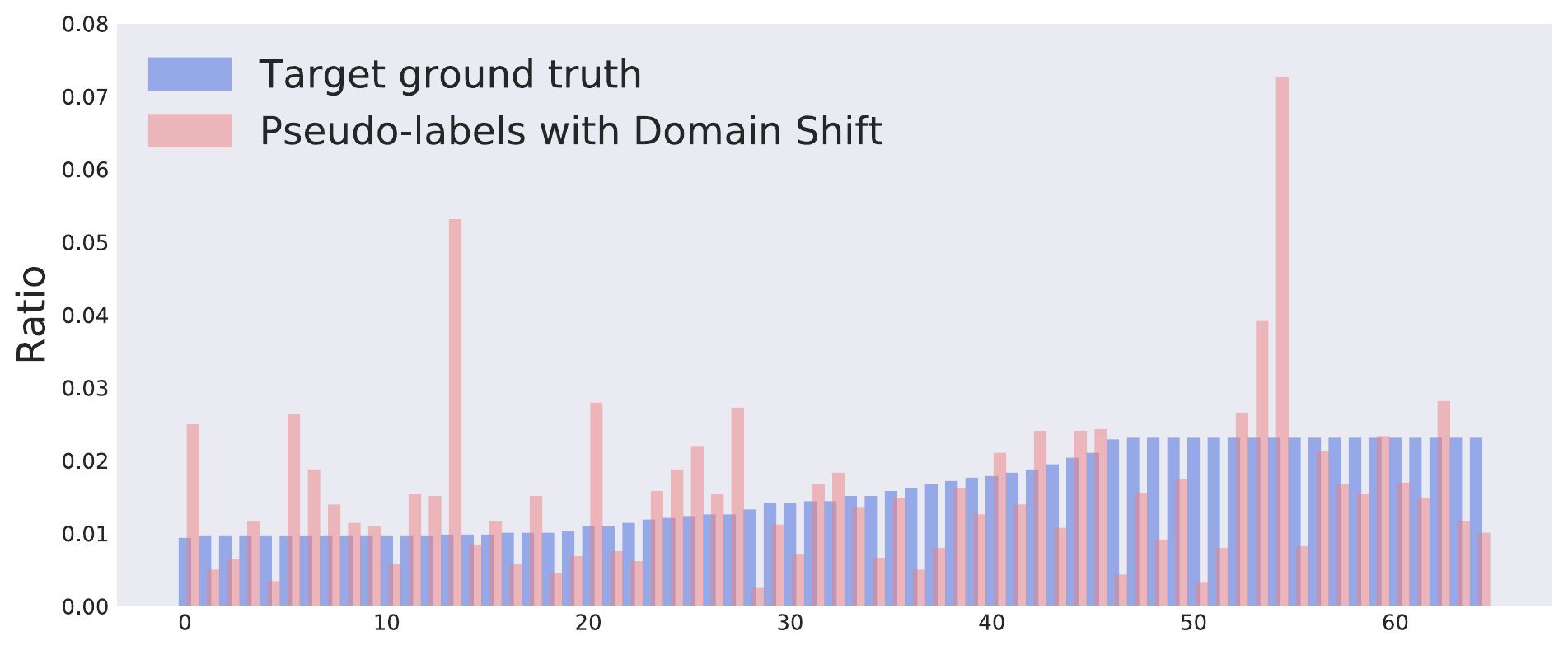} 
    }
   \subfigure{
   \includegraphics[width=0.38\textwidth]{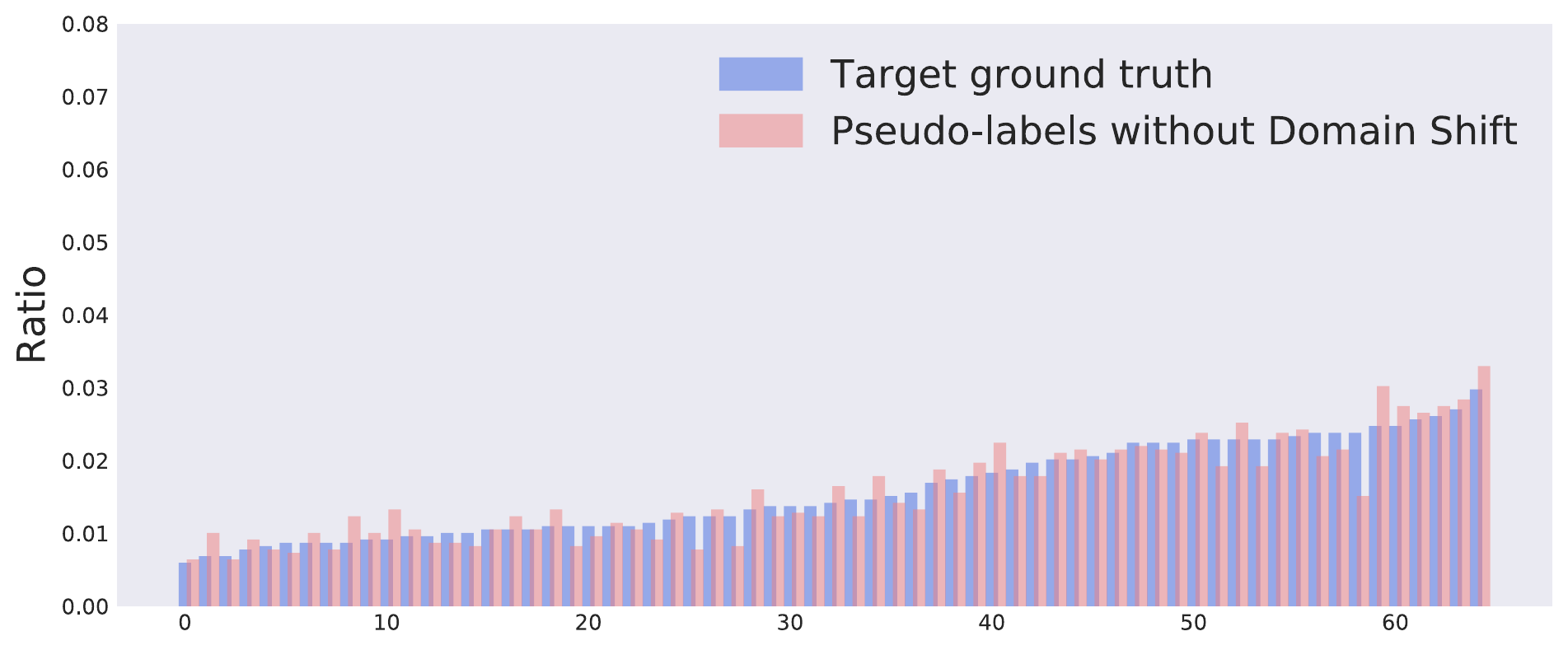} 
    }
   \subfigure{
    \includegraphics[width=0.16\textwidth]{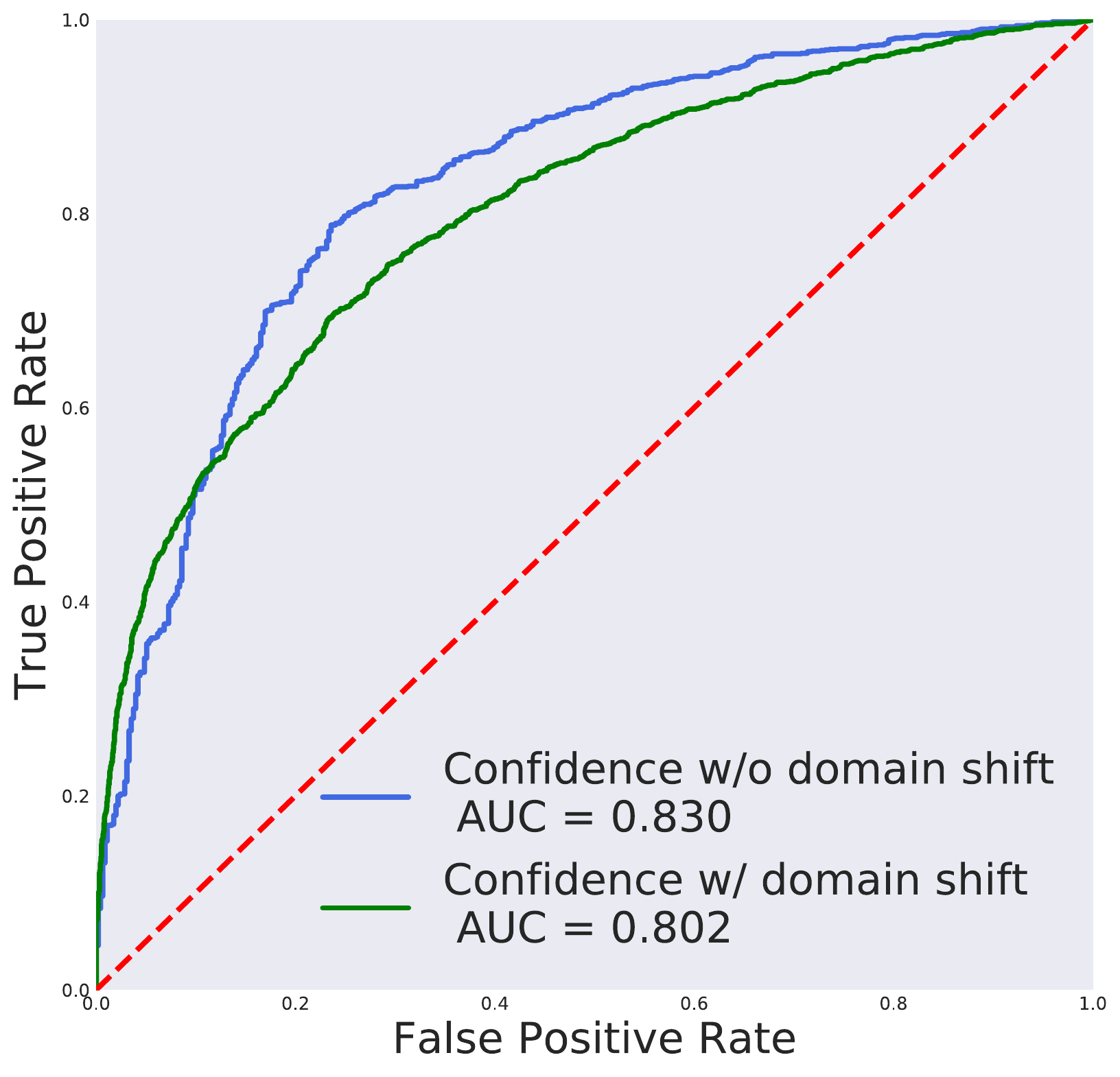}
    }
    \vskip -0.08in\
    
    \vspace{-5pt}
  \caption{\textbf{Analysis of pseudo-labels under domain shift on Art$\rightarrow$Clipart.} Left: Comparison of pseudo-label distributions with and without domain shift. Right: Comparison of pseudo-label selection with and without domain shift. } 
  \label{fig:pseudo1}
\end{figure*}

\begin{figure*}[!htbp]
  \centering
   \subfigure{
    \includegraphics[width=0.38\textwidth]{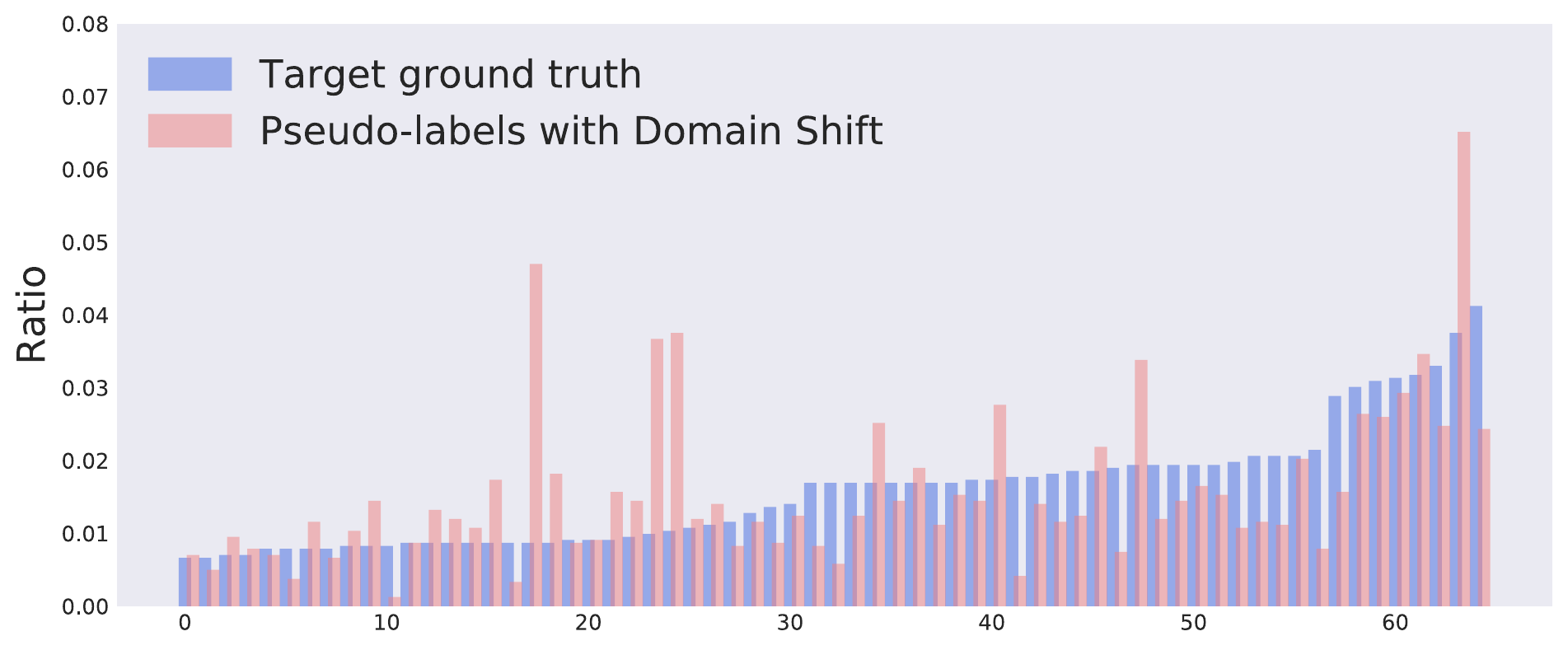} 
    }
   \subfigure{
   \includegraphics[width=0.38\textwidth]{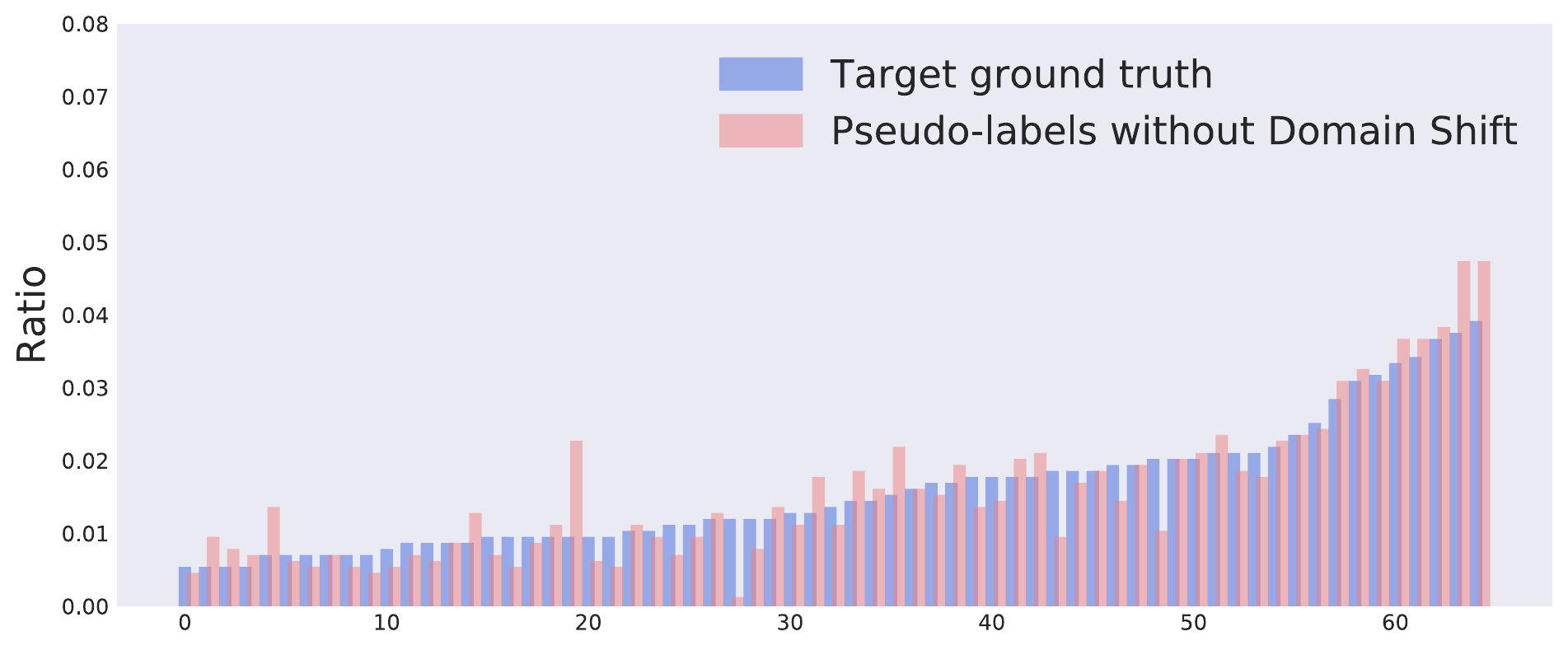} 
    }
   \subfigure{
    \includegraphics[width=0.16\textwidth]{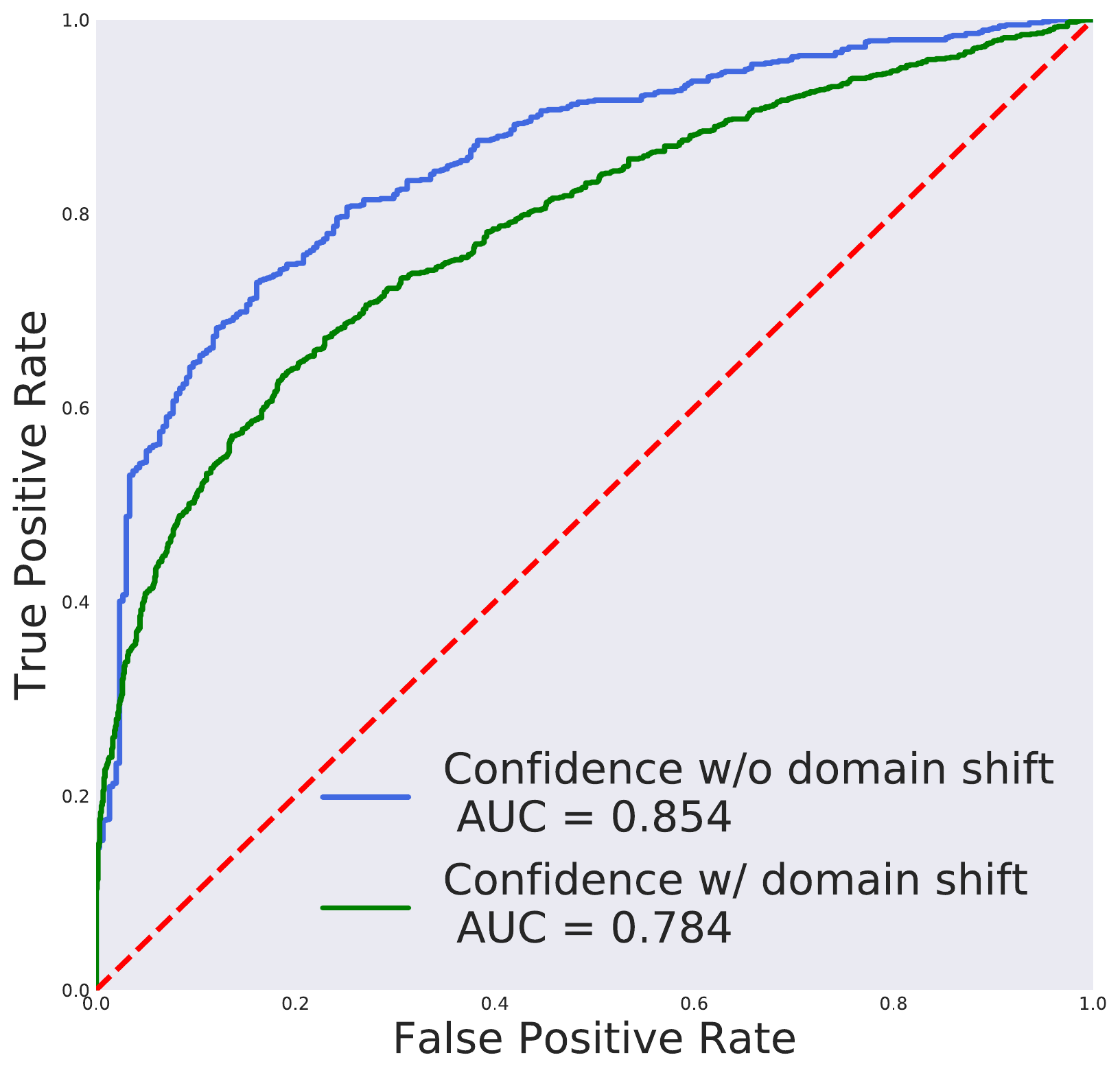}
    }
    \vskip -0.08in\
        \vspace{-5pt}
  \caption{\textbf{Analysis of pseudo-labels under domain shift on Product$\rightarrow$Art.} Left: Comparison of pseudo-label distributions with and without domain shift. Right: Comparison of pseudo-label selection with and without domain shift. } 
  \label{fig:pseudo2}
\end{figure*}

\begin{figure*}[!htbp]
  \centering
   \subfigure{
    \includegraphics[width=0.38\textwidth]{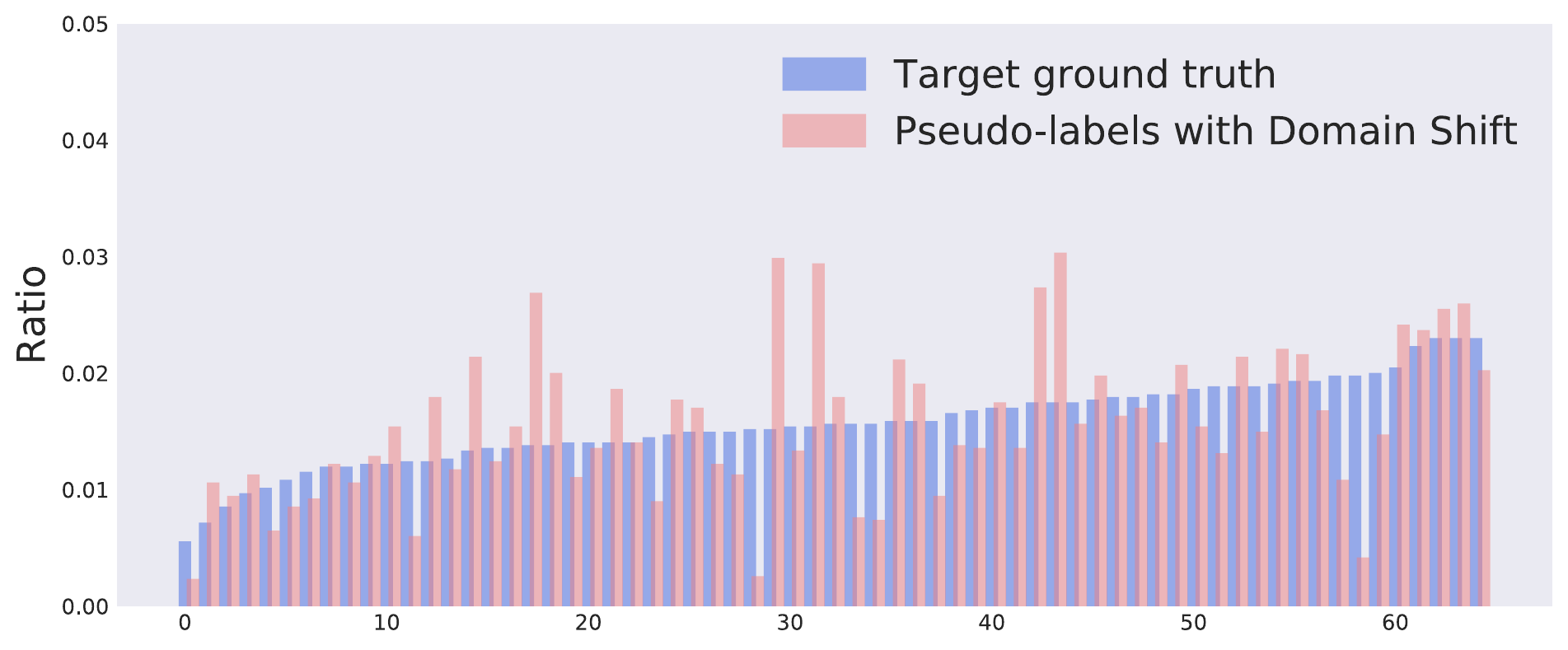} 
    }
   \subfigure{
   \includegraphics[width=0.38\textwidth]{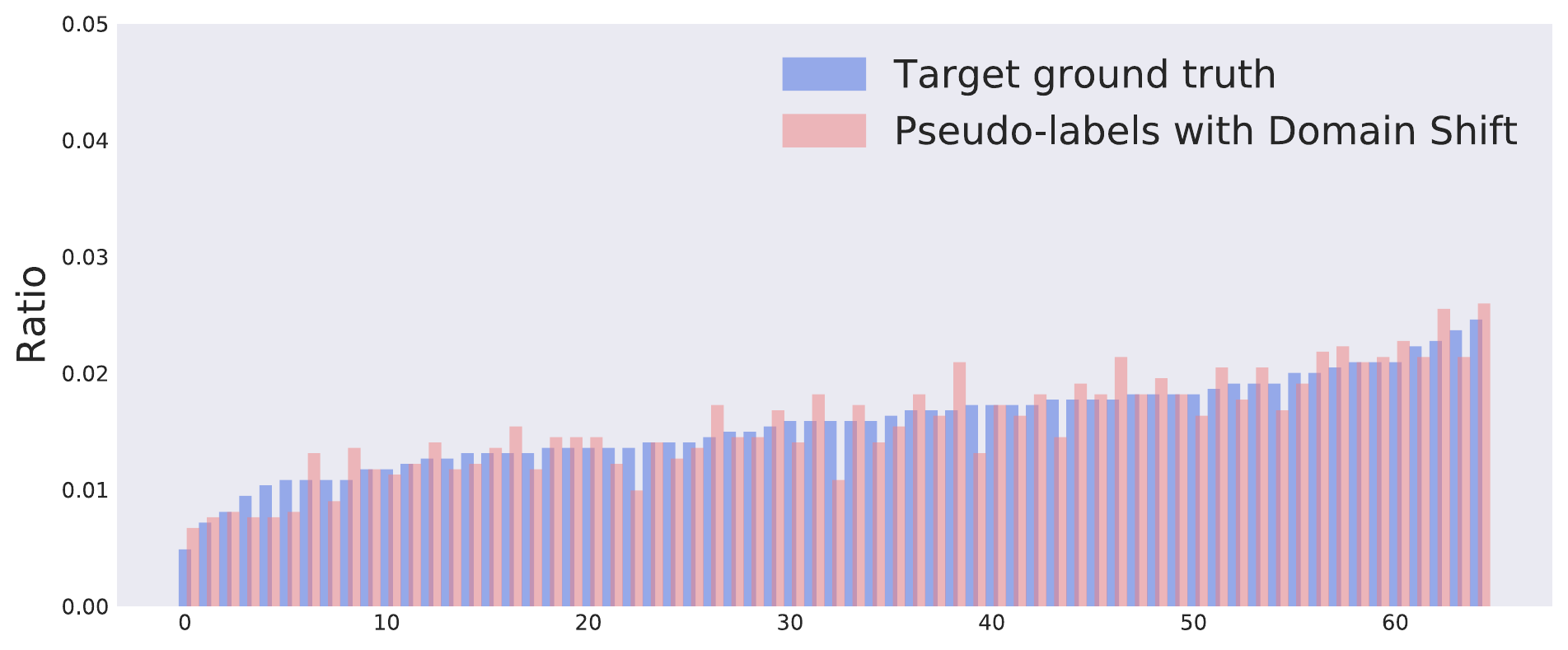} 
    }
   \subfigure{
    \includegraphics[width=0.16\textwidth]{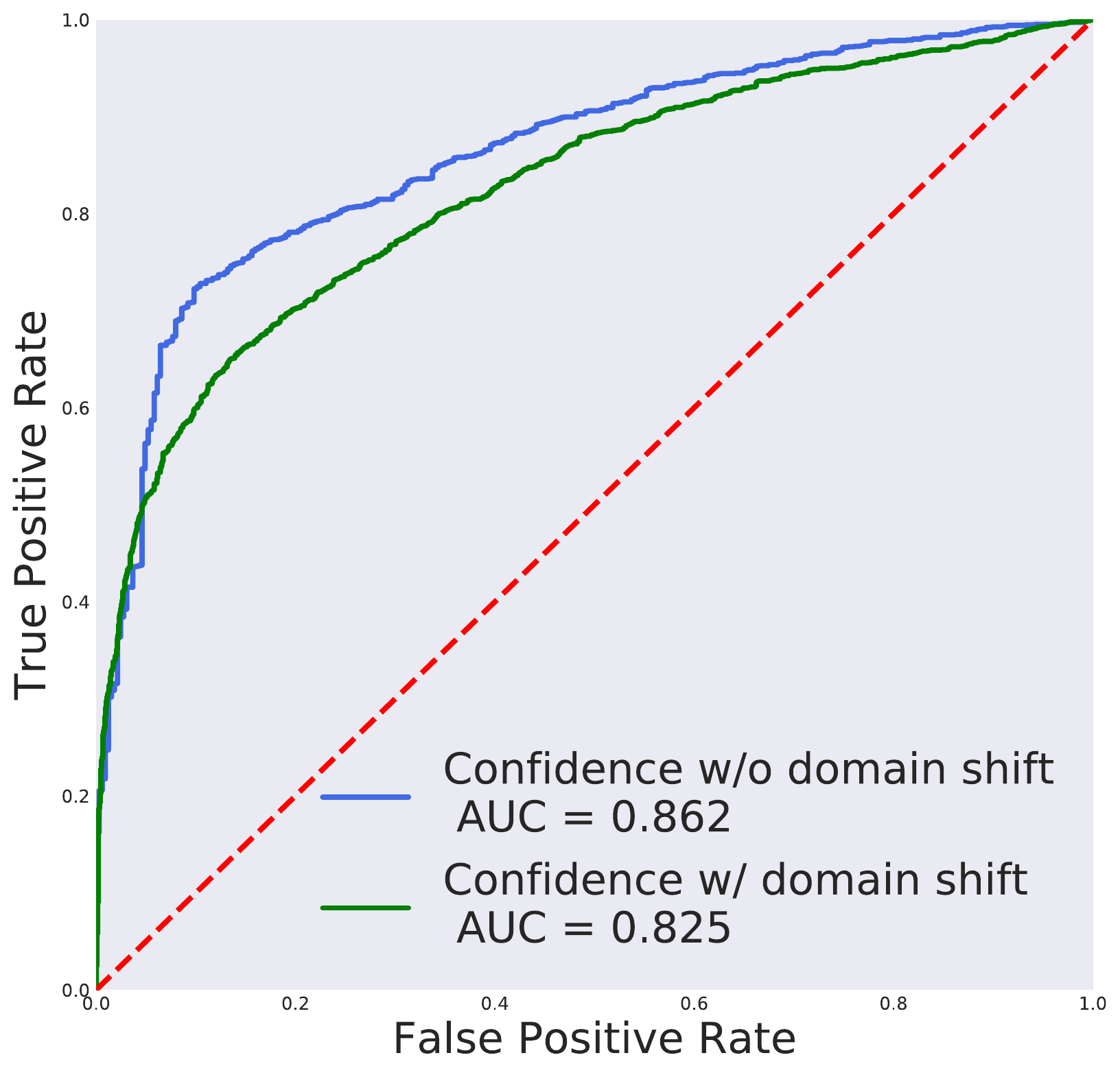}
    }
    \vskip -0.08in\
        \vspace{-5pt}
  \caption{\textbf{Analysis of pseudo-labels under domain shift on Clipart$\rightarrow$Real World.} Left: Comparison of pseudo-label distributions with and without domain shift. Right: Comparison of pseudo-label selection with and without domain shift. } 
  \label{fig:pseudo3}
\end{figure*}

\begin{figure*}[!htbp]
  \centering
   \subfigure{
    \includegraphics[width=0.38\textwidth]{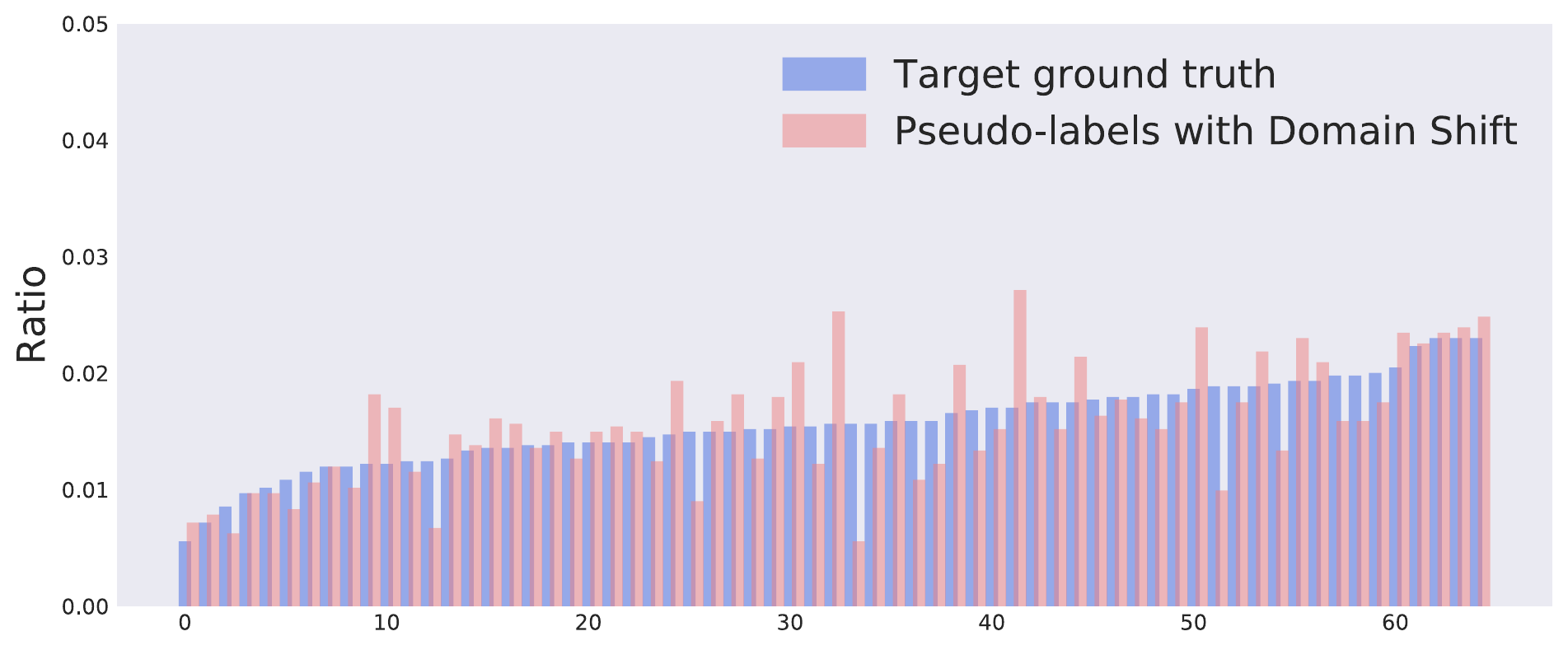} 
    }
   \subfigure{
   \includegraphics[width=0.38\textwidth]{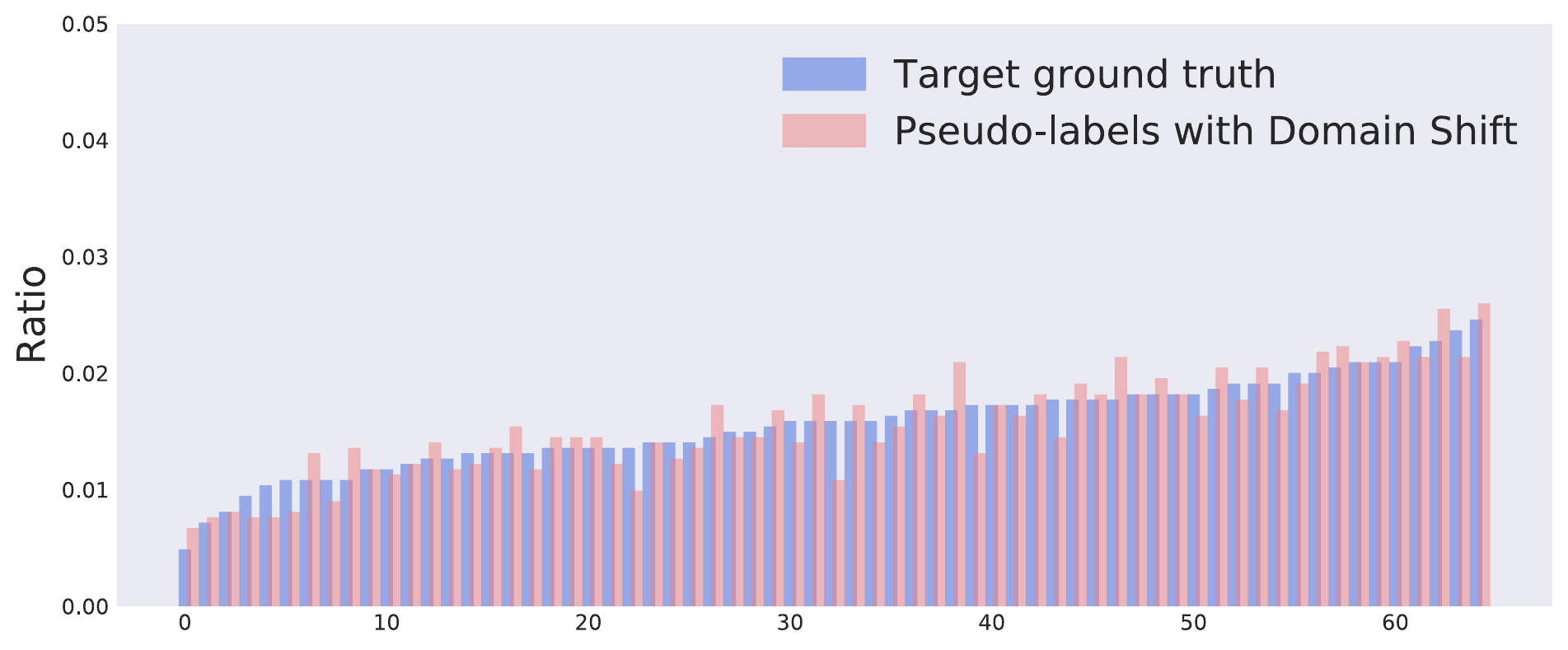} 
    }
   \subfigure{
    \includegraphics[width=0.16\textwidth]{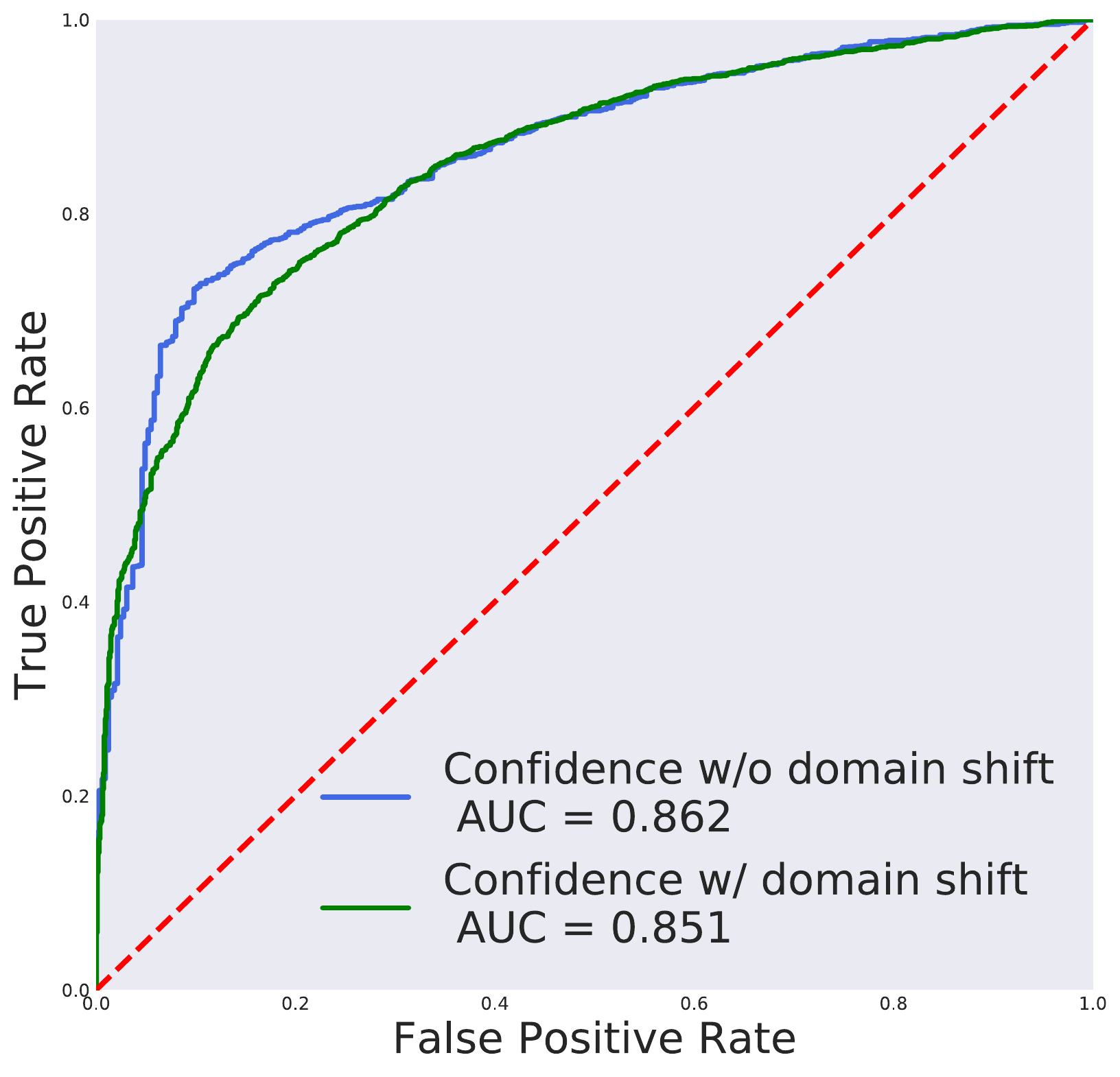}
    }
    \vskip -0.08in\
        \vspace{-5pt}
  \caption{\textbf{Analysis of pseudo-labels under domain shift on Art$\rightarrow$Real World.} Left: Comparison of pseudo-label distributions with and without domain shift. Right: Comparison of pseudo-label selection with and without domain shift. } 
  \label{fig:pseudo4}
\end{figure*}

\begin{figure*}[!htbp]
  \centering
   \subfigure{
    \includegraphics[width=0.38\textwidth]{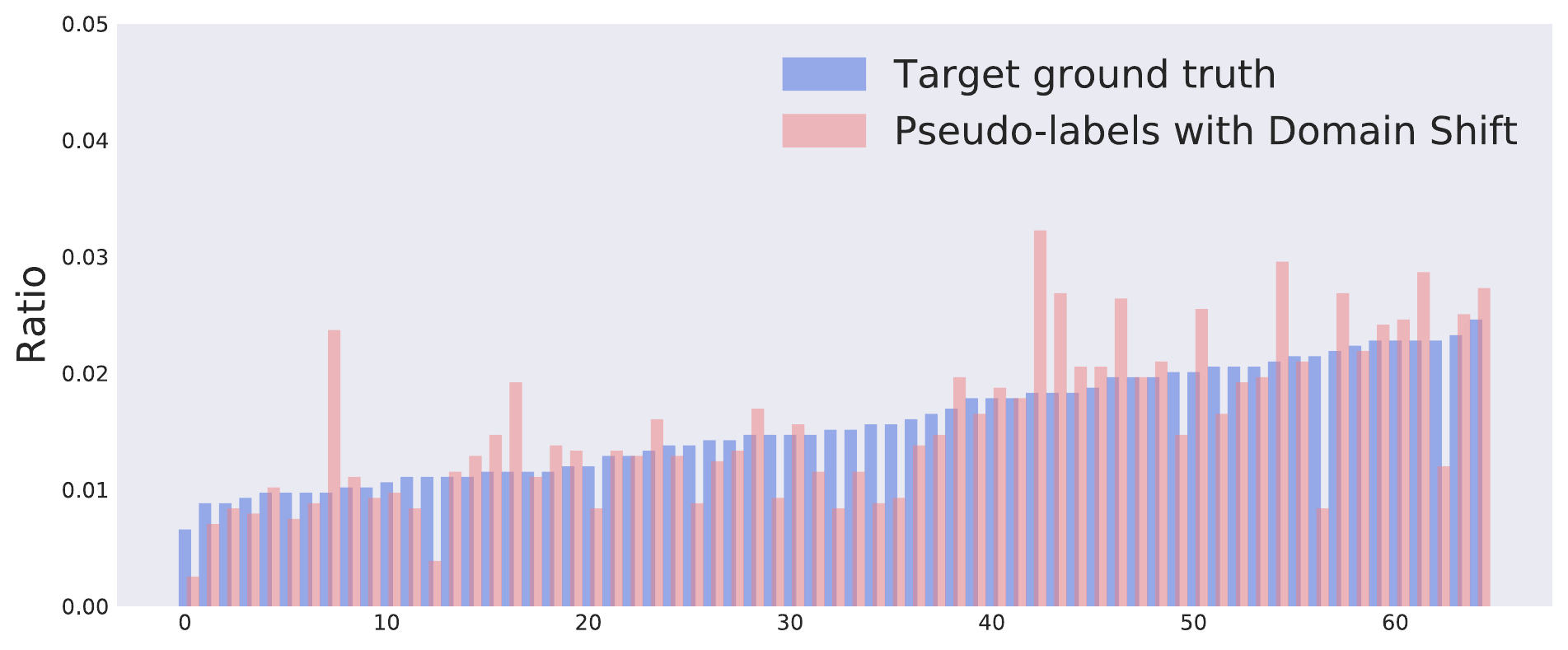} 
    }
   \subfigure{
   \includegraphics[width=0.38\textwidth]{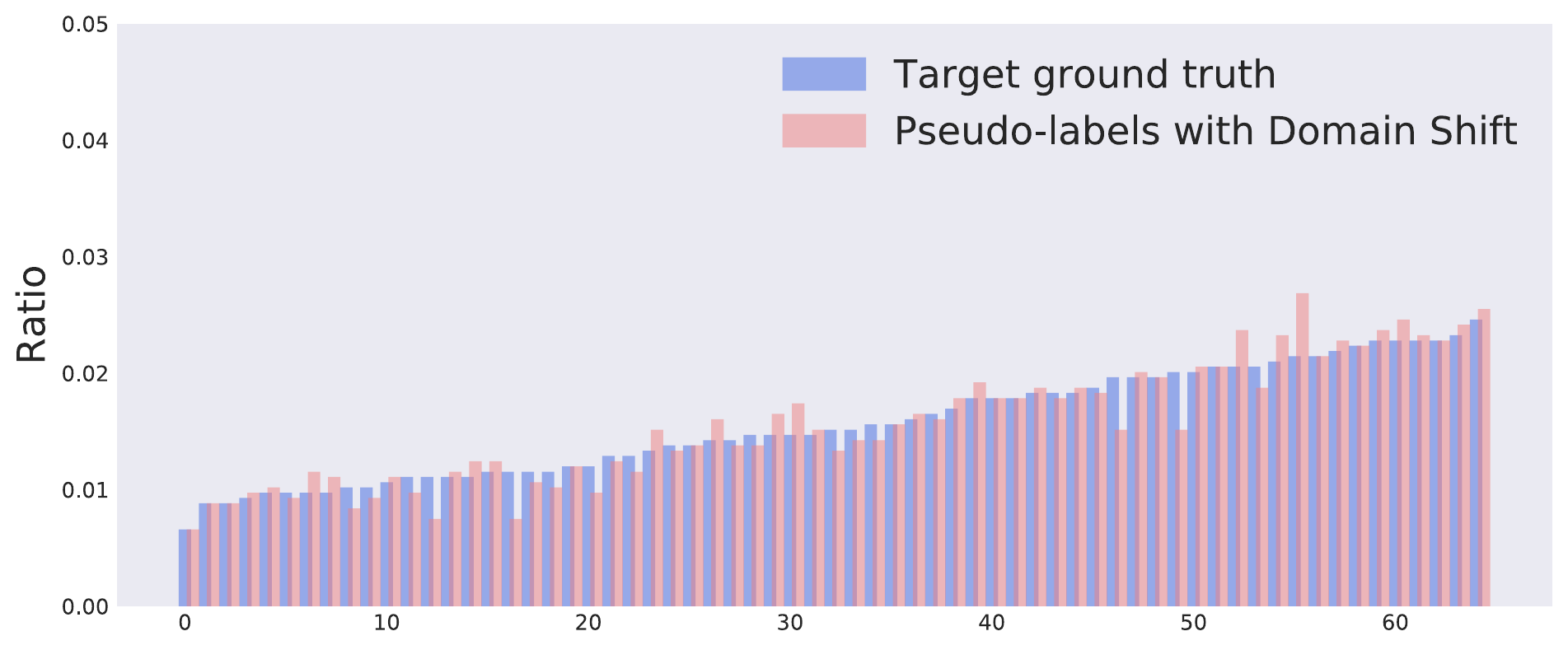} 
    }
   \subfigure{
    \includegraphics[width=0.16\textwidth]{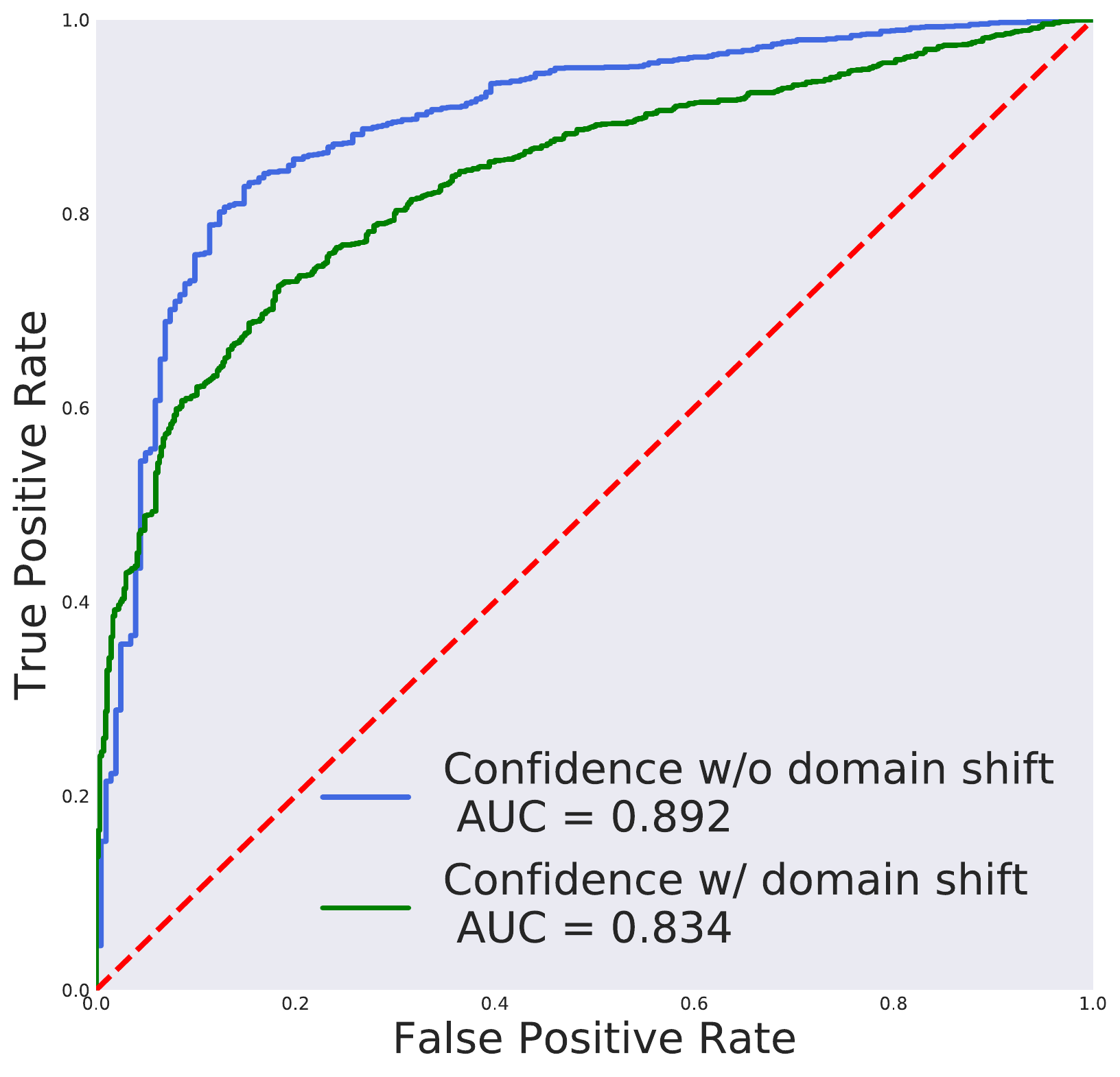}
    }
    \vskip -0.08in\
        \vspace{-5pt}
  \caption{\textbf{Analysis of pseudo-labels under domain shift on Real World$\rightarrow$Product.} Left: Comparison of pseudo-label distributions with and without domain shift. Right: Comparison of pseudo-label selection with and without domain shift. } 
  \label{fig:pseudo5}
\end{figure*}

\newpage
\subsection{Results on digit datasets}
We provide results on the digit datasets to test the performance of the proposed method without using pre-training. We use DTN architecture following~\citet{cite:NIPS18CDAN}. Results in Table~\ref{digit} indicate that CST achieve comparable performance to state-of-the-art.

\begin{table}[htbp]
\addtolength{\tabcolsep}{1pt} 
\centering 
\caption{Accuracy (\%) on digits datasets with DTN}
\label{digit}
\begin{small}
\begin{tabular}{l|cc}
\toprule
Method & MNIST$\rightarrow$USPS & SVHN$\rightarrow$MNIST \\
\midrule
CDAN & 95.6 $\pm$ 0.2 & 96.9 $\pm$ 0.2 \\  
RWOT (CVPR 2020) & 98.5 $\pm$ 0.2 & 97.5 $\pm$ 0.2 \\  
\midrule
\textbf{CST} & 98.5 $\pm$ 0.2 & \textbf{98.2} $\pm$ 0.2 \\  
\bottomrule
\end{tabular}
\end{small}
\end{table}

\subsection{Results on DomainNet}
We test the performance of the proposed method on the 40-class DomainNet~\citep{9010750} subset following the protocol of~\citet{2019Class}. Results in Table~\ref{DomainNet} indicate that CST outperforms MDD by a large margin.
\begin{table*}[htbp]
\addtolength{\tabcolsep}{-1.1pt} 
\centering
\caption{Accuracy (\%) on {DomainNet} for unsupervised domain adaptation (\texttt{ResNet-50}).}
\label{DomainNet}
\begin{small}
\begin{tabular}{l|cccccccccccc|c}
\toprule
Method & R-C & R-P & R-S & C-R & C-P & C-S & P-R & P-C & P-S & S-R & S-C & S-P & Avg. \\
\midrule
DANN \citep{cite:JMLR17DANN} & 63.4 & 73.6 & 72.6 & 86.5 & 65.7 & 70.6 & 86.9 & 73.2 & 70.2 & 85.7 & 75.2 & 70.0 & 74.5\\
COAL~\citep{2019Class} & 73.9 & 75.4 & 70.5 & 89.6 & 70.0 & 71.3 & 89.8 & 68.0 & 70.5 & 88.0 & 73.2 & 70.5 & 75.9\\
MDD~\citep{pmlr-v97-zhang19i} & 77.6 & 75.7 & 74.2 & 89.5 & 74.2 & 75.6 & 90.2 & 76.0 & 74.6 & 86.7 & 72.9 & 73.2 & 78.4\\
\midrule
\textbf{CST} & \textbf{83.9} & \textbf{78.1} & \textbf{77.5} & \textbf{90.9} & \textbf{76.4} & \textbf{79.7} & \textbf{90.8} & \textbf{82.5} & \textbf{76.5} & \textbf{90.0} &  \textbf{82.8}& \textbf{74.4} &\textbf{82.0}\\
\bottomrule
\end{tabular}
\end{small}
\end{table*}

\subsection{Standard deviations of Tables}
We visualize the performance of CST and best baselines in Table~\ref{Sentiment} with standard deviations. Results indicate that the improvement of CST over previous methods is significant.

\begin{figure*}[h]
  \centering 
      \subfigure{
    \includegraphics[width=0.9\textwidth]{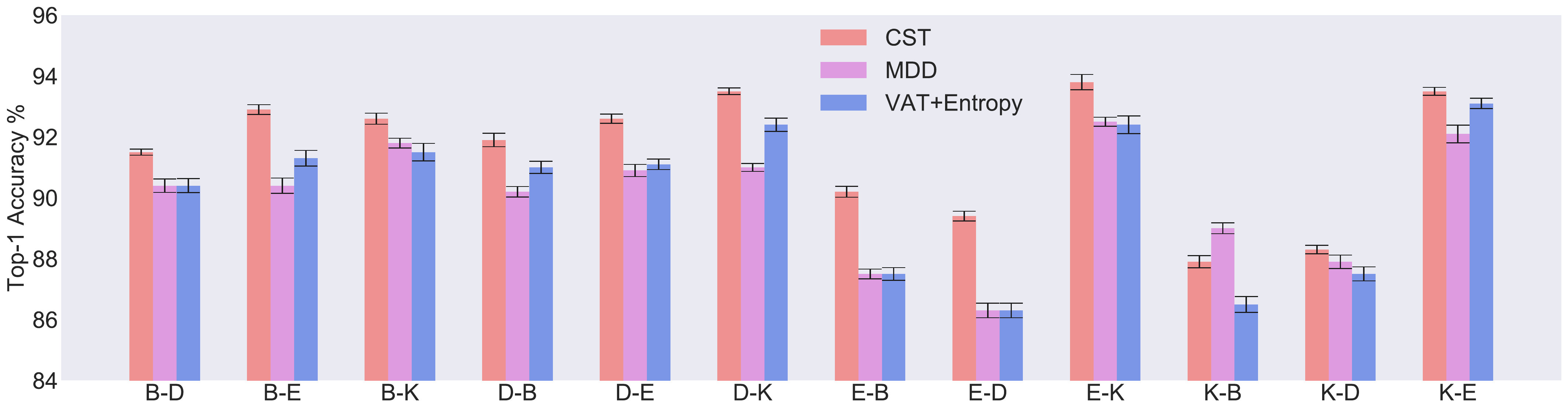}
    }
        \vskip -0.12in\
    \label{fig:err_bar}
  \caption{\textbf{Visualization of standard deviations of CST and baselines.} CST outperforms baselines significantly on all tasks except K$\rightarrow$B.} 
\end{figure*}

\newpage
\section{Limitations of CST and Future Directions}\label{sec:lim}
CST overcomes the drawbacks of standard pseudo-labeling in domain adaptation by dealing with the domain discrepancy explicitly with the cycle loss. However, pseudo-labeling is only one main direction of semi-supervised learning. Consistency regularization~\citep{cite:TPAMI18VAT} and self-ensembling~\citep{NIPS2014_66be31e4} are also important paradigms in semi-supervised learning. How to apply them to the setting with distribution shift and achieve consistent performance gain is still an open question. More recently, \citet{carlini2021poisoning} investigated the effect of adversarial unlabeled data poisoning to self-training. Future works can extend CST to this setting and extend consistency regularization as a potential way of defense.  

\section{Broader Impact}\label{sec:ethic}
This work studies and improves self-training in the unsupervised domain adaptation setting. When deployed in real-world applications, distribution shift between labeled and unlabeled data can come in various ways. Although the quality of labeled datasets can be monitored, enabling the mitigation of bias in pre-processing, bias in unlabeled datasets can be intractable. Self-training with biased unlabeled data is highly risky since it may potentially amplify the biased models predictions. This work explores how to mitigate the effect of dataset bias in unlabeled data, and can potentially promote fair self-training systems.

\end{document}